\documentclass[11pt, oneside]{article}
\usepackage{etoolbox}
\newcommand{\arxiv}[1]{\iftoggle{icml}{}{#1}}
\newcommand{\icml}[1]{\iftoggle{icml}{#1}{}}
\newtoggle{icml}
\global\togglefalse{icml}

\icml{
  \usepackage[subtle]{savetrees} %
  }

\icml{
  \usepackage{microtype}
\usepackage{graphicx}
\usepackage{subfigure}
\usepackage{booktabs} %

}

\icml{
  \newcommand{\State}{\STATE}
  \newcommand{\For}{\FOR}
  \newcommand{\EndFor}{\ENDFOR}
  \newcommand{\If}{\IF}
  \newcommand{\EndIf}{\ENDIF}
  \newcommand{\Return}{\RETURN}
  \newcommand{\Statex}{\State}
}

  \usepackage{comment}

\arxiv{
  \usepackage{geometry}
  \geometry{letterpaper}
  }

\usepackage{graphicx}	
\usepackage{amssymb}
\usepackage{amsfonts,latexsym,amsthm,amssymb,amsmath,amscd,euscript}
\usepackage{bbm}
\usepackage{framed}
\arxiv{\usepackage{fullpage}}
\usepackage{mathrsfs}

\newcommand{\loose}{\looseness=-1}

\PassOptionsToPackage{dvipsnames}{xcolor}

\usepackage{etoolbox}

\arxiv{\usepackage{algorithm}}
\arxiv{\usepackage{verbatim}}
\arxiv{\usepackage[noend]{algpseudocode}}
\newcommand{\multiline}[1]{\parbox[t]{\dimexpr\linewidth-\algorithmicindent}{#1}}

\usepackage{mathabx}
\usepackage{accents}
\usepackage{setspace}
\usepackage{tikz-cd}
\usepackage{xspace}
\usepackage[final]{showlabels}

\makeatletter
\newtheorem*{rep@theorem}{\rep@title}
\newcommand{\newreptheorem}[2]{%
\newenvironment{rep#1}[1]{%
 \def\rep@title{#2 \ref{##1}}%
 \begin{rep@theorem}}%
 {\end{rep@theorem}}}
\makeatother

\makeatletter
\newcommand\xlabel[2][]{\phantomsection\def\@currentlabelname{#1}\label{#2}}
\makeatother

\icml{\usepackage{icml2023}}

\usepackage{hyperref}
\hypersetup{colorlinks=true,citecolor=blue,urlcolor=black,linkbordercolor={1 0 0}}
\hypersetup{
  colorlinks=true,
  linkcolor=blue,
  filecolor=blue,
  citecolor = black,      
  urlcolor=cyan,
}

\theoremstyle{plain}
\newtheorem{theorem}{Theorem}
\newreptheorem{theorem}{Theorem}
\newtheorem{lemma}[theorem]{Lemma}
\newtheorem{corollary}[theorem]{Corollary}

\newtheorem{proposition}[theorem]{Proposition}
\newtheorem{fact}[theorem]{Fact}

\newtheorem{problem}[theorem]{Problem}

\theoremstyle{definition}
\newtheorem{defn}[theorem]{Definition}

\theoremstyle{remark}
\newtheorem{remark}[theorem]{Remark}

\numberwithin{theorem}{section}

\icml{\icmltitlerunning{Hardness of Independent Learning in Markov Games}}

\usepackage{amsthm}
\usepackage{mathtools}
\usepackage{amsmath}
\usepackage{bbm}
\usepackage{amsfonts}
\usepackage{amssymb}

\usepackage{inconsolata}

\usepackage[capitalize,noabbrev]{cleveref}

\renewcommand{\eqref}[1]{\texorpdfstring{\hyperref[#1]{Eq. (\ref*{#1})}}{Eq. (\ref*{#1})}}

\crefformat{equation}{#2(#1)#3}
\Crefformat{equation}{#2(#1)#3}

\Crefformat{figure}{#2Figure #1#3}
\Crefname{assumption}{Assumption}{Assumptions}
\Crefname{defn}{Definition}{Definitions}
\Crefformat{assumption}{#2Assumption #1#3}
\crefformat{defn}{#2Definition #1#3}
\Crefformat{defn}{#2Definition #1#3}

\newcommand{\ind}[1]{^{\scriptscriptstyle (#1)}}

\DeclarePairedDelimiter{\abs}{\lvert}{\rvert} %
\DeclarePairedDelimiter{\brk}{[}{]}
\DeclarePairedDelimiter{\crl}{\{}{\}}
\DeclarePairedDelimiter{\prn}{(}{)}

\DeclareMathOperator{\En}{\mathbb{E}}

\newcommand{\wt}[1]{\widetilde{#1}}
\newcommand{\wb}[1]{\widebar{#1}}

\def\ddefloop#1{\ifx\ddefloop#1\else\ddef{#1}\expandafter\ddefloop\fi}
\def\ddef#1{\expandafter\def\csname bb#1\endcsname{\ensuremath{\mathbb{#1}}}}
\ddefloop ABCDEFGHIJKLMNOPQRSTUVWXYZ\ddefloop
\def\ddefloop#1{\ifx\ddefloop#1\else\ddef{#1}\expandafter\ddefloop\fi}
\def\ddef#1{\expandafter\def\csname sf#1\endcsname{\ensuremath{\mathsf{#1}}}}
\ddefloop ABCDEFGHIJKLMNOPQRSTUVWXYZ\ddefloop
\def\ddef#1{\expandafter\def\csname c#1\endcsname{\ensuremath{\mathcal{#1}}}}
\ddefloop ABCDEFGHIJKLMNOPQRSTUVWXYZ\ddefloop
\def\ddef#1{\expandafter\def\csname h#1\endcsname{\ensuremath{\widehat{#1}}}}
\ddefloop ABCDEFGHIJKLMNOPQRSTUVWXYZ\ddefloop
\def\ddef#1{\expandafter\def\csname hc#1\endcsname{\ensuremath{\widehat{\mathcal{#1}}}}}
\ddefloop ABCDEFGHIJKLMNOPQRSTUVWXYZ\ddefloop
\def\ddef#1{\expandafter\def\csname tc#1\endcsname{\ensuremath{\widetilde{\mathcal{#1}}}}}
\ddefloop ABCDEFGHIJKLMNOPQRSTUVWXYZ\ddefloop
\def\ddefloop#1{\ifx\ddefloop#1\else\ddef{#1}\expandafter\ddefloop\fi}
\def\ddef#1{\expandafter\def\csname scr#1\endcsname{\ensuremath{\mathscr{#1}}}}
\ddefloop ABCDEFGHIJKLMNOPQRSTUVWXYZ\ddefloop

\newcommand{\eps}{\epsilon}

\newcommand{\ldef}{\vcentcolon=}

\newcommand{\bigoh}{O}
\newcommand{\bigoht}{\wt{O}}

\newcommand{\bigom}{\Omega}
\newcommand{\bigomt}{\wt{\Omega}}

\usepackage[suppress]{color-edits}
\addauthor{df}{ForestGreen}
\addauthor{ng}{purple}

\arxiv{
  \usepackage{parskip}
}

\newcommand{\citet}[1]{\cite{#1}}
\newcommand{\citep}[1]{\cite{#1}}

\icml{
\addtocontents{toc}{\protect\setcounter{tocdepth}{0}}
  }

  \newcommand{\utility}{reward\xspace}

  \newcommand{\vlearning}{\texttt{V-learning}\xspace}

\newcommand{\sigmahat}{\wh{\sigma}}

\newcommand{\tstar}{t^{\star}}

\newcommand{\nc}{\newcommand}
\nc{\DMO}{\DeclareMathOperator}
\newcount\Comments  %
\DeclareMathOperator*{\argmax}{arg\,max}

\DMO{\prox}{prox}
\DMO{\Span}{span}
\DMO{\UCB}{UCB}
\DMO{\LCB}{LCB}
\nc{\br}{\mathbf{r}}
\nc{\depth}[1]{{\rm d}({#1})}
\nc{\tA}{\textsc{A}}
\nc{\child}[2]{{\rm ch}_{#1}({#2})}
\nc{\parent}[1]{{\rm pa}({#1})}
\nc{\dg}{\dagger}
\nc{\bB}{\mathbf{B}}
\nc{\unif}{\mu_{\rm unif}}
\nc{\indsig}[2]{\mathcal{I}_{#1}({#2})}
\nc{\total}{{\rm fin}}
\nc{\early}{{\rm pre}}
\nc{\zsink}{z_{\rm sink}}
\nc{\lowv}{{\rm low}}
\nc{\ol}{\overline}
\nc{\ul}{\underline}
\nc{\madec}[3]{\texttt{ma-dec}_{#1}({#2}, {#3})}
\nc{\madeco}[1]{\texttt{ma-dec}_{#1}}
\nc{\madecd}[3]{\texttt{ma-dec}^{\texttt{d}}_{#1}({#2}, {#3})}
\nc{\SF}{\mathscr{F}}
\nc{\PiMarkovdet}{\Pi^{\mathrm{markov,det}}}
\nc{\PiMarkov}{\Pi^{\mathrm{markov}}}
\nc{\Pidet}{\Pi^{\mathrm{gen,det}}}
\nc{\Pirnd}{{\til\Pi}^{\mathrm{gen,rnd}}}
\nc{\Pirndrnd}{{\Pi}^{\mathrm{gen,rnd}}}
\nc{\Piall}{\til{\Pi}^{\mathrm{all}}}
\nc{\Piprod}{\Pi^{\mathrm{gen,prod}}}
\nc{\modf}{\til}
\nc{\distp}{P}
\nc{\Vdef}{V^{\mathrm{def}}}
\nc{\pidef}{\pi^{\mathrm{def}}}
\nc{\piatt}{\pi^{\mathrm{att}}}

\nc{\gamvec}{\gamma}
\nc{\til}{\widetilde}
\nc{\td}{\tilde}
\nc{\wh}{\widehat}
\nc{\todo}[1]{\ifnum\Comments=1 {\color{red}  [TODO: #1]}\fi}
\nc{\old}[1]{\ifnum\Comments=1 {\color{brown}  [OLD: #1]}\fi}
\nc{\noah}{\ngcomment}
\nc{\noahold}[1]{}
\nc{\BP}{\mathbb{P}}
\nc{\BI}{\mathbb{I}}
\nc{\PE}{\mathsf{Emb}}
\nc{\PF}{\mathsf{Fac}}
\nc{\indic}[1]{\mathbb{I}_{#1}}
\nc{\enc}{\mathsf{enc}}

\nc{\icmlcut}[1]{\color{black}{#1}\color{black}}
\nc{\icmlmove}[1]{\color{blue}{#1}\color{black}}
 \nc{\icmlfncut}[1]{\footnote{\textcolor{black}{#1}}}
\nc{\icmlfnmove}[1]{\footnote{\textcolor{blue}{#1}}}

\nc{\Tnm}{T}
\nc{\epnm}{\ep}
\nc{\Nnm}{N}
\nc{\Qnm}{Q}
\nc{\Tn}{T}
\nc{\epn}{\ep}
\nc{\Nn}{N}

\nc{\fools}[3]{\MF_{#3}({#1}, {#2})}
\nc{\fool}[2]{\MF({#1},{#2})}
\nc{\clip}[2]{{\rm clip}\left[ \left. {#1} \right| {#2} \right]}
\nc{\imax}{\omega}
\DMO{\conv}{conv}
\nc{\MH}{\mathcal{H}}
\nc{\CF}{\mathscr{F}}
\nc{\CH}{\mathscr{H}}
\nc{\MV}{\mathcal{V}}
\nc{\MC}{\mathcal{C}}
\nc{\MI}{\mathcal{I}}
\nc{\st}{\star}
\nc{\lng}{\langle}
\nc{\rng}{\rangle}
\DMO{\OOPT}{opt}
\nc{\dopt}[2]{\ell_{\OOPT}({#1},{#2})}
\nc{\grad}{\nabla}
\nc{\MG}{\mathcal{G}}
\nc{\MP}{\mathcal{P}}
\nc{\TT}{\mathbb{T}}
\nc{\TTmax}{\TT_{\max}}
\DMO{\WREG}{wReg}
\nc{\Reg}{\mathrm{Reg}}
\DMO{\Ham}{Ham}
\DMO{\Gap}{Gap}
\DMO{\GD}{GD}
\DMO{\GDA}{GDA}
\DMO{\EG}{EG}
\DMO{\OGDA}{OGDA}
\DMO{\Unif}{Unif}
\DMO{\Tr}{Tr}
\nc{\Qu}{\ul{Q}}
\nc{\Qo}{\ol{Q}}
\nc{\Ro}{\ol{R}}
\nc{\Vu}{\ul{V}}
\nc{\Vo}{\ol{V}}
\nc{\RanQ}{\Delta Q}
\nc{\RanV}{\Delta V}
\nc{\clipQ}{\Delta \breve{Q}}
\nc{\frzQ}{\Delta \mathring{Q}}
\nc{\clipV}{\Delta \breve{V}}
\nc{\clipdelta}{\breve{\delta}}
\nc{\cliptheta}{\breve{\theta}}
\nc{\delmin}{\Delta_{{\rm min}}}
\nc{\delmins}[1]{\Delta_{{\rm min},{#1}}}
\nc{\gapfinal}[1]{\max \left\{ \frac{\frzQ_{{#1}}^{k^\st}(x,a)}{2H}, \frac{\delmin}{4H} \right\}}
\nc{\post}[2]{R({#1}; {#2})}
\nc{\posts}[3]{R_{#3}({#1}; {#2})}

\nc{\PPAD}{\textsf{PPAD}\xspace}
\nc{\PP}{\textsf{P}\xspace}
\nc{\FP}{\textsf{FP}\xspace}
\nc{\RFP}{\textsf{RFP}\xspace}
\nc{\RP}{\textsf{RP}\xspace}
\nc{\NP}{\textsf{NP}\xspace}
\nc{\MarkCCE}{\textsc{SparseMarkovCCE}\xspace}
\nc{\GenCCE}{\textsc{SparseCCE}\xspace}
\nc{\Nash}{\textsc{Nash}\xspace}

\nc{\algnst}[1]{\begin{align*}#1\end{align*}}
\nc{\algn}[1]{\begin{align}#1\end{align}}
\nc{\matx}[1]{\left(\begin{matrix}#1\end{matrix}\right)}
\renewcommand{\^}[1]{^{\scriptscriptstyle{(#1)}}}

\nc{\lgls}{\ell_{\mathrm{log}}}

\nc{\bel}[1]{\mathbf{b}({#1})}
\nc{\nbel}[1]{\bar{\mathbf{b}}({#1})}
\nc{\sbel}[2]{\mathbf{b}'_{#1}({#2})}
\nc{\nsbel}[2]{\bar{\mathbf{b}}'_{#1}({#2})}

\nc{\bv}{\mathbf{v}}
\nc{\bone}{\mathbf{1}}
\nc{\bX}{\mathbf{X}}
\nc{\bY}{\mathbf{Y}}
\nc{\bG}{\mathbf{G}}
\nc{\bz}{\mathbf{z}}
\nc{\bw}{\mathbf{w}}
\nc{\bA}{\mathbf{A}}
\nc{\bJ}{\mathbf{J}}
\nc{\bK}{\mathbf{K}}
\nc{\bb}{\mathbf{b}}
\nc{\ba}{\mathbf{a}}
\nc{\bc}{\mathbf{c}}
\nc{\bC}{\mathbf{C}}
\nc{\BR}{\mathbb R}
\nc{\BA}{\mathbb{A}}
\nc{\BC}{\mathbb C}
\nc{\bx}{\mathbf{x}}
\nc{\bS}{\mathbf{S}}
\nc{\bM}{\mathbf{M}}
\nc{\bR}{\mathbf{R}}
\nc{\bN}{\mathbf{N}}
\nc{\by}{\mathbf{y}}
\nc{\sy}{y}
\nc{\sx}{x}

\nc{\CA}{\mathscr{A}}
\nc{\CB}{\mathscr{B}}
\nc{\MO}{\mathcal O}
\nc{\MU}{\mathcal{U}}
\nc{\ME}{\mathcal{E}}
\nc{\MN}{\mathcal{N}}
\nc{\MK}{\mathcal{K}}
\nc{\MM}{\mathcal{M}}
\nc{\MS}{\mathcal{S}}
\nc{\MT}{\mathcal{T}}
\nc{\BF}{\mathbb F}
\nc{\BQ}{\mathbb Q}
\nc{\MX}{\mathcal{X}}
\nc{\MA}{\mathcal{A}}
\nc{\MD}{\mathcal{D}}
\nc{\MB}{\mathcal{B}}
\nc{\MZ}{\mathcal{Z}}
\nc{\MJ}{\mathcal{J}}
\nc{\MW}{\mathcal{W}}
\nc{\MR}{\mathcal{R}}
\nc{\MY}{\mathcal{Y}}
\nc{\BZ}{\mathbb Z}
\nc{\BN}{\mathbb N}
\nc{\ep}{\epsilon}
\nc{\vep}{\varepsilon}
\nc{\gapfn}[1]{\varepsilon_{#1}}
\nc{\ggapfn}[2]{\varphi_{#1}({#2})}
\nc{\epsahk}{\gapfn{0}}
\nc{\BH}{\mathbb H}
\nc{\BG}{\mathbb{G}}
\nc{\D}{\Delta}
\nc{\MF}{\mathcal{F}}
\nc{\One}[1]{\mathbbm{1}\{{#1}\}}
\nc{\bOne}{\mathbf{1}}
\nc{\Aopt}{\mathcal{A}^{\rm opt}}
\nc{\Amul}{\mathcal{A}^{\rm mul}}

\nc{\SP}{\mathsf P}
\nc{\SQ}{\mathsf Q}

\nc{\DO}{\accentset{\circ}{\D}}
\nc{\mf}{\mathfrak}
\nc{\mfp}{\mathfrak{p}}
\nc{\mfq}{\mf{q}}
\nc{\Sp}{\mbox{Spec}}
\nc{\Spm}{\mbox{Specm}}
\nc{\hookuparrow}{\mathrel{\rotatebox[origin=c]{90}{$\hookrightarrow$}}}
\nc{\hookdownarrow}{\mathrel{\rotatebox[origin=c]{-90}{$\hookrightarrow$}}}
\nc{\hra}{\hookrightarrow}
\nc{\tra}{\twoheadrightarrow}
\nc{\sgn}{{\rm sgn}}
\nc{\aut}{{\rm Aut}}
\nc{\Hom}{{\rm Hom}}
\nc{\img}{{\rm Im}}
\DMO{\id}{Id}
\DMO{\supp}{supp}
\DMO{\KL}{KL}
\nc{\kld}[2]{D_{\mathsf{KL}}({#1}\|{#2})}
\nc{\ren}[2]{D_2({#1}||{#2})}
\nc{\chisq}[2]{\chi^2({#1}||{#2})}
\nc{\tvd}[2]{D_{\mathsf{TV}}({#1}, {#2})}
\nc{\hell}[2]{{\rm H}^2({#1}, {#2})}
\DMO{\BSS}{BSS}
\DMO{\BES}{BES}
\DMO{\BGS}{BGS}
\DMO{\poly}{poly}
\nc{\indep}{\perp}
\DMO{\sink}{sink}
\nc{\fp}[1]{\MP_1({#1})}
\nc{\BO}{\mathbb{O}}
\nc{\BT}{\mathbb{T}}

\nc{\RR}{\mathbb{R}}
\nc{\Gradient}{\nabla}
\DMO{\diag}{diag}
\nc{\norm}[1]{\left \lVert #1 \right \rVert}
\nc{\EE}{\mathbb{E}}

\DMO{\PR}{Pr}

\nc{\E}{\mathbb{E}}
\nc{\ra}{\rightarrow}

  \icml{
\renewcommand{\paragraph}{\textbf}
  }

\arxiv{
\title{Hardness of Independent Learning and \\ Sparse Equilibrium Computation in Markov Games}

\author{  Dylan J. Foster\\{\small \texttt{dylanfoster@microsoft.com}} \and Noah Golowich\\{\small \texttt{nzg@mit.edu}}\\ \and Sham M.  Kakade\\{\small \texttt{sham@seas.harvard.edu}}}
\date{\today}
}

\begin{document}

\icml{
\twocolumn[
\icmltitle{Hardness of Independent Learning \\ and Sparse Equilibrium
  Computation in Markov Games}

\icmlsetsymbol{equal}{*}

\begin{icmlauthorlist}
\icmlauthor{Firstname1 Lastname1}{equal,yyy}
\icmlauthor{Firstname2 Lastname2}{equal,yyy,comp}
\icmlauthor{Firstname3 Lastname3}{comp}
\icmlauthor{Firstname4 Lastname4}{sch}
\icmlauthor{Firstname5 Lastname5}{yyy}
\icmlauthor{Firstname6 Lastname6}{sch,yyy,comp}
\icmlauthor{Firstname7 Lastname7}{comp}
\icmlauthor{Firstname8 Lastname8}{sch}
\icmlauthor{Firstname8 Lastname8}{yyy,comp}
\end{icmlauthorlist}

\icmlaffiliation{yyy}{Department of XXX, University of YYY, Location, Country}
\icmlaffiliation{comp}{Company Name, Location, Country}
\icmlaffiliation{sch}{School of ZZZ, Institute of WWW, Location, Country}

\icmlcorrespondingauthor{Firstname1 Lastname1}{first1.last1@xxx.edu}
\icmlcorrespondingauthor{Firstname2 Lastname2}{first2.last2@www.uk}

\icmlkeywords{Independent Learning, Nash equilibrium, Folk theorem}

\vskip 0.3in
]

\printAffiliationsAndNotice{}  %
}

\arxiv{\maketitle}

\begin{abstract}
\icml{
We consider the problem of decentralized multi-agent reinforcement
learning in Markov games. A key question is whether there are algorithms
      that, when run independently by all agents\icmlcut{ and run independently in a
      decentralized fashion}, lead to
      no-regret for each player, analogous to celebrated results for no-regret learning in normal-form games. While recent work has shown that such algorithms exist for
      restricted settings (e.g., when regret is defined with
      respect to deviations to \emph{Markov} policies), the question of whether
      independent no-regret learning can be achieved in the standard
      Markov game framework was open.
      We provide a decisive negative resolution
      to this problem, both from a computational and statistical perspective. We show
      that:
      \begin{enumerate}
      \item Under the assumption
        that $\PPAD$-hard problems cannot be solved in polynomial time, there is no polynomial-time algorithm
        that attains no-regret in general-sum
        Markov games when executed independently by all players, even when
        the game is known to the algorithm
        designer and the number of players is a small constant.
      \item When the game is unknown, no algorithm, efficient or
        otherwise, can achieve no-regret without
        observing exponentially many episodes in the number of players.\loose
      \end{enumerate}
      \icmlcut{      Perhaps surprisingly, our lower bounds hold even for
      seemingly easier setting in which all agents are controlled by a
      a centralized algorithm. }
      These results are proven via lower bounds for a simpler problem we refer to
      as \GenCCE, in which the goal is to compute a coarse
      correlated equilibrium that is ``sparse'' in the sense that it
      can be represented as a mixture of a small number of
      product policies. \loose
      
    }
    
\arxiv{
We consider the problem of decentralized multi-agent reinforcement
learning in Markov games.
A fundamental question is whether there exist algorithms
      that, when adopted by all agents and run independently in a
      decentralized fashion, lead to
      no-regret for each player, analogous to celebrated convergence results for no-regret learning in normal-form games. While recent work has shown that such algorithms exist for
      restricted settings (notably, when regret is defined with
      respect to deviations to \emph{Markovian} policies), the question of whether
      independent no-regret learning can be achieved in the standard
      Markov game framework was open.
      We provide a decisive negative resolution
      this problem, both from a computational and statistical perspective. We show
      that:
      \begin{enumerate}
      \item Under the widely-believed complexity-theoretic assumption
        that $\PPAD$-hard problems cannot be solved in polynomial time, there is no polynomial-time algorithm
        that attains no-regret in general-sum
        Markov games when executed independently by all players, even when
        the game is known to the algorithm
        designer and the number of players is a small constant.\dfcomment{should we be more clear about the whole
          two-player vs three-player thing (and P vs RP) here or is
          this ok?}\noah{adjusted accordingly}
      \item When the game is unknown, no algorithm---regardless of computational efficiency---can achieve no-regret without observing a number of episodes that is exponential in the
        number of players.
      \end{enumerate}
      
    Perhaps surprisingly, our lower bounds  hold even for
      seemingly easier setting in which all agents are controlled by a
      a centralized algorithm.
      They are proven via lower bounds for a simpler problem we refer to
      as \GenCCE, in which the goal is to compute by any
      means---centralized, decentralized, or otherwise---a coarse
      correlated equilibrium that is ``sparse'' in the sense that it
      can be represented as a mixture of a small number of
      ``product'' policies. The crux of our approach is a novel
      application of aggregation techniques from online learning
      \cite{vovk1990aggregating,cesa2006prediction}, whereby we show
      that any algorithm for the \GenCCE problem can be used to
      compute approximate Nash
      equilibria for non-zero sum normal-form games; this enables the
      application of well-known hardness results for Nash.
      }

\end{abstract}

\section{Introduction}
\label{sec:intro}
\noah{I feel like the paper was longer before, make sure nothing is getting cut out due to switch to arxiv}

The framework of \emph{multi-agent reinforcement learning
  (MARL)}, which describes settings in which multiple agents interact
in a dynamic environment, has played a key role in recent breakthroughs in artificial intelligence,
including the development of agents that approach or surpass human
performance in games such as Go \cite{silver2016mastering}, Poker
\cite{brown2018superhuman}, Stratego \cite{perolat2022mastering}, and
Diplomacy \cite{kramar2022negotiation,bakhtin2022human}. MARL also
shows promise for real-world multi-agent systems, including autonomous driving
\cite{shalevshwartz2016safe}, and cybersecurity
\cite{malialis2015distributed}, and economic policy
\cite{zheng2022ai}. These applications, where reliability is critical, necessitate the development of algorithms that are practical and efficient, yet provide strong formal guarantees and robustness.

Multi-agent reinforcement learning is typically studied using the framework of \emph{Markov games} (also known as \emph{stochastic games}) \cite{shapley1953stochastic}. In a Markov game, agents interact over a finite number of steps: at each step, each agent observes the \emph{state} of the environment, takes an \emph{action}, and observes a \emph{reward} which depends on the current state as well as the other agents' actions. Then the environment transitions to a new state as a function of the current state and the actions taken. An \emph{episode} consists of a finite number of such steps, and agents interact over the course of multiple episodes, progressively learning new information about their environment.
Markov games generalize the well-known model of \emph{Markov Decision Processes (MDPs)} \cite{puterman_markov_1994}, which describe the special case in which there is a single agent acting in a dynamic environment, and we wish to find a policy that maximizes its reward. %
By contrast, for Markov games, we typically aim to find a distribution over agents' policies which constitutes some type of \emph{equilibrium}.

  \subsection{Decentralized learning}

In this paper, we focus on the problem of \emph{decentralized} (or, independent) learning in Markov games. In decentralized MARL, each agent in the Markov game behaves independently, 
optimizing their policy myopically while treating the effects of the other agents as exogenous. Agents observe local
information (in particular, their own actions and rewards), but do not observe the actions of the other agents directly. %
Decentralized learning enjoys a number of desirable properties, including scalability\icmlcut{\xspace(computation is inherently linear in the number of agents)}, versatility\icmlcut{\xspace(by virtue of independence, algorithms can be applied in uncertain environments in which the nature of the interaction and number of other agents are not known)}, and practicality\icmlcut{\xspace(architectures for single-agent reinforcement learning can often be adapted directly)}. The central question we consider is whether there exist decentralized learning algorithms which, when employed by all agents in a Markov game, lead them to play near-equilibrium strategies over time.

\dfcomment{it might be good to mention below that we like CCE because Nash is intractable (since we will refer to nash in the next section of the intro).}\noah{did so in below paragraph}

Decentralized equilibrium computation in MARL is not well understood theoretically, and algorithms with provable guarantees are scarce. To motivate the challenges and most salient issues, it will be helpful to contrast with the simpler problem of decentralized learning in \emph{normal-form games}, which may be interpreted as Markov games with a single state. \icmlcut{Normal-form games enjoy a rich and celebrated theory of decentralized learning, dating back to Brown's work on fictitious play \cite{brown1949some} and Blackwell's theory of approachability \cite{blackwell1956analog}. }
Much of the modern work on decentralized learning in normal-form games centers on \emph{no-regret} learning, where agents select actions independently using \emph{online learning} algorithms \cite{cesa2006prediction} designed to minimize their \emph{regret} (that is, the gap between realized payoffs and the payoff of the best fixed action in hindsight).
In particular, a foundational result is that if each agent employs a no-regret learning strategy,
then the average of the agents' joint action distributions approaches a \emph{coarse correlated equilibrium (CCE)} for the normal-form game \cite{cesa2006prediction,hannan1957approximation,blackwell1956analog}. %
CCE is a natural relaxation of the foundational concept of \emph{Nash equilibrium}, which has the downside of being intractable to compute. 
On the other hand, there are many efficient algorithms that can achieve vanishing regret in a normal-form game, even when opponents select their actions in an arbitrary, potentially adaptive fashion, and thus converge to a CCE  \cite{vovk1990aggregating,littlestone1994weighted,cesa1997how,hart2000simple, syrgkanis2015fast}.

\dfcomment{I am on the fence about whether to shorten these two sentences on motivation for no-regret in normal-form games (or perhaps merge with the general motivation for decentralized learning in the first paragraph).}\noah{I think it's ok}

This simple connection between no-regret learning and decentralized convergence to equilibria has been influential in game theory, leading to numerous lines of research including
fast rates of convergence to equilibria
\cite{syrgkanis2015fast,chen2020hedging,daskalakis2021nearoptimal,anagnostides2022uncoupled}, price of anarchy bounds for smooth games \cite{roughgarden2015intrinsic}, and lower bounds on query and communication complexity for equilibrium computation \cite{fearnley2013learning,rubinstein2016settling,babichenko2017communication}.  Empirically, no-regret algorithms such as regret matching \cite{hart2000simple} and Hedge \cite{vovk1990aggregating,littlestone1994weighted,cesa1997how} have been used to compute equilibria that can achieve state-of-the-art performance in application domains such as Poker \cite{brown2018superhuman} and Diplomacy \cite{bakhtin2022human}. 
Motivated by these successes, we ask whether an analogous theory can be developed for Markov games. In particular:
\begin{center}
\emph{Are there efficient
  algorithms \icml{\\}for no-regret learning in Markov games?}
\end{center}
\icmlcut{
  Any Markov game can be viewed as a large normal-form game where each agent's action space consists of their exponentially-sized space of policies, and their utility function is given by their expected reward. Thus, any learning algorithm for normal-form games can also be applied to Markov games, but the resulting sample and computational complexities will be intractably large. Our goal is to explore whether more efficient decentralized learning guarantees can be established.
  }

\paragraph{Challenges for no-regret learning.}
In spite of active research effort and many promising pieces of progress \cite{jin2021vlearning,song2022when,mao2021provably,daskalakis2022complexity,erez2022regret}, no-regret learning guarantees for Markov games have been elusive. A barrier faced by naive algorithms is that
it is intractable to ensure no-regret against an \emph{arbitrary} adversary, both computationally \cite{bai2020nearoptimal,abbasi2013online} and statistically \cite{liu2022learning,kwon2021rl,foster2022complexity}.

Fortunately, many of the implications of no-regret learning (in particular, convergence to equilibria) do not require the algorithm to have sublinear regret against an arbitrary adversary, but rather only against other agents who are running the same algorithm independently. This observation has been influential in normal-form games, where the \arxiv{well-known }line of work on fast rates of convergence to equilibrium \cite{syrgkanis2015fast,chen2020hedging,daskalakis2021nearoptimal,anagnostides2022uncoupled} holds only in this more restrictive setting. This motivates the following relaxation to our central question.\loose
  \icml{\vspace{-5pt}}
\begin{problem}
  \label{pr:main-question}
Is there an efficient algorithm that, when adopted by all agents in a Markov game and run independently, leads to sublinear regret for each individual agent?
  \end{problem}
  \icml{\vspace{-5pt}}

  \paragraph{Attempts to address Problem
  \ref{pr:main-question}.}
Two recent lines of research have made progress toward addressing Problem
\ref{pr:main-question} and related questions. In one direction,
several recent papers have provided algorithms, including
\texttt{V-learning}
\cite{jin2021vlearning,song2022when,mao2021provably} and
\texttt{SPoCMAR} \cite{daskalakis2022complexity}, that do not achieve no-regret, but can nevertheless compute and then sample from a coarse
correlated equilibrium in a Markov game in a (mostly) \emph{decentralized}
fashion, with the caveat that they require a shared source of
  random bits as a mechanism to coordinate. Notably, \texttt{V-learning} depends only mildly on
the shared randomness: agents first play policies in a fully independent fashion (i.e.,
without shared randomness) according to a simple learning algorithm
for $T$ episodes, and use shared random bits only once learning finishes as part of a post-processing procedure to  extract a CCE policy. A question left open by these works, is whether the sequence of policies played by 
the \texttt{V-learning} algorithm in the initial independent phase can itself guarantee each
agent sublinear regret\icml{.}\arxiv{; this would eliminate the need for a separate
post-processing procedure and shared randomness.}

Most closely related to our work, \citet{erez2022regret} recently
showed that Problem \ref{pr:main-question} can be solved positively
for a restricted setting in which regret for each agent is defined as the maximum gain in
value they can achieve by deviating to a fixed \emph{Markov}
  policy. Markov policies are those whose choice of action depends
only on the current state as opposed to the entire history of
interaction. This notion of deviation is restrictive because in general, even when the opponent plays a sequence of Markov policies, the best response will be \emph{non-Markov}. In challenging settings that abound in practice, it is standard to consider non-Markov policies \cite{leibo2021scalable,agapiou2022melting}, since they often achieve higher value than Markov policies; we provide a simple example in Proposition \ref{prop:nonmarkov-deviation}.
Thus, while a regret guarantee with respect to the class of Markov policies (as in \cite{erez2022regret}) is certainly interesting, it may be too weak in general, and it is of great interest to understand whether Problem \ref{pr:main-question} can be answered positively in the general setting.\icmlfncut{We remark that the \texttt{V-learning} and \texttt{SPoCMAR} algorithms mentioned above do learn equilibria that are robust to deviations to non-Markov policies, though they do not address Problem \ref{pr:main-question} since they do not have sublinear regret.}

\subsection{Our contributions}
We resolve Problem
\ref{pr:main-question} in the negative, from both a computational and statistical perspective.

\dfcomment{It occurs to me that we never mention in the intro that we work in an episodic setting. I think it might be more clear if we do this, and then use ``episodes'' in place of iterations}\noah{good point, made this change}

\paragraph{Computational hardness.}
We provide two computational lower bounds (\cref{thm:markov-intro,thm:nonmarkov-intro}) which show that under standard complexity-theoretic assumptions, there is no efficient algorithm that runs for a polynomial number of episodes and guarantees each agent non-trivial (``sublinear'') regret when used in tandem by all agents. Both results hold even if the Markov game is explicitly known to the algorithm designer; \cref{thm:nonmarkov-intro} is stronger and more general, but applies only to $3$-player games, while \cref{thm:markov-intro} applies to $2$-player games, but only for agents restricted to playing Markovian policies.

To state our first result, \cref{thm:markov-intro}, we define a \emph{product Markov policy} to be a joint policy in which players choose their actions independently according to Markov policies (see Sections \ref{sec:prelim} and \ref{sec:markov} for formal
definitions). Note that if all players use independent no-regret algorithms to choose Markov policies at each episode, then their joint play at each round is described by a product Markov policy, since any randomness in each player's policy must be generated independently. %

\begin{theorem}[Informal version of Corollary \ref{cor:markov-formal-ppad}]
  \label{thm:markov-intro}
  If $\PPAD \neq \PP$, then there is no polynomial-time algorithm that, given the description of a 2-player Markov game, outputs a sequence of joint product Markov policies which guarantees each agent sublinear regret. %
\end{theorem}
\cref{thm:markov-intro} provides a decisive negative resolution to Problem
\ref{pr:main-question} under the assumption that $\PPAD\neq
\PP$,\footnote{Technically, the class we are denoting by $\PP$, namely of total search problems that have a deterministic polynomial-time algorithm, is sometimes denoted by $\mathsf{FP}$, as it is a search problem. We ignore this distinction.} which is standard in the theory of computational
complexity
\cite{papadimitriou1994complexity}.\footnote{\PPAD is the most well-studied complexity class in algorithmic game theory, and is widely believed to not admit polynomial time algorithms.
  Notably, the problem of computing a Nash equilibrium for normal-form games with two or more players is \PPAD-complete \cite{daskalakis2009complexity,chen2006computing,rubinstein2018inapproximability}.}
Beyond simply ruling out the existence of fully
decentralized no-regret algorithms, it rules out existence of
\emph{centralized} algorithms that compute a sequence of product 
policies for which each agent has sublinear regret, even if such a
sequence does not arise naturally as the result of agents
independently following some learning algorithm. Salient implications include:
\begin{itemize}
\item \cref{thm:markov-intro} provides a separation between Markov games and normal-form games, since standard no-regret algorithms for normal-form games i) run in polynomial time and ii) produce sequences of joint product policies that guarantee each agent sublinear regret. Notably, no-regret learning for  normal-form games is efficient whenever the number of agents is polynomial, whereas \cref{thm:markov-intro} rules out polynomial-time algorithms for as few as two agents.
  \item A question left open by the work of \citet{jin2021vlearning,song2022when,mao2021provably} was whether the sequence of policies played by 
    the \texttt{V-learning} algorithm during its independent learning phase can guarantee each
    agent sublinear regret.
  Since \texttt{V-learning}
  plays product Markov policies during the independent phase and is computationally efficient, \cref{thm:markov-intro} implies that these policies \emph{do not} enjoy sublinear regret (assuming $\PPAD \neq \PP$). 
\end{itemize}

Our second result, Theorem \ref{thm:nonmarkov-intro}, extends the
guarantee of Theorem \ref{thm:markov-intro} to the more general setting in
which agents can select arbitrary, potentially \emph{non-Markovian} policies at each episode. This comes at the cost of only providing hardness for
3-player games as opposed to 2-player games, as well as relying on the slightly stronger complexity-theoretic assumption that $\PPAD\nsubseteq\RP$.\footnote{We use $\RP$ to denote the class of total search problems for which there exists a polynomial-time randomized algorithm which outputs a solution with probability at least $2/3$, and otherwise outputs ``fail''.}
\begin{theorem}[Informal version of Corollary \ref{cor:nonmarkov-formal-ppad}]
  \label{thm:nonmarkov-intro}
    If $\PPAD \nsubseteq \RP$, then there is no polynomial-time algorithm that, given the description of a 3-player Markov game, outputs a sequence of joint product general policies (i.e., potentially non-Markov) which guarantees each agent sublinear regret. 
  \end{theorem}

\paragraph{Statistical hardness.}
\cref{thm:markov-intro,thm:nonmarkov-intro} rely on the widely-believed complexity
  theoretic assumption that $\PPAD$-complete problems cannot be solved
  in (randomized) polynomial time. Such a restriction is inherent if
  we assume that the game is known to the algorithm designer.
  To avoid complexity-theoretic assumptions, we consider a setting in which the Markov game is \emph{unknown} to the algorithm designer, and algorithms must learn about the game by executing policies (``querying'') and observing the resulting sequences of states, actions, and rewards.
  Our final result, \cref{thm:query-intro}, shows \emph{unconditionally} that, for $m$-player Markov games whose parameters
  are unknown, any algorithm computing a
  no-regret sequence as in Theorem \ref{thm:nonmarkov-intro} requires a number of queries that is exponential in $m$.
\begin{theorem}[Informal version of \cref{thm:statistical-lb}]
  \label{thm:query-intro}
  Given query access to a $m$-player Markov game, no algorithm that makes fewer than $2^{\Omega(m)}$ queries can output a sequence of joint product policies which guarantees each agent sublinear regret.
\end{theorem}
Similar to our computational lower bounds, \cref{thm:query-intro} goes far beyond decentralized algorithms, and rules out even centralized algorithms that compute a no-regret sequence by jointly controlling all players. The result provides another separation between Markov games and normal-form games, since standard no-regret algorithms for normal-form games can achieve sublinear regret using $\poly(m)$ queries for any $m$. The $2^{\Omega(m)}$ scaling in the lower bound, which does not rule out query-efficient algorithms when $m$ is constant, is to be expected for an unconditional result: If the game has only polynomially many parameters (which is the case for constant $m$), one can estimate all of the parameters using standard techniques \cite{jin2020reward}, then directly find a no-regret sequence.

\paragraph{Proof techniques: the \GenCCE problem.}
\icml{
Our proofs proceed via establishing lower bounds for a computational problem we refer to as \GenCCE. In the \GenCCE problem, the aim is to compute a CCE that can be represented as the mixture of a small number of product policies. See \cref{sec:markov,sec:nonmarkov} for detailed proof overview. 
  }

\icmlcut{
Rather than directly proving lower bounds for the problem of no-regret learning, we establish lower bounds for a simpler problem we refer to as \GenCCE. In the \GenCCE problem, the aim is to compute by any means---centralized, decentralized, or otherwise---a coarse correlated equilibrium that is ``sparse'' in the sense that it can be represented as the mixture of a small number of product policies. Any algorithm that computes a sequence of product policies with sublinear regret (in the sense of \cref{thm:nonmarkov-intro}) immediately yields an algorithm for the \GenCCE problem, as---using the standard connection between CCE and no-regret---the uniform mixture of the policies in the no-regret sequence forms a sparse CCE. Thus, any lower bound for the sparse \GenCCE problem yields a lower bound for computation of no-regret sequences.}

\icmlcut{
To provide lower bounds for the \GenCCE problem, we reduce from the problem of Nash equilibrium computation in normal-form games. We show that given any two-player normal-form game, it is possible to construct a Markov game (with two players in the case of \cref{thm:markov-intro} and three players in the case of \cref{thm:nonmarkov-intro}) with the property that i) the description length is polynomial in the description length of the normal-form game, and ii) any (approximate) \GenCCE for the Markov game can be efficiently transformed into a approximate Nash equilibrium for the normal-form game. With this reduction established, our computational lower bounds follow from celebrated \PPAD-hardness results for approximate Nash equilibrium computation in two-player normal-form games, and our statistical lower bounds follow from query complexity lower bounds for Nash. Proving the reduction from Nash to \GenCCE constitutes the bulk of our work, and makes novel use of aggregation techniques from online learning \cite{vovk1990aggregating,cesa2006prediction}, as well as techniques from the literature on anti-folk theorems in game theory \cite{borgs2008myth}.
}

\icml{
\paragraph{Organization.} \cref{sec:prelim} presents preliminaries, and \cref{sec:markov,sec:nonmarkov} provide our computational lower bounds. \textbf{\emph{Due to space limitations, our statistical lower bounds for multi-player Markov games are in \cref{sec:multiplayer}.}}
  }
\icmlcut{
  \subsection{Organization}
  \cref{sec:prelim} presents preliminaries on no-regret learning and equilibrium computation in Markov games and normal-form games. \cref{sec:markov,sec:nonmarkov,sec:multiplayer} present our main results:
  \begin{itemize}
  \item \cref{sec:markov,sec:nonmarkov} provide our computational lower bounds for no-regret in Markov games. \cref{sec:markov} gives a lower bound for the setting in which algorithms are constrained to play Markovian policies, and \cref{sec:nonmarkov} builds on the approach in this section to give a lower bound for general, potentially non-Markovian policies.
  \item \cref{sec:multiplayer} provides statistical (query complexity) lower bounds for multi-player Markov games.
  \end{itemize}
  Proofs are deferred to the appendix unless otherwise stated.
}
\dfcomment{Add a sentence here in bold telling people to look at Part I of the appendix.}\noah{did so}

\noahold{Main motivating work about no-regret learning in MDPs:
  \begin{itemize}
  \item In the worst case, it is computationally hard (due to hardness of Latent MDPs) and statistically hard (Chi's paper on adversarial MDPs, which uses statistical hardness of latent MDPs; also adversarial DEC paper).
  \item Question we consider (considered by many works in literature in other simpler online learning settings): is there some ``independent learning algorithm'' (broadly construed; see prelims for formal defn) so that if each player in a MG plays it, they have no regret? This avoids above upper bounds since the opponent is no longer acting adversarially.
\item Upper bounds against general policy classes: V-learning, spocmar give decentralized algorithms, though not no-regret (playing the output policy requires sharing random bits -- this is a very important point)!
\item Upper bounds against Markovian policies: recent work by Erez et al shows that you can get no-regret independent learning algorithm (i.e., don't share any random bits during the course of the algorithm) if (a) you know (or can explore) the MG, and (b) your comparison class is Markov policies. In particular, the policies' average is a CCE (nice!). 
\item But remember that the best response to a mixture of markov policies is in general non-Markov -- so can you remove the requirement of Erez et al of competing against class of Markov policies?

  We should drive the point home that in many settings (esp in game theory/economics settings, where people consider repeated games; also POMGs, cite Deepmind's melting pot), history dependent policies are the norm.
\item Our answer is no:

  (a) first show this if the individual policies played by the algorithm are Markov (product) policies. (Simpler proof)

  (b) More involved proof applies to when the inidividual policies played by the algorithm are general product policies.
  \end{itemize}
}

\noahold{Also mention statistical implications for $n$-player games.}

\noahold{Open questions:
  \begin{itemize}
  \item $O(1)$-player, $O(1)$-action games: can we show lower bounds? No, there's actually upper bounds: you can find Nash equilibria in polynomial time. 
  \item 2-player  games with non-Markov policies: here I expect there is an upper bound (for sparse CCE, as long as accuracy is $> 1/H$). 
  \end{itemize}
}
  \paragraph{Notation.} 
  For $n \in \BN$, we write $[n] := \{ 1, 2, \ldots, n\}$. For a finite set $\MT$, $\Delta(\MT)$ denotes the space of distributions on $\MT$. For an element $t\in\cT$, $\indic{t}\in\Delta(\cT)$ denotes the delta distribution that places probability mass $1$ on $t$. We adopt standard
        big-oh notation, and write $f=\bigoht(g)$ to denote that $f =
        \bigoh(g\cdot{}\max\crl*{1,\mathrm{polylog}(g)})$, with $\bigom(\cdot)$ and $\bigomt(\cdot)$ defined analogously.

  \section{Preliminaries}
  \label{sec:prelim}

This section contains preliminaries necessary to present our main
results. We first introduce the Markov game framework
(\cref{sec:markov_games}), then provide a brief review of normal-form
games (\cref{sec:normal_form}), and finally introduce the concepts of
coarse correlated equilibria and regret minimization (\cref{sec:regret}).

\subsection{Markov games}
\label{sec:markov_games}

  We consider general-sum Markov games in a finite-horizon, episodic
  framework. For $m \in \BN$, an $m$-player Markov game $\MG$ consists of a tuple $\MG = (\MS,H,  (\MA_i)_{i \in [m]} , \BP, (R_i)_{i \in [m]}, \mu)$, where:\loose
  \begin{itemize}
  \item $\MS$ denotes a finite state space and $H \in \BN$ denotes a finite time horizon. We write $S := |\MS|$. 
  \item For $i \in [m]$, $\MA_i$ denotes a finite action space for agent $i$. We let $\MA := \prod_{i=1}^m \MA_i$ denote the \emph{joint action space} and $\MA_{-i} := \prod_{i' \neq i} \MA_{i'}$.
    We denote joint actions in \arxiv{boldface}\icml{bold}, \arxiv{for example}\icml{e.g.}, $\ba = (a_1, \ldots, a_m) \in \MA$. We write $A_i \ldef |\MA_i|$ and $A \ldef |\MA|$. 
  \item $\BP = (\BP_1, \ldots, \BP_H)$ is the transition kernel, with
    each $\BP_h : \MS \times \MA \ra \Delta(\MS)$ denoting the kernel
    for step $h\in\brk{H}$. In particular, $\BP_h(s' | s, \ba)$ is the probability of transitioning to $s'$ from the state $s$ at step $h$ when agents play $\ba$.
  \item For $i \in [m]$ and $h \in [H]$, $R_{i,h} : \MS \times \MA \ra
    [-1/H,1/H]$ is the \arxiv{instantaneous reward function of}\icml{reward function for} agent
    $i$:\footnote{We assume that rewards lie in $\brk*{-1/H,1/H}$ for
      notational convenience, as this ensures that the cumulative reward for each
      episode lies in $\brk*{-1,1}$. This assumption is not important
      to our results.}
the reward agent $i$
    receives in state $s$ at step $h$ if agents play
    $\ba$ is \arxiv{given by }$R_{i,h}(s, \ba)$.\footnote{\arxiv{We consider Markov games in which the rewards at
      each step are a deterministic function of the state and action
      profile. While some works consider the more general case of
      stochastic rewards, since our main goal is to prove lower
      bounds, it is without loss for us to assume that rewards are deterministic.}\icml{We restrict our attention to Markov games in which the rewards at each step are a deterministic function of the state and action profile. Since our goal is to prove lower bounds, this is without loss.}}\loose
  \item $\mu \in \Delta(\MS)$ denotes the initial state distribution. 
  \end{itemize}
  An \emph{episode} in the Markov game proceeds as follows:
  \icml{
   the initial state $s_1$ is drawn from the initial state
    distribution $\mu$.
    Then, for each $h
    \leq H$, given state $s_h$, each agent $i$ plays action $a_{i,h} \in
    \MA_i$, and given the joint action profile $\ba_h = (a_{1,h}, \ldots,
    a_{m,h})$, each agent $i$ receives reward of $r_{i,h}= R_{i,h}(s_h,
    \ba_h)$ and the state of the system transitions to $s_{h+1} \sim
    \BP_h(\cdot | s_h, \ba_h)$.
    }
  \arxiv{
  \begin{itemize}
  \item the initial state $s_1$ is drawn from the initial state
    distribution $\mu$.
    \item For each $h
    \leq H$, given state $s_h$, each agent $i$ plays action $a_{i,h} \in
    \MA_i$, and given the joint action profile $\ba_h = (a_{1,h}, \ldots,
    a_{m,h})$, each agent $i$ receives reward of $r_{i,h}= R_{i,h}(s_h,
    \ba_h)$ and the state of the system transitions to $s_{h+1} \sim
    \BP_h(\cdot | s_h, \ba_h)$.
  \end{itemize}
  }
  We denote the tuple of agents' rewards at each step $h$ by $\br_h =
  (r_{1,h}, \ldots, r_{m,h})$, and refer to the resulting sequence
$\tau_H\ldef{}(s_1,\ba_1,\br_1),\ldots,(s_H,\ba_H,\br_H)$ as a
\emph{trajectory}. For $h\in\brk{H}$, we define the prefix of the
trajectory via
$\tau_h\ldef{}(s_1,\ba_1,\br_1),\ldots,(s_h,a_h,r_h)$.

\arxiv{\paragraph{Indexing.} }
  We use the following notation: for some quantity $x$ (e.g., action, reward, etc.) indexed by agents, i.e., $x = (x_1, \ldots, x_m)$, and an agent $i \in [m]$, we write $x_{-i} = (x_1, \ldots, x_{i-1}, x_{i+1}, \ldots, x_m)$ to denote the tuple consisting of all $x_{i'}$ for $i' \neq i$.

  \subsection{Policies and value functions}
  \label{sec:polval-prelim}
We now introduce the notion of policies and value functions for Markov
games. Policies are mappings from states (or sequences of states) to
actions for the agents. We consider several different types of policies, which play a
  crucial role in distinguishing the types of equilibria that are
  tractable and those that are intractable to compute
  efficiently. 

  \paragraph{Markov policies.} A randomized \emph{Markov policy} for
 agent $i$ is a sequence $\sigma_i = (\sigma_{i,1}, \ldots,
\sigma_{i,H})$, where $\sigma_{i,h} : \MS \ra \Delta(\MA_i)$. We
denote the space of randomized Markov policies for agent $i$ by
$\PiMarkov_i$. We write $\PiMarkov := \PiMarkov_1 \times \cdots \times
\PiMarkov_m$ to denote the space of \emph{product Markov policies}, which are
joint policies in which each agent $i$ independently follows a policy
in $\PiMarkov_i$. In
particular, a policy $\sigma\in \PiMarkov$ is specified by a
collection $\sigma = (\sigma_1, \ldots, \sigma_H)$, where $\sigma_h :
\MS \ra \Delta(\MA_1) \times \cdots \times \Delta(\MA_m)$.  We
additionally define $\PiMarkov_{-i} := \prod_{i' \neq i}
\PiMarkov_{i'}$, and for a policy $\sigma \in \PiMarkov$, write
$\sigma_{-i}$ to denote the collection of mappings $\sigma_{-i} =
(\sigma_{-i,1}, \ldots, \sigma_{-i,H})$, where $\sigma_{-i,h} : \MS
\ra \prod_{i' \neq i} \Delta(\MA_{i'})$ denotes the tuple of all but
player $i$'s policies.

When the Markov game $\cG$ is clear from context, for a policy
$\sigma\in\PiMarkov$ we let $\bbP_{\sigma}\brk*{\cdot}$
denote the law of the trajectory $\tau$ when players select actions
via $\ba_{h}\sim{}\sigma(s_h)$, and let $\En_{\sigma}\brk*{\cdot}$
denote the corresponding expectation.

    \paragraph{General (non-Markov) policies.} In addition to Markov
    policies, we will consider general \emph{history-dependent} (or,
    \emph{non-Markov}) policies, which select actions based on the
    \emph{entire sequence of states and actions} observed up the
    current step. To streamline notation, %
  for $i \in [m]$, let $\tau_{i,h} = (s_1, a_{i,1}, r_{i,1}, \ldots,
  s_h, a_{i,h}, r_{i,h})$ denote the history of agent $i$'s states,
  actions, and reward up to step $h$. Let $\CH_{i,h} = (\MS \times
  \MA_i \times [0,1])^h$ denote the space of all possible histories of
  agent $i$ up to step $h$. %
  For $i \in [m]$, a \emph{randomized general (i.e., non-Markov)
    policy of agent $i$} is a collection of mappings $\sigma_i =
  (\sigma_{i,1}, \ldots, \sigma_{i,H})$ where $\sigma_{i,h} :
  \CH_{i,h-1} \times \MS \ra \Delta(\MA_i)$ is a mapping that takes
  the history observed by agent $i$ up to step $h-1$ and the current
  state and outputs a distribution over actions for agent
  $i$.\loose

  We denote by $\Pirndrnd_i$ the space of randomized general policies
  of agent $i$, and further write $\Pirndrnd := \Pirndrnd_1 \times
  \cdots \times \Pirndrnd_m$ to denote the space of product general
  policies; note that $\PiMarkov_i \subset \Pirndrnd_i$ and $\PiMarkov
  \subset \Pirndrnd$. In particular, a policy $\sigma \in \Pirndrnd$ is
  specicfied by a collection $(\sigma_{i,h})_{i \in [m], h \in [H]}$,
  where $\sigma_{i,h} : \CH_{i,h-1} \times \MS \ra
  \Delta(\MA_i)$. When agents play according to a general policy
  $\sigma\in \Pirndrnd$, at each step $h$, each agent, given the
  current state $s_h$ and their history $\tau_{i,h-1} \in
  \CH_{i,h-1}$, chooses to play an action $a_{i,h} \sim
  \sigma_{i,h}(\tau_{i,h-1}, s_h)$, independently from all other
  agents. For a policy $\sigma\in\Pirndrnd$, we let
  $\bbP_{\sigma}\brk*{\cdot}$ and $\En_{\sigma}\brk*{\cdot}$
denote the law and expectation operator for the trajectory $\tau$ when players select actions
via $\ba_{h}\sim{}\sigma(\tau_{h-1}, s_h)$, and write
  $\sigma_{-i}$ to denote the collection of policies of all agents but
  $i$, i.e., $\sigma_{-i} = (\sigma_{j,h})_{h \in [H], j \in
    [m]\backslash \{ i \}}$.

  We will also consider distributions over product randomized general
  policies, namely elements of $\Delta(\Pirndrnd)$.\footnote{When
    $\cT$ is not a finite set, we take $\Delta(\cT)$ to be the set
    of Radon probability measures over $\cT$ equipped with the Borel $\sigma$-algebra.} We will refer to elements of $\Delta(\Pirndrnd)$ as \emph{distributional policies}. To play \arxiv{according to some}\icml{a} distributional policy $\distp\in\Delta(\Pirndrnd)$, agents draw a randomized policy $\sigma \sim \distp$ (so that $\sigma \in \Pirndrnd$) and then play \arxiv{according to }$\sigma$.

  \icmlcut{
  \begin{remark}[Alternative definition for randomized general policies]
    Instead of defining distributional policies as above, one might alternatively define $\Pirndrnd_i$ as the set of distributions over agent $i$'s deterministic general policies, namely as the set  $\Delta(\Pidet_i)$. We show in Section \ref{sec:equivalence} that this alternative definition is equivalent to our own: in particular, there is a mapping from $\Pirndrnd$ to $\Delta(\Pidet_1) \times \cdots \times \Delta(\Pidet_m)$ so that, for any Markov game, any policy $\sigma \in \Pirndrnd$ produces identically distributed trajectories to its corresponding policy in $\Delta(\Pidet_1) \times \cdots \times \Delta(\Pidet_m)$. Further, this mapping is one-to-one if we identify policies that produce the same distributions over trajectories for all Markov games. %
  \end{remark}
  }

  \icmlcut{
    \paragraph{Deterministic policies.}
  It will be helpful to introduce notation for
  \emph{deterministic} general (non-Markov) policies, which correspond
  to the special case of randomized policies where each policy $\sigma_{i,h}$ exclusively maps to singleton distributions. {In particular, a deterministic general policy of agent $i$ is 
  a collection of mappings $\pi_i =
  (\pi_{i,1}, \ldots, \pi_{i,H})$, where $\pi_{i,h} : \CH_{i,h-1} \times
  \MS \ra \MA_i$.} %
  We denote by $\Pidet_i$ the space of deterministic
  general policies of agent $i$, and further write $\Pidet:= \Pidet_1
  \times \cdots \times \Pidet_m$ to denote the space of \emph{joint
    deterministic policies}. {We use the convention throughout that
  deterministic policies are denoted by the letter $\pi$, whereas
  randomized policies are denoted by $\sigma$.}
}

  \paragraph{Value functions.} For a general policy
  $\sigma \in \Pirndrnd$, we define the value function for agent $i \in
  [m]$ as
    \icml{
$V_i^\sigma := \E_\sigma \left[ \sum_{h=1}^H R_{i,h}(s_h, \ba_h) \ \mid s_1 \sim \mu \right]$;%
  }
  \arxiv{
  \begin{align}
V_i^\sigma := \E_\sigma \left[ \sum_{h=1}^H R_{i,h}(s_h, \ba_h) \ \mid s_1 \sim \mu \right];\label{eq:valfn}
  \end{align}
  }
  this represents the expected reward that agent $i$ receives when
  each agent chooses their actions via $a_{i,h} \sim \sigma_h(\tau_{i,h-1}, s_h)$. For a distributional policy $\distp \in \Delta(\Pirndrnd)$, we extend this notation by defining $V_i^{\distp} := \E_{\sigma\sim \distp} [V_i^\sigma]$.

  \subsection{Normal-form games}
  \label{sec:normal_form}
  \noah{I feel like it might make sense to swap this subsection and the following one -- currently we're going back and forth between MGs and normal form.}
  \dfedit{To motivate the solution concepts we consider for Markov games, let us first revisit the notion of normal-form games, which may be interpreted as Markov games with a single state.}
  For $m,n \in \BN$, an \emph{$m$-player $n$-action normal-form game
    $G$} is specified by a tuple of $m$ \emph{\utility tensors} $M_1,
  \ldots, M_m \in [0,1]^{n \times \cdots \times n}$, where each tensor
  is of order $m$ (i.e., has $n^m$ entries). We will write $G = (M_1,
  \ldots, M_m)$.  We assume for simplicity that each player has the
  same number $n$ of actions, and identify each player's action space
  with $[n]$. Then an an action profile is specified by $\ba \in
  [n]^m$; if each player acts according to $\ba$, then the \utility for
  player $i \in [m]$ is given by $(M_i)_{\ba} \in [0,1]$. 
    Our hardness results will use the standard notion of \emph{Nash equilibrium} in normal-form games.       We define the \emph{$m$-player $(n,\ep)$-\Nash problem} to be the
  problem of computing an $\ep$-approximate Nash equilibrium of a given $m$-player
  $n$-action normal-form game. (See Definition \ref{def:nash} for a formal definition of $\ep$-Nash equilibrium.)  A celebrated result is that Nash equilibria are \PPAD-hard to approximate, i.e., the 2-player $(n, n^{-c})$-\Nash problem is \PPAD-hard for any constant $c > 0$  \cite{daskalakis2009complexity,chen2006computing}. We refer the reader to Section \ref{sec:nash-prelims} for further background on these concepts.

\arxiv{\subsection{Markov games: Equilibria, no-regret, and
    independent learning}}
\icml{\subsection{Markov games: Equilibria and no-regret}}
  \label{sec:regret}

We now turn our focus back to Markov games, and introduce the main
solution concepts we consider, as well as the notion of no-regret. Since computing Nash equilibria is intractable even for normal-form
games, much of the work on efficient equilibrium computation has
focused on alternative notions of equilibrium, notably \emph{coarse
  correlated equilibria}\icmlcut{{} and \emph{correlated equilibria}. We focus on coarse correlated equilibria: being a superset of
correlated equilibria, any lower bound for computing a coarse
correlated equilibrium implies a lower bound for computing a
correlated equilibrium}.

For a distributional policy $\distp \in \Delta(\Pirndrnd)$ and a
randomized policy $\sigma_i' \in \Pirndrnd_i$ of player $i$, we let
$\sigma_i' \times \distp_{-i} \in \Delta(\Pirndrnd)$ denote the
distributional policy which is given by the distribution of
$(\sigma_i', \sigma_{-i}) \in \Pirndrnd$ for  $\sigma \sim \distp$
(and $\sigma_{-i}$ denotes the marginal of $\sigma$ on all players but
$i$). For $\sigma \in \Pirndrnd$, we write $\sigma_i' \times
\sigma_{-i}$ to denote the policy given by $(\sigma_i', \sigma_{-i}) \in \Pirndrnd$.
 Let us fix a Markov game $\MG$, which in particular determines the players' value functions $V_i^\sigma$\arxiv{\xspace as in (\ref{eq:valfn})}.
\begin{defn}[Coarse correlated equilibrium]
  \label{def:cce}
  For $\ep > 0$, a distributional policy $\distp \in \Delta(\Pirndrnd)$ is defined to be an \emph{$\ep$-coarse correlated equilibrium (CCE)} if for each $i \in [m]$, it holds that \icml{$\max_{\sigma_i' \in \Pirndrnd_i} V_i^{\sigma_i' \times \distp_{-i}} - V_i^\distp \leq \ep$}.
  \arxiv{
  \begin{align}
\max_{\sigma_i' \in \Pirndrnd_i} V_i^{\sigma_i' \times \distp_{-i}} - V_i^\distp \leq \ep\nonumber.
  \end{align}
  }
\icmlcut{The maximizing policy $\sigma_i'$ can always be chosen to be determinimistic, so $\distp$ is an $\ep$-CCE if and only if $\max_{\pi_i \in \Pidet_i} V_i^{\pi_i \times \distp_{-i}} - V_i^\distp \leq \ep$. }
\end{defn}

Coarse correlated equilibria can be computed efficiently for both
normal-form games and Markov games, and are fundamentally connected to
the notion of no-regret and independent learning, which we now introduce.

\paragraph{Regret.}
For a policy $\sigma \in \Pirndrnd$, we denote the distributional
policy which puts all its mass on $\sigma$ by $\indic{\sigma} \in
\Delta(\Pirndrnd)$. Thus $\frac 1T \sum_{t=1}^T \indic{\sigma\^t} \in
\Delta(\Pirndrnd)$ denotes the distributional policy which randomizes
uniformly over the $\sigma\^t$.  We define \emph{regret} as follows.
  \begin{defn}[Regret]
    \label{def:regret}
    Consider a sequence of policies $\sigma\^1, \ldots, \sigma\^T \in \Pirndrnd$. For $i \in [m]$, the \emph{regret of agent $i$} with respect to this sequence is defined as:
    \begin{align}
\arxiv{\Reg_{i,T} = }\Reg_{i,T}(\sigma\^1, \ldots, \sigma\^T) = \max_{\sigma_i \in \Pirndrnd_i} \sum_{t=1}^T  V_i^{\sigma_i \times \sigma_{-i}\^t} -  V_i^{\sigma\^t} \label{eq:reg-defn}.
    \end{align}
\icmlcut{    In (\ref{eq:reg-defn}) the maximum over $\sigma_i \in \Pirndrnd_i$ is always achieved by a deterministic general policy, so we have $\Reg_{i,T} = \max_{\pi_i \in \Pidet_i} \sum_{t=1}^T \prn[\big]{ V_i^{\pi_i \times \sigma_{-i}\^t} - V_i^{\sigma\^t} }$. }
\end{defn}

\icml{It is immediate from the above definitions that a sequence of policies $\sigma\^1, \ldots, \sigma\^T \in \Pirndrnd$ satisfies $\Reg_{i,t}(\sigma\^1, \ldots, \sigma\^T) \leq \ep \cdot T$ if and only if the distributional policy $\ol \sigma := \frac 1T \sum_{t=1}^T \indic{\sigma\^t}$ is an $\ep$-CCE (stated formally in Fact \ref{fac:no-regret-cce} in the appendix).}

\icmlcut{The following standard result
shows that the uniform average of any no-regret sequence forms an
approximate coarse correlated equilibrium.

\begin{fact}[No-regret is equivalent to CCE]
  \label{fac:no-regret-cce}
 Suppose that a sequence of policies $\sigma\^1, \ldots, \sigma\^T\in
 \Pirndrnd$ satisfies $\Reg_{i,T}(\sigma\^1, \ldots, \sigma\^T ) \leq
 \ep \cdot T$ for each $i \in [m]$. Then the uniform average of these
 $T$ policies, namely the distributional policy $\ol \sigma :=
 \frac{1}{T} \sum_{t=1}^T \indic{\sigma\^t} \in \Delta(\Pirndrnd)$, is
 an $\ep$-CCE.

Likewise if a sequence of policies $\sigma\^1, \ldots, \sigma\^T\in
\Pirndrnd$ has the property that the distributional policy $\ol \sigma
:= \frac{1}{T} \sum_{t=1}^T \indic{\sigma\^t} \in \Delta(\Pirndrnd)$,
is an $\ep$-CCE, then we have $\Reg_{i,T}(\sigma\^1, \ldots, \sigma\^T ) \leq \ep \cdot T$ for all $i \in [m]$.
\end{fact}
}

\paragraph{No-regret learning.}
\icmlcut{Fact \ref{fac:no-regret-cce} is an immediate consequence of Definitions \ref{def:regret} and \ref{def:cce}. }
A standard approach to
decentralized equilibrium computation, which exploits Fact \ref{fac:no-regret-cce}, is to select
$\sigma\ind{1},\ldots,\sigma\ind{T} \in\Pirndrnd$ using independent \emph{no-regret
  learning} algorithms. A no-regret learning algorithm for player $i$ selects $\sigma_i\ind{t} \in \Pirndrnd_i$ based on the realized trajectories $\tau\ind{1}_{i,H},\ldots,\tau\ind{t-1}_{i,H} \in \CH_{i,H}$ that player $i$ observes over the course of play,\icmlfncut{An alternative model allows for player $i$ to have knowledge of the previous joint policies $\sigma\^1, \ldots, \sigma\^{t-1}$, when selecting $\sigma_i\^t$.}\xspace 
 but
with no knowledge of $\sigma\ind{t}_{-i}$, so as to ensure that
no-regret is achieved: $\Reg_{i,T}(\sigma\^1, \ldots, \sigma\^T ) \leq
\ep \cdot T$. If each player $i$ uses their own, independent no-regret
learning algorithm, this approach yields \emph{product policies}
$\sigma\ind{t}=\sigma_1\ind{t}\times\cdots\times\sigma\ind{t}_m$, and the uniform average of the $\sigma\^t$ yields a CCE as long as all of the players
can keep their regret small.\footnote{In \cref{sec:discussion}, we
  discuss the implications of relaxing the stipulation that
  $\sigma\^t$ be product policies \dfedit{(for example, by allowing
    the use of shared randomness, as in \texttt{V-learning})}. In short, allowing $\sigma\^t$ to be non-product
  essentially trivializes the problem.
  \label{fn:shared-randomness}}

\icml{
For the special case of normal-form games, there are several efficient
algorithms, %
which---when run independently---ensure that each player's regret after $T$ episodes is bounded
above by $O(\sqrt{T})$ (that is $\eps=\bigoh(1/\sqrt{T})$), even when
the other players' actions are chosen
\emph{adversarially}.
}

\arxiv{
For the special case of normal-form games, the no-regret learning approach
has been fruitful.
There are several efficient algorithms, including regret matching \cite{hart2000simple}, Hedge (also known as exponential weights) \cite{vovk1990aggregating,littlestone1994weighted,cesa1997how}, and generalizations of Hedge based on the
\emph{follow-the-regularized-leader (FTRL)} framework \cite{shalevshwartz2012online},
which ensure that each player's regret after $T$ episodes is  bounded
above by $O(\sqrt{T})$ (that is $\eps=\bigoh(1/\sqrt{T})$), even when
the other players' actions are chosen
\emph{adversarially}.
All of these guarantees, which bound regret by a sublinear function in
$T$, lead to efficient, decentralized computation of approximate
coarse correlated equilibrium in normal-form games. The success of
this motivates our central question, which is whether similar
guarantees may be established for Markov games. In particular, a
formal version of Problem \ref{pr:main-question} asks: \emph{Is there an efficient
algorithm that, when adopted by all agents in a Markov game and run
independently, ensures that for all $i$, $\Reg_{i,T}\leq\eps\cdot{}T$
for some $\eps=o(1)$?} \noah{This last sentence feels somewhat
repetitive -- I'm not sure if the change from Problem
\ref{pr:main-question}, namely the $\ep = o(1)$, adds much?}
\dfcomment{I wonder if we should formally define the model for
  \emph{independent} no-regret algorithms
  somewhere (in particular, what info they have access to when selecting
  $\sigma\ind{t}_i$). I kind of sketched this above, but it's a bit informal/vague.}\noah{there's not necessarily a fixed model, so not sure it makes sense to do so more formally -- I think what you wrote above is good.}
}

\arxiv{\section{Lower bound for Markovian algorithms}}
\icml{\section{Lower bound for Markovian algorithms}}
\label{sec:markov}
In this section we prove Theorem \ref{thm:markov-intro} (restated
formally below as Theorem \ref{thm:markov-formal}), establishing that
in two-player Markov games, there is no computationally efficient
algorithm that computes a sequence $\sigma\^1, \ldots, \sigma\^T$ of product Markov policies so that each player has small regret
under this sequence. This section serves as a warm-up for our results in \cref{sec:nonmarkov}, which remove the assumption that $\sigma\ind{1},\ldots,\sigma\ind{T}$ are Markovian.

\arxiv{\subsection{\MarkCCE problem and computational model}}
\icml{\subsection{\MarkCCE and computational model}}
\label{sec:model}
As discussed in the introduction, our lower bounds for no-regret learning are a consequence of lower bounds for the \emph{\GenCCE} problem. In what follows, we formalize this problem (specifically, the Markovian variant, which we refer to as \emph{\MarkCCE}), as well as our computational model.

\paragraph{Description length for Markov games (constant $m$).}
Given a Markov game $\MG$, we let $\beta(\MG)$ denote the maximum number
of bits needed to describe any of the rewards $R_{i,h}(s, \ba)$ or transition probabilities $\BP_h(s' | s,\ba)$ in binary.\footnote{We emphasize that $\beta(\cG)$ is defined as the maximum number of bits required by any particular $(s,\ba)$ pair, not the total number of bits required for \emph{all} $(s,\ba)$ pairs.} We define $|\MG| := \max\{ S, \max_{i \in [m]} A_i, H, \beta(\MG) \}$. The interpretation of $|\MG|$ depends on the number of players $m$: If $m$ is a constant (as will be the case in the current section and Section \ref{sec:nonmarkov}), then $|\MG|$ should be interpreted as the description length of the game $\MG$, up to polynomial factors. In particular, for constant $m$, %
the game $\MG$ can be described using $|\MG|^{O(1)}$
bits. %
In Section \ref{sec:multiplayer}, we discuss the interpretation of $|\MG|$ when $m$ is large.

\paragraph{The \MarkCCE problem.}  
From Fact \ref{fac:no-regret-cce}, we know that the problem of computing a sequence $\sigma\^1, \ldots, \sigma\^T$ of
joint product Markov policies for which each player has at most $\eps\cdot{}T$ regret is equivalent to computing a sequence $\sigma\^1, \ldots, \sigma\^T$ for which the uniform mixture forms an $\eps$-approximate CCE. We define $(T, \ep)$-\MarkCCE as the computational problem of computing such a CCE directly.

\begin{defn}[\MarkCCE problem]
  \label{def:markcce}
  For an $m$-player Markov game $\MG$ and parameters $T \in \BN$ and $\ep > 0$ (which may depend on the size of the game $\MG$), $(T, \ep)$-\MarkCCE is the problem of finding a sequence $\sigma\^1, \ldots, \sigma\^T$, with each $\sigma\^t \in \PiMarkov$, such that the distributional policy $\ol \sigma = \frac{1}{T} \sum_{t=1}^T \indic{\sigma\^t} \in \Delta(\Pirndrnd)$ is an $\ep$-CCE of $\MG$ (or equivalently, such that for all $i \in [m]$, $\Reg_{i,T}(\sigma\^1, \ldots, \sigma\^T) \leq \ep \cdot T$).
\end{defn}
Decentralized learning algorithms naturally lead to solutions to the \MarkCCE problem. In particular, consider any decentralized protocol which runs for $T$ episodes, where at each timestep $t\in\brk{T}$, each player $i \in [m]$ chooses a Markov policy $\sigma_i\^t \in \PiMarkov_i$ to play, without knowledge of the other players' policies $\sigma_{-i}\ind{t}$ (but possibly using the history); any strategy in which players independently run online learning algorithms falls under this protocol. If each player experiences overall regret at most $\ep \cdot T$, then the sequence $\sigma\ind{1},\ldots,\sigma\ind{T}$ is a solution to the $(T,\ep)$-\MarkCCE problem. However, one might expect the $(T,\ep)$-\MarkCCE problem to be much easier than decentralized learning, since it allows for algorithms that produce $(\sigma\^1, \ldots, \sigma\^T)$ satisfying the constraints of Definition \ref{def:markcce} in a centralized manner. The main result of this section, Theorem \ref{thm:markov-formal}, rules out the existence of \emph{any} efficient algorithms, including centralized ones, that solve the \MarkCCE problem.

{Before moving on, let us give a sense for what sort of scaling one should expect for the parameters $T$ and $\eps$ in the $(T,\eps)$-\MarkCCE problem.}\xspace First, we note that there always exists a solution to the $(1,0)$-\MarkCCE problem in a Markov game, which is given by a (Markov) Nash equilibrium of the game; of course, Nash equilibria are intractable to compute in general.\icmlfncut{Such a Nash equilibrium can be seen to exist by using backwards induction to specify the player's joint distribution of play at each state at steps $H, H-1, \ldots, 1$.}
For the special case of normal-form games (where there is only a single state, and $H=1$), no-regret learning (e.g., Hedge) yields a computationally efficient solution to the $(T, \til O(1/\sqrt{T}))$-\MarkCCE problem, where the $\til O(\cdot)$ hides a $\max_i \log |A_i|$ factor. Refined convergence guarantees of \citet{daskalakis2021nearoptimal,anagnostides2022uncoupled} improve upon this result, and yield an efficient solution to the $(T, \til O(1/T))$-\MarkCCE problem.

\subsection{Main result}
\label{sec:markov-formal}
\begin{theorem}
  \label{thm:markov-formal}
There is a constant $C_0 > 1$ so that the following holds. Let $n \in \BN$ be given, and let $\Tn \in \BN$ and $\epn > 0$ satisfy $\Tn < \exp(\epn^2 \cdot n^{1/2}/2^{5})$. Suppose there is an algorithm that, given the description of any 2-player Markov game $\MG$ with $|\MG| \leq n$, solves the $(\Tn, \epn)$-\MarkCCE problem in time $U$, for some $U \in \BN$. %
  Then, for each $n \in \BN$, the 2-player $(\lfloor n^{1/2}\rfloor, 4 \cdot \epn)$-\Nash problem (Definition \ref{def:nash}) can be solved in time $(n\Tn U)^{C_0}$. \loose
\end{theorem}
\dfcomment{based on reviewer suggestion: add remark on why restriction to $T\leq\exp(.)$ is natural.}\noah{added it}

We emphasize that the range $T < \exp(n^{O(1)})$ ruled out by Theorem \ref{thm:markov-formal} is the most natural parameter regime, since the runtime of any decentralized algorithm which runs for $T$ episodes and produces a solution to the \MarkCCE problem is at least linear in $T$. 
Using that 2-player $(n, \ep)$-\Nash is \PPAD-complete for $\ep = n^{-c}$ (for any $c > 0$) \citep{daskalakis2009complexity,chen2006computing,rubinstein2018inapproximability}, we obtain the following corollary.
\begin{corollary}[\MarkCCE is \PPAD-complete]
  \label{cor:markov-formal-ppad}
For any constant $C > 4$, if there is an algorithm which, given the description of a 2-player Markov game $\MG$, solves the $(|\MG|^C, |\MG|^{-\frac{1}{C}})$-\MarkCCE problem  in time $\poly(|\MG|)$, then $\PPAD = \PP$.
\end{corollary}
The condition $C > 4$ in Corollary \ref{cor:markov-formal-ppad} is set to ensure that $|\MG|^C < \exp(|\MG|^{-2/C} \cdot \sqrt{|\MG|} / 2^6)$ for sufficiently large $|\MG|$, so as to satisfy the condition of Theorem \ref{thm:markov-formal}. 
Corollary \ref{cor:markov-formal-ppad} rules out the existence of a polynomial-time algorithm that solves the \MarkCCE problem with accuracy $\eps$ polynomially small and $T$ polynomially large in $\abs{\cG}$.
\icmlcut{Using a stronger complexity-theoretic assumption, the Exponential Time Hypothesis for \PPAD \cite{rubinstein2016settling}, we can obtain a stronger hardness result which rules out efficient algorithms even when 1)  the accuracy $\eps$ is constant and 2) $T$ is quasipolynomially large.\icmlfncut{This is a consequence of the fact that for some absolute constant $\ep_0 > 0$, there are no polynomial-time algorithms for computing $\ep_0$-Nash equilibria in 2-player normal-form games under the Exponential Time Hypothesis for \PPAD \cite{rubinstein2016settling}.}
\begin{corollary}[ETH-hardness of \MarkCCE]
There is a constant $\ep_0 > 0$ such that if there exists an algorithm that solves the $(|\MG|^{o(\log |\MG|)}, \ep_0)$-\MarkCCE problem in $|\MG|^{o(\log |\MG|)}$ time, then the Exponential Time Hypothesis for \PPAD fails to hold.
\end{corollary}
}

\paragraph{Proof overview.} The proof of Theorem \ref{thm:markov-formal} is based on a reduction, which shows that any algorithm that efficiently solves the $(T, \ep)$-\MarkCCE problem, for $T$ not too large, can be used to efficiently compute an approximate Nash equilibrium of any given normal-form game. In particular, fix $n_0\in\bbN$, and let a 2-player normal form game $G$ with $n_0$ actions be given. We construct a Markov game $\MG = \MG(G)$ with horizon $H = n_0$ and action sets identical to those of the game $G$, i.e., $\MA_1 = \MA_2 = [n_0]$. The state space of $\MG$ consists $n_0^2$ states, which are indexed by joint action profiles; the transitions are defined so that the value of the state at step $h$ encodes the action profile taken by the agents at step $h-1$.\footnote{For technical reasons, this only is the case for even values of $h$; we discuss further details in the full proof in Section \ref{sec:markov-proof}.} At each state of $\MG$, the reward functions are given by the payoff matrices of $G$, scaled down by a factor of $1/H$ (which ensures that the rewards received at each step belong to $[0,1/H]$). In particular, the rewards and transitions out of a given state do not depend on the identity of the state, and so $\MG$ can be thought of as a repeated game where $G$ is played $H$ times. The formal definition of $\MG$ is given in Definition \ref{def:mg-g}.

Fix any algorithm for the \MarkCCE problem, and recall that for each step $h$ and state $s$ for $\MG$, $\sigma\^t_h(s) \in \Delta(\MA_1) \times \Delta(\MA_2)$ denotes the joint action distribution taken in $s$ at step $h$ for the sequence of $\sigma\ind{1},\ldots,\sigma\ind{T}$ produced by the algorithm. The bulk of the proof of \cref{thm:markov-formal} consists of proving a key technical result, Lemma \ref{lem:2nash-main}, which states that if $\sigma\ind{1},\ldots,\sigma\ind{T}$ indeed solves $(T,\eps)$-\MarkCCE, then there exists some tuple $(h,s,t)$ such that $\sigma_h\^t(s)$ is an approximate Nash equilibrium for $G$. With this established, it follows that we can find a Nash equilibrium efficiently by simply trying all $HST$ choices for $(h,s,t)$.

To prove Lemma \ref{lem:2nash-main}, we reason as follows. Assume that $\ol \sigma := \frac 1T \sum_{t=1}^T \indic{\sigma\^t}\in\Delta(\Pirndrnd)$ is an $\eps$-CCE. If, by contradiction, none of the distributions $\crl*{\sigma_h\^t(s)}_{h\in\brk{H},s\in\cS,t\in\brk{T}}$ are approximate Nash equilibria for $G$, then it must be the case that for each $t$, one of the players has a profitable deviation in $G$ with respect to the product strategy $\sigma_h\^t(s)$, at least for a constant fraction of the tuples $(s,h)$. We will argue that if this were to be the case, it would imply that there exists a non-Markov deviation policy for at least one player $i$ in Definition \ref{def:cce}, meaning that $\ol \sigma$ is not in fact an $\ep$-CCE.

To sketch the idea, recall that to draw a trajectory from $\ol \sigma$, we first draw a random index $t^\st \sim [T]$ uniformly at random, and then execute $\sigma\^{t^\st}$ for an episode. We will show (roughly) that for each player $i$, it is possible to compute a non-Markov deviation policy $\pi_i^\dagger$ which, under the draw of a trajectory from $\ol \sigma$, can ``infer'' the value of the index $t^\st$ within the first few steps of the episode. Then policy $\pi_i^\dagger$ then, at each state $s$ and step $h$ after the first few steps, play a best response to their opponent's portion of the strategy $\sigma_h\^{t^\st}(s)$. If, for each possible value of $t^\st$, none of the distributions $\sigma_h\^{t^\st}(s)$ are approximate Nash equilibria of $G$, this means that at least one of the players $i$ can significantly increase their value in $\MG$ over that of $\ol \sigma$ by playing $\pi_i^\dagger$, which contradicts the assumption that $\ol \sigma$ is an $\ep$-CCE.

It remains to explain how we can construct a non-Markov policy $\pi_i^\dagger$ which ``infers'' the value of $t^\st$. Unfortunately, exactly inferring the value of $t^\st$ in the fashion described above is impossible: for instance, if there are $t_1 \neq t_2$ so that $\sigma\^{t_1} = \sigma\^{t_2}$, then clearly it is impossible to distinguish between the cases $t^\st= t_1$ and $t^\st = t_2$. Nevertheless, by using the fact that each player observes the full joint action profile played at each step $h$, we can construct a non-Markov policy which employs \emph{Vovk's aggregating algorithm} for online density estimation \cite{vovk1990aggregating,cesa2006prediction} in order to compute a distribution which is \emph{close} to $\sigma_h\^{t^\st}(s)$ for most $h \in [H]$.\icmlfncut{Vovk's aggregating algorithm is essentially the exponential weights algorithm with the logarithmic loss. A detailed background for the algorithm is provided in Section \ref{sec:online-density}.} This guarantee is stated formally in an abstract setting in Proposition \ref{prop:online-tvd}, and is instantiated in the proof of Theorem \ref{thm:markov-formal} in (\cref{eq:sum-tv-bound}). As we show in Section \ref{sec:markov-proof}, approximating $\sigma_h\^{t^\st}(s)$ as we have described is sufficient to carry out the reasoning from the previous paragraph. %

\arxiv{\section{Lower bound for non-Markov algorithms}}
\icml{\section{Lower bound for non-Markov algorithms}}
\label{sec:nonmarkov}
In this section, we prove Theorem \ref{thm:nonmarkov-intro} (restated formally below as Theorem \ref{thm:nonmarkov-formal}), which strengthens Theorem \ref{thm:markov-formal} by allowing the sequence $\sigma\^1, \ldots, \sigma\^T$ of product policies to be non-Markovian. This additional strength comes at the cost of our lower bound only applying to \emph{3-player} Markov games (as opposed to Theorem \ref{thm:markov-formal}, which applied to 2-player games). 

\subsection{\GenCCE problem and computational model}

To formalize the computational model for the \GenCCE problem, we must first describe how the non-Markov product policies $\sigma\^t =
(\sigma\^t_1, \ldots, \sigma\^t_m)$ are represented.
Recall that a non-Markov policy $\sigma_i\^t \in \Pirndrnd_i$ is, by definition, a mapping from agent $i$'s history and current state to a distribution over their next action. Since there are exponentially many possible histories, it is information-theoretically impossible to express an arbitrary policy in $\Pirndrnd_i$ with polynomially many bits. As our focus is on computing a sequence of such policies $\sigma\^t$ in polynomial time, certainly a prerequisite is that $\sigma\^t$ can be expressed in polynomial space. %
Thus, we adopt the representational assumption, stated formally in Definition \ref{def:computable-policy}, that each of the policies $\sigma_i\^t \in \Pirndrnd_i$ is described by a bounded-size circuit that can compute the conditional distribution of each next action given the history. This assumption is satisfied by essentially
all empirical and theoretical work concerning non-Markov policies (e.g.,
\cite{leibo2021scalable,agapiou2022melting,jin2021vlearning,song2022when}). %
\noah{may want to state that each player's output policies from V-learning do satisfy it as well, though nontrivially since you have to do dynamic programming to compute the distributions -- in particular, you compute the probability of going from each $k_h$ to each $k_{h'}'$ for all $h < h'$, for increasing $h'-h$, given the sequence of states/actions of agent $i$;  I will leave this for a later version in the interest of time.} 

\begin{defn}[Computable policy]
  \label{def:computable-policy}
  Given a $m$-player Markov game $\MG$ and $N \in \BN$, we say that a policy $\sigma_i \in \Pirndrnd_i$ is \emph{$N$-computable} if for each $h \in [H]$, there is a circuit of size $N$ that,\footnote{For concreteness, we suppose that ``circuit'' means ``boolean circuit'' as in \cite[Definition 6.1]{arora2006computational}, where probabilities are represented in binary. The precise model of computation we use does not matter, though, and we could equally assume that the policies $\sigma_i$ may be computed by Turing machines that terminate after $N$ steps.} on input $(\tau_{i,h-1}, s) \in \CH_{i,h-1} \times \MS$, outputs the distribution $\sigma_i(\tau_{i,h-1}, s) \in \Delta(\MA_i)$. A policy $\sigma = (\sigma_1, \ldots, \sigma_m) \in \Pirndrnd$ is $N$-computable if each constituent policy $\sigma_i$ is. \loose
\end{defn}
{\noah{Perhaps discuss that it is more or less equivalent for it to be a randomized circuit that outputs a sample from this distribution, since given that circuit we can learn the distribution to inverse poly accuracy, which should(?) be enough for the aggregation algorithm? -- leaving this for later.}}
Our lower bound applies to algorithms that produce sequences $\sigma\^1, \ldots, \sigma\^T$ for which each $\sigma\^t$ is $N$-computable, where the value $N$ is taken to be polynomial in the description length of the game $\MG$. For example, Markov policies whose probabilities can be expressed with $\beta$ bits are $O(HSA_i \beta)$-computable for each player $i$, since one can simply store each of the probabilities $\sigma_{i,h}(s_h, a_{i,h})$\icmlcut{\xspace (for $h \in [H]$, $i \in [m]$, $a_{i,h} \in \MA_i$, $s_h \in \MS$)}, each of which takes $\beta$ bits to represent.

\paragraph{The \GenCCE problem.}
\GenCCE is the problem of computing a sequence of non-Markov product policies $\sigma\^1, \ldots, \sigma\^T$ such that the uniform mixture forms an $\eps$-approximate CCE. The problem generalizes \MarkCCE (Definition \ref{def:markcce}) by relaxing the condition that the policies $\sigma\^t$ be Markov.
   \begin{defn}[\GenCCE Problem]
     \label{def:nonmarkcce}
  For an $m$-player Markov game $\MG$ and parameters $T,N \in \BN$ and $\ep > 0$ (which may depend on the size of the game $\MG$), $(T, \ep,N)$-\GenCCE is the problem of finding a sequence $\sigma\^1, \ldots, \sigma\^T\in\Pirndrnd$, with each $\sigma\^t$ being $N$-computable, such that the distributional policy $\ol \sigma = \frac{1}{T} \sum_{t=1}^T \indic{\sigma\^t} \in \Delta(\Pirndrnd)$ is an $\ep$-CCE for $\MG$ (equivalently, such that for all $i \in [m]$, $\Reg_{i,T}(\sigma\^1, \ldots, \sigma\^T) \leq \ep \cdot T$).
   \end{defn}

\subsection{Main result}

Our main theorem for this section, Theorem \ref{thm:nonmarkov-formal}, shows that for appropriate values of $T$, $\ep$, and $N$, solving the $(T, \ep, N)$-\GenCCE problem is at least as hard as computing Nash equilibria in normal-form games. 

\begin{theorem}
  \label{thm:nonmarkov-formal}
  Fix $n \in \BN$, and let $\Tn, \Nn \in \BN$, and $\epn > 0$ satisfy $1 < \Tn < \exp \left( \frac{\epn^2 \cdot n}{16} \right)$. Suppose there exists an algorithm that, given the description of any $3$-player Markov game $\MG$ with $|\MG| \leq n$, solves the $(\Tn, \epn, \Nn)$-\GenCCE problem in time $U$, for some $U \in \BN$. Then, for any $\delta > 0$, the $2$-player $(\lfloor n/2 \rfloor, 50\epn)$-\Nash problem can be solved in randomized time $(n\Tn\Nn U \log(1/\delta)/\epn)^{C_0}$ with failure probability $\delta$, where $C_0>0$ is an absolute constant.
\end{theorem}
By analogy to Corollary \ref{cor:markov-formal-ppad}, we obtain the following immediate consequence.
\begin{corollary}[\GenCCE is hard under $\PPAD\nsubseteq\RP$]
  \label{cor:nonmarkov-formal-ppad}
For any \arxiv{constant} $C > 4$, if there is an algorithm which, given the description of a 3-player Markov game $\MG$, solves the $(|\MG|^C, |\MG|^{-\frac{1}{C}}, |\MG|^C)$-\GenCCE problem in time $\poly(|\MG|)$, then $\PPAD \subseteq \RP$.\loose
\end{corollary}

\paragraph{Proof overview for Theorem \ref{thm:nonmarkov-formal}.}
The proof of \cref{thm:nonmarkov-formal} has a similar high-level structure to that of \cref{thm:markov-formal}: given an $m$-player normal-form $G$, we define an $(m+1)$-player Markov game $\MG = \MG(G)$ which has $n_0 := \lfloor n/m \rfloor$ actions per player and horizon $H \approx n_0$. The key difference in the proof of \cref{thm:nonmarkov-formal} is the structure of the players' reward functions. To motivate this difference and the addition of an $(m+1)$-th player,
\icml{we explain why the proof of \cref{thm:markov-formal} fails to extend: a sequence $\sigma\^1, \ldots, \sigma\^T$ can hypothetically solve the \GenCCE problem by attempting to punish any one player's deviation policy, and thus avoid having to compute a Nash equilibrium of $G$. In particular, if player $i$ plays according to the policy $\pi_i^\dagger$ that we described in Section \ref{sec:markov-formal}, then other players $j \neq i$ can use the non-Markov property of $\sigma_j\^t$ to adjust their choice of actions in later rounds to decrease player $i$'s value.
  }
\icmlcut{let us consider what goes wrong in the proof of \cref{thm:markov-formal} when the policies $\sigma\^t$ are allowed to be non-Markov. We will explain how a sequence $\sigma\^1, \ldots, \sigma\^T$ can hypothetically solve the \GenCCE problem by attempting to punish any one player's deviation policy, and thus avoid having to compute a Nash equilibrium of $G$. In particular, for each player $j$, suppose $\sigma_j\^t$ tries to detect, based on the state transitions and player $j$'s rewards, whether every other player $i \neq j$ is playing according to $\sigma_i\^t$. If some player $i$ is not playing according to $\sigma_i\^t$ at some step $h$, then at steps $h' > h$, the policy $\sigma_j\^t$ can select actions that attempt to minimize player $i$'s rewards. In particular, if player $i$ plays according to the policy $\pi_i^\dagger$ that we described in Section \ref{sec:markov-formal}, then other players $j \neq i$ can adjust their choice of actions in later rounds to decrease player $i$'s value. \dfcomment{Are we supposed to be assuming that $\sigma_1\ind{t},\ldots,\sigma_T\ind{t}$ satisfy some property of interest above (eg, $\eps$-CCE)? Right now it is not immediately clear why it is useful to know that player $i$ is not playing according to $\sigma_i\ind{t}$. I guess the point we want to make clear is: we are punishing a hypothetical deviation policy in Def \ref{def:cce}. If we were not able to do this, it could mean that $\wb{\sigma}$ is not actually a CCE. But the fact that $\sigma$ are non-Markov makes, makes it easier for them to punish hypothetical deviations}\noah{added sentence accordingly}
}
  
\icml{
  \nocite{fudenberg1994folk}
  This behavior is reminiscent of ``tit-for-tat'' strategies which are used to establish the \emph{folk theorem} in the theory of repeated games \cite{maskin1986folk}. The folk theorem describes how Nash equilibria are more numerous in repeated games than in single-shot normal form games. As it turns out, the folk theorem does not yield to worst-case speedups in repeated games, when the number of players is at least 3. Indeed, \citet{borgs2008myth} gave an ``anti-folk theorem'', showing that computing Nash equilibria in $(m+1)$-player repeated games is \PPAD-hard for $m \geq 2$, via a reduction to $m$-player normal-form games. We adapt their reduction to our setting: roughly speaking, this approach adds an $(m+1)$-th player whose actions represent potential deviations for each of the $m$ players. The structure of the rewards ensures that if $\ol \sigma = \frac 1T \sum_{t=1}^T \indic{\sigma\^t}$ is an $\ep$-CCE, then for some policy $\pi_{m+1}^\dagger$ of the $(m+1)$-th player, the first $m$ players will play an approximate Nash of $G$ with constant probability, under a trajectory drawn from the joint policy $\ol \sigma_{-(m+1)} \times \pi_{m+1}^\dagger$. Thus, in order to efficiently find a Nash (see \cref{alg:3nash}), we need to simulate the policy $\ol \sigma_{-(m+1)} \times \pi_{m+1}^\dagger$, which involves running Vovk's algorithm. This approach is in contrast to the proof of \cref{thm:markov-formal}, which used Vovk's algorithm as an ingredient in the proof but not in the Nash computation algorithm. 
  }

\icmlcut{
This behavior is reminiscent of ``tit-for-tat'' strategies which are used to establish the \emph{folk theorem} in the theory of repeated games \cite{maskin1986folk,fudenberg1994folk}. The folk theorem describes how Nash equilibria are more numerous (and potentially easier to find) in repeated games than in single-shot normal form games. As it turns out, the folk theorem does not provably yield to worst-case computational speedups in repeated games, at least when the number of players is at least 3. Indeed, \citet{borgs2008myth} gave an ``anti-folk theorem'', showing that computing Nash equilibria in $(m+1)$-player repeated games is \PPAD-hard for $m \geq 2$, via a reduction to $m$-player normal-form games. We utilize their reduction, for which the key idea is as follows: given an $m$-player normal form game $G$, we construct an $(m+1)$-player Markov game $\MG(G)$ in which the $(m+1)$-th player acts as a \emph{kibitzer},\icmlfncut{Kibitzer is a Yiddish term for an observer who offers advice. \noah{need to cite dictionary?} \dfcomment{i think it's fine haha}} with actions indexed by tuples $(j, a_j)$, for $j \in [m]$ and $a_j \in \MA_j$. The kibitzer's action $(j,a_j)$ represents 1) a player $j$ to give advice to, and 2) their advice to the player, which is to take action $a_j$.
In particular, if the kibitzer plays $(j, a_j)$, it receives reward equal to the amount that player $j$ would obtain by deviating to $a_j$, and player $j$ receives the negation of the kibitzer's reward. Furthermore, all other players receive 0 reward.

\dfcomment{I think it would be good to remind below (or elsewhere) that 1) we draw $t^\st$ uniformly and want to learn it and 2) our goal is to show that the first $m$ players *do* often play nash equilibria.}\noah{did so below}

To see why the addition of the kibitzer is useful, suppose that $\sigma\^1, \ldots, \sigma\^T$ solves the \GenCCE problem, so that $\ol\sigma := \frac 1T \sum_{t=1}^T \indic{\sigma\^t}$ is an $\ep$-CCE. We will show that, with at least constant probability over a trajectory drawn from $\ol \sigma$ (which involves drawing $t^\st \sim [T]$ uniformly), the joint strategy profile played by the first $m$ players constitutes an approximate Nash equilibrium of $G$.
Suppose for the purpose of contradiction that this were not the case.
We show that there exists a non-Markov deviation policy $\pi_{m+1}^\dagger$ for the kibitzer which, similar to the proof of \cref{thm:markov-formal}, learns the value of $t^\st$ and plays a tuple $(j, a_j)$ such that action $a_j$ increases player $j$'s payoff in $G$, thereby increasing its own payoff. Even if the other players attempt to punish the kibitzer for this deviation, they will not be able to since, roughly speaking, the kibitzer game as constructed above has the property that for any strategy for the first $m$ players, the kibitzer can always achieve reward at least $0$.

The above argument shows that under the joint policy $\ol \sigma_{-(m+1)} \times \pi_{m+1}^\dagger$ (namely, the first $m$ players play according to $\ol \sigma$ and the kibitzer plays according to $\pi_{m+1}^\dagger$) then with constant probability over a trajectory  drawn from this policy, the distribution of the first $m$ players' actions is an approximate Nash equilibrium of $G$. Thus, in order to efficiently find such a Nash (see Algorithm \ref{alg:3nash}), we need to simulate the policy $\ol \sigma_{-(m+1)} \times \pi_{m+1}^\dagger$, which involves running Vovk's aggregating algorithm. This approach is in contrast to the proof of \cref{thm:markov-formal}, for which Vovk's aggregating algorithm was an ingredient in the proof but was not actually used in the Nash computation algorithm (Algorithm \ref{alg:2nash}). 
The details of the proof of correctness of \cref{alg:3nash} are somewhat delicate, and may be found in \cref{sec:nonmarkov-proof}.

\dfcomment{should we add one final sentence remarking that vovk is used algorithmically in the reduction, in this case?}\noah{did so}
}

\paragraph{Two-player games.} One intruiging question we leave open is whether the \GenCCE problem remains hard for two-player Markov games. Interestingly, as shown by \citet{littman2005polynomial}, there is a polynomial time algorithm to find an exact Nash equilibrium for the special case of repeated two-player normal-form games. Though their result only applies in the infinite-horizon setting, it is possible to extend their results to the finite-horizon setting, which 
rules out naive approaches to extend the proof of \cref{thm:nonmarkov-formal} and Corollary \ref{cor:nonmarkov-formal-ppad} to two players.

\arxiv{
\section{Multi-player games: Statistical lower bounds}
\label{sec:multiplayer}

In this section we present Theorem \ref{thm:query-intro} (restated formally below as \cref{thm:statistical-lb}), which gives a statistical lower bound for the \GenCCE problem. The lower bound applies to any algorithm, regardless of computational cost, that accesses the underlying Markov game through a \emph{generative model}.
\begin{defn}[Generative model]
  \label{def:generative-model}
  For an $m$-player Markov game $\MG = (\MS, H, (\MA_i)_{i \in [m]}, \BP, (R_i)_{i \in [m]}, \mu)$, a \emph{generative model oracle} is defined as follows: given a query described by a tuple $(h, s, \ba) \in [H] \times \MS \times \MA$, the oracle returns the distribution $\BP_h(\cdot | s, \ba) \in \Delta(\MS)$ and the tuple of rewards  $(R_{i,h}(s, \ba))_{i \in [m]}$. 
\end{defn}
From the perspective of lower bounds, the assumption that the algorithm has access to
a generative model is quite reasonable, as it encompasses most standard access models in RL, including the online access model, in which the algorithm repeatedly queries a policy and observes a
trajectory drawn from it, as well as the \emph{local access generative model}
used in from \cite{yin2022efficient,weisz2021query}. We remark that it is slightly more standard to assume that queries to the generative model only return a \emph{sample} from the distribution $\BP_h(\cdot
| s, \ba)$ as opposed to the distribution itself \cite{kakade2003sample,kearns1999sparse}, but since our goal is to prove lower bounds, the notion in Definition \ref{def:generative-model} only makes our results stronger.

To state our main result, we recall the definition $|\MG| = \max\{ S, \max_{i \in [m]} A_i, H, \beta(\MG) \}$. In the present section, we consider the setting where the number of players $m$ is large. Here, $|\MG|$ does not necessarily correspond to the description length for $\cG$, and should be interpreted, roughly speaking, as a measure of the description complexity of $\cG$ $|\MG|$ with respect to \emph{decentralized} learning algorithms. In particular, from the perspective of an individual agent implementing a decentralized learning algorithm, their sample complexity should depend only on the size of their \emph{individual action set} (as well as the global parameters $S, H, \beta(\MG)$), as opposed to the size of the \emph{joint action set}, which grows exponentially in $m$; the former is captured by $\abs{\cG}$, while the latter is not. Indeed, a key advantage shared by much prior work on decentralized RL \cite{jin2021vlearning,song2022when,mao2021provably,daskalakis2022complexity} is their avoidance of the \emph{curse of multi-agents}, which describes the situation where an algorithm has sample and computational costs that scale exponentially in $m$.

Our main result for this section, Theorem \ref{thm:statistical-lb}, states that for $m$-player Markov games, exponentially many generative model queries (in $m$) are necessary to produce a solution to the $(T,\ep, N)$-\GenCCE problem, unless $T$ itself is exponential in $m$.
\begin{theorem}
  \label{thm:statistical-lb}
  Let $m\geq{}2$ be given.  There are constants $c, \ep > 0$ so that the following holds. 
 Suppose there is an algorithm $\CB$ which, given access to a generative model for a $(m+1)$-player Markov game $\MG$ with $|\MG| \leq 2m^6$, solves the $(T, \ep/(10m), N)$-\GenCCE problem for $\MG$ for some $T$ satisfying $1 < T < \exp(cm)$, and \emph{any} $N \in \BN$. Then $\CB$ must make at least $2^{\Omega(m)}$ queries to the generative model. 
\end{theorem}
Theorem \ref{thm:statistical-lb} establishes that there are $m$-player Markov games, where the number of states, actions per player, and horizon are bounded by $\poly(m)$, but any algorithm with regret $o(T/m)$ must make $2^{\Omega(m)}$ queries (via Fact \ref{fac:no-regret-cce}). In particular, if there are $\poly(m)$ queries per episode, as is standard in the online simulator model where a trajectory is drawn from the policy $\sigma\^t$ at each episode $t \in [T]$, then $T > 2^{\Omega(m)}$ episodes are required to have regret $o(T/m)$.  \dfcomment{this is assuming $O(1)$ queries/episode? maybe we should be a little more clear on the episodic model.}\noah{added a sentence} This is in stark contrast to the setting of normal-form games, where even for the case of bandit feedback (which is a special case of the generative model setting), standard no-regret algorithms have the property that each player's regret scales as $\til O(\sqrt{Tn})$ (i.e., independently of $m$), where $n$ denotes the number of actions per player \cite{lattimore2020bandit}. As with our computational lower bounds, Theorem \ref{thm:statistical-lb} is not limited to decentralized algorithms, and also rules out \emph{centralized} algorithms which, with access to a generative model, compute a sequence of policies which constitutes a solution to the \GenCCE problem. Furthermore, it holds for arbitrary values of $N$, thus allowing the policies $\sigma\^1, \ldots, \sigma\^T \in \Pirndrnd$ solving the \GenCCE problem to be arbitrary general policies. 

\dfcomment{should we remark on the fact that $N$ can be arbitrary?}\noah{did so}

}

\arxiv{
\section{Discussion and interpretation}
\label{sec:discussion}

Theorems \ref{thm:markov-formal}, \ref{thm:nonmarkov-formal}, and \ref{thm:statistical-lb} present barriers---both computational and statistical---toward developing efficient decentralized no-regret guarantees for multi-agent reinforcement learning. We emphasize that no-regret algorithms are the only known approach for obtaining fully decentralized learning algorithms (i.e., those which do not rely even on shared randomness) in normal-form games, and it seems unlikely that a substantially different approach would work in Markov games. Thus, these lower bounds for finding subexponential-length sequences of policies with the no-regret property represent a significant obstacle for fully decentralized multi-agent reinforcement learning. 
Moreover, these results rule out even the prospect of developing efficient \emph{centralized} algorithms that produce no-regret sequences of policies, i.e., those which ``resemble'' independent learning.
In this section, we compare our lower bounds with recent upper bounds for decentralized learning in Markov games, and explain how to reconcile these results.

\dfcomment{Based on icml reviewer discussion: recap significance of no-regret in the context of decentralized algorithms.}\noah{added it above}

\dfcomment{Based on icml reviewer discussion: We may want to add a discussion around why our lower bounds are different in nature from \citet{daskalakis2022complexity,jin2022complexity}, even though it is obvious if you're familiar.}\noah{added it at end of section}

\subsection{Comparison to \vlearning}

The \texttt{V-learning} algorithm \cite{jin2021vlearning,song2022when,mao2021provably} is a polynomial-time decentralized learning algorithm that proceeds in two phases. In the first phase, the $m$ agents interact over the course of $K$ episodes in a decentralized fashion, playing product Markov policies $\sigma\^1, \ldots, \sigma\^K \in \PiMarkov$. In the second phase, the agents use data gathered during the first phase to produce a distributional policy $\wh \sigma \in \Delta(\Pirndrnd)$, which we refer to as the \emph{output policy} of \texttt{V-learning}. As discussed in Section \ref{sec:intro}, one implication of Theorem \ref{thm:markov-formal} is that the first phase of \texttt{V-learning} cannot guarantee each agent sublinear regret. Indeed if $K$ is of polynomial size (and $\PPAD \neq \PP$), this follows because a bound of the form $\Reg_{i,K}(\sigma\^1, \ldots, \sigma\^K) \leq \ep K$ for all $i$ implies that $(\sigma\^1, \ldots, \sigma\^K)$ solves the $(K, \ep)$-\MarkCCE problem.

The output policy $\wh{\sigma}\in\Delta(\Pirndrnd)$ produced by \vlearning is an approximate CCE (per Definition \ref{def:cce}), and it is natural to ask how many product policies it takes to represent $\sigmahat$ as a uniform mixture (that is, whether $\sigmahat$ solves the $(T,\ep)$-\MarkCCE problem for a reasonable value of $T$).
First, recall that \texttt{V-learning} requires $K = \poly(H, S, \max_i A_i) / \ep^2$ episodes to ensure that $\wh \sigma$ is an $\ep$-CCE. It is straightforward to show that $\wh \sigma$ can be expressed as a \emph{non-uniform} mixture of at most $K^{KHS+1}$ policies in $\Pirndrnd$ (we prove this fact in detail below). By discretizing the non-uniform mixture, one can equivalently represent it as \emph{uniform} mixture of $O(1/\ep) \cdot K^{KHS+1}$ product policies, up to $\eps$ error. Recalling the value of $K$, we conclude that we can express $\sigmahat$ as a uniform mixture of $T = \exp( \til O(1/\ep^2)\cdot \poly(H, S, \max_i A_i))$ product policies in $\Pirndrnd$. %
Note that the lower bound of Theorem \ref{thm:nonmarkov-formal} rules out the efficient computation of an $\eps$-CCE represented as a uniform mixture of $T \ll \exp(\ep^2 \cdot \max\{H, S, \max_i A_i \})$ efficiently computable policies in $\Pirndrnd$. Thus, in the regime where $1/\ep$ is polynomial in $H, S, \max_i A_i$, %
this upper bound on the sparsity of the policy $\wh \sigma$ produced \texttt{V-learning} matches that from Theorem \ref{thm:nonmarkov-formal}, up to a polynomial in the exponent. 

\paragraph{The sparsity of the output policy from \texttt{V-learning}.}
\label{sec:vlearning-output}
We now sketch a proof of the fact that the output policy $\wh \sigma$ produced by \texttt{V-learning} can be expressed as a (non-uniform) average of $K^{KHS+1}$ policies in $\Pirndrnd$, where $K$ is the number of episodes in the algorithm's initial phase. We adopt the notation and terminology from \citet{jin2021vlearning}.

Consider Algorithm 3 of \citet{jin2021vlearning}, which describes the second phase of \texttt{V-learning}, which produces the output policy $\wh \sigma$. We describe how to write $\wh \sigma$ as a weighted average of a collection of product policies, each of which is indexed by a function $\phi : [H] \times \MS \times [K] \ra [K]$ and a parameter $k_0 \in [K]$: in particular, we will write $\wh \sigma = \sum_{k_0, \phi} w_{k_0, \phi} \cdot \sigma_{k_0, \phi} \in \Delta(\Pirndrnd)$, where $w_{k_0, \phi} \in [0,1]$ are mixing weights summing to 1 and $\sigma_{k_0, \phi} \in \Pirndrnd$. The number of tuples $(k_0, \phi)$ is $K^{1+KHS}$.

We define the mixing weight allocated $w_{k_0, \phi}$ to any tuple $(k_0,\phi)$ to be:
\begin{align}
\arxiv{w_{k_0, \phi} :=} \frac{1}{K} \cdot \prod_{(h,s,k) \in [H] \times \MS \times [K]} \One{\phi(h,s,k) \in [N_h^k(s)]} \cdot \alpha_{N_h^k(s)}^{\phi(h,s,k)}\nonumber,
\end{align}
where $N_h^k(s) \in [K]$ and $\alpha_{N_h^k(s)}^i \in [0,1]$ (for $i \in [N_h^k(s)]$) are defined as in \cite{jin2021vlearning}.

Next, for each $k_0, \phi$, we define $\sigma_{k_0, \phi} \in \Pirndrnd$ to be the following policy: it maintains a parameter $k \in [K]$ over the first $h\leq{}H$ steps of the episode (as in Algorithm 3 of \cite{jin2021vlearning}), but upon reaching state $s$ at step $h$, given the present value of $k \in [K]$, sets $i := \phi(h, s, k)$, and updates $k \gets k_h^i(s)$, and then samples an action $\ba \sim \pi_h^k(\cdot | s)$ (where $k_h^i(s), \pi_h^k(\cdot | s)$ are defined in \cite{jin2021vlearning}). Since the mixing weights $w_{k_0, \phi}$ defined above exactly simulate the random draws of the parameter $k$ in Line 1 and the parameters $i$ in Line 4 of \cite[Algorithm 3]{jin2021vlearning}, it follows that the distributional policy $\wh \sigma$ defined by \cite[Algorithm 3]{jin2021vlearning} is equal to $\sum_{k_0, \phi} w_{k_0, \phi} \cdot \sigma_{k_0, \phi} \in \Delta(\Pirndrnd)$. \dfcomment{not clear where $\pi_h^k(\cdot\mid{}s)$ is defined above}\noah{added rmk saying it's defined in v-learning paper}

\subsection{No-regret learning against Markov deviations}

As discussed in Section \ref{sec:intro}, \citet{erez2022regret} showed the existence of a learning algorithm with the property that if each agent plays it independently for $T$ episodes, then no player can achieve regret more than $O(\poly(m,H,S, \max_i A_i) \cdot T^{3/4})$ by deviating to any fixed \emph{Markov policy}. This notion of regret corresponds to, in the context of Definition \ref{def:regret}, replacing $\max_{\sigma_i \in \Pirndrnd_i}$ with the smaller quantity $\max_{\sigma_i \in \PiMarkov_i}$. Thus, the result of \citet{erez2022regret} applies to a weaker notion of regret than that of the \GenCCE problem, and so does not contradict any of our lower bounds. One may wonder which of these two notions of regret (namely, best possible gain via deviation to a Markov versus non-Markov policy) is the ``right'' one. We do not believe that there is a definitive answer to this question, but we remark that in many empirical applications of multi-agent reinforcement learning it is standard to consider non-Markov policies \citep{leibo2021scalable,agapiou2022melting}. Furthermore, as shown in the proposition below, there are extremely simple games, e.g., of constant size, in which Markov deviations lead to ``vacuous'' behavior: in particular, all Markov policies have the same (suboptimal) value but the best non-Markov policy has much greater value:
\begin{proposition}
  \label{prop:nonmarkov-deviation}
There is a 2-player, 2-action, 1-state Markov game with horizon $2$ and a non-Markov policy $\sigma_2 \in \Pirndrnd_2$ for player 2 so that for all $\sigma_1 \in \PiMarkov_1$,  $ V_1^{\sigma_1 \times \sigma_2} = 1/2$ yet $\max_{\sigma_1 \in \Pirndrnd_1} \left\{ V_1^{\sigma_1 \times \sigma_2} \right\} = 3/4$.%
\end{proposition}
The proof of Proposition \ref{prop:nonmarkov-deviation} is provided in Section \ref{sec:discussion-proofs} below. 

Other recent work has also proved no-regret guarantees with respect to deviations to restricted policy classes. In particular,  \citet{zhan2022decentralized} studies a setting in which each agent $i$ is allowed to play policies in an arbitrary restricted policy class $\Pi_i' \subseteq \Pirndrnd_i$ in each episode, and regret is measured with respect to deviations to any policy in $\Pi_i'$. \citet{zhan2022decentralized} introduces an algorithm, \texttt{DORIS}, with the property that when all agents play it independently, each agent $i$ experiences regret $O\left(\poly(m, A, S, H) \cdot \sqrt{T \sum_{i=1}^m \log |\Pi_i'|}\right)$ to their respective class $\Pi'_i$.\footnote{Note that in the tabular setting, the sample complexity of \texttt{DORIS} \dfedit{(Corollary 1)} scales with the size $A$ of the \emph{joint} action set, since each player's value function class consists of the class of all functions $f : \MS \times \MA \ra [0,1]$, which has Eluder dimension scaling with $S \cdot A$, i.e., exponential in $m$. \noah{that paper has a remark on the top of p.16 saying they avoid curse of multi-agents, but I think that's wrong for the reason I wrote here...} \dfcomment{agreed--i think the covering number in their bound should scale with $A$ as well}}

\texttt{DORIS} is not computationally efficient, since it involves performing exponential weights over the class $\Pi_i'$, which requires space complexity $\abs{\Pi'_I}$. Nonetheless, one can compare the statistical guarantees the algorithm provides to our own results. Let $\PiMarkovdet_i \subset \PiMarkov_i$ denote the set of deterministic Markov policies of agent $i$, namely sequences $\pi_i = (\pi_{i,1}, \ldots, \pi_{i,H})$ so that $\pi_{i,h} : \MS \ra \MA_i$. 
In the case that $\Pi_i' = \PiMarkovdet_i$, $\Pi_i'$, we have $\log |\Pi_i'| = O(SH \log A_i)$, which means that \texttt{DORIS} obtains no-regret against Markov deviations when $m$ is constant, comparable to \citet{erez2022regret}.\footnote{\citet{erez2022regret} has the added bonus of computational efficiency, even for polynomially large $m$, though has the significant drawback of assuming that the Markov game is known.} However, we are interested in the setting in which each player's regret is measured with respect to all deviations in $\Pirndrnd_i$ (equivalently, $\Pidet_i$). Accordingly, if we take $\Pi_i' = \Pidet_i \subset \Pirndrnd_i$,\footnote{\texttt{DORIS} plays distributions over policies in $\Pi_i' = \Pidet_i$ at each episode, whereas in our lower bounds we consider the setting where a policy in $\Pirndrnd_i$ is played each episode; Facts \ref{fac:gen-decompose} and \ref{fac:emb-inverse} shows that these two settings are essentially equivalent, in that any policy in $\Pirndrnd_1 \times \cdots \times \Pirndrnd_m$ can be simulated by one in $\Delta(\Pidet_1) \times \cdots \times \Delta(\Pidet_m)$, and vise versa.} then $\log |\Pi_i'| > (SA_i)^{H-1}$, meaning that \texttt{DORIS} does not imply any sort of sample-efficient guarantee, even for $m=2$.

Finally, we remark that the algorithm \texttt{DORIS} \citep{zhan2022decentralized}, as well as the similar algorithm \texttt{OPMD} from earlier work of \citet{liu2022learning}, obtains the same regret bound stated above even when the opponents are controlled by (possibly adaptive) adversaries. However, this guarantee crucially relies on the fact that any agent implementing \texttt{DORIS} must observe the policies played by opponents following each episode; this feature is the reason that the regret bound of \texttt{DORIS} does not contradict the exponential lower bound of \citet{liu2022learning} for no-regret learning against an adversarial opponent. As a result of being restricted to this ``revealed-policy'' setting, \texttt{DORIS} is not a fully decentralized algorithm in the sense we consider in this paper.

\subsection{On the role of shared randomness}
A key assumption in our lower bounds for no-regret learning is that each of the joint policies $\sigma\^1, \ldots, \sigma\^T$ produced by the algorithm is a \emph{product policy}; such an assumption is natural, since it subsumes independent learning protocols in which each agent $i$ selects $\sigma_i\ind{t}$ without knowledge of $\sigma_{-i}\ind{t}$. Compared to general (stochastic) joint policies, product policies have the desirable property that, to sample a trajectory from $\sigma\^t = (\sigma_1\^t, \ldots, \sigma_m\^t) \in \Pirndrnd_1 \times \cdots \times \Pirndrnd_m = \Pirndrnd$, the agents do no require access to shared randomness. In particular, each agent $i$ can independently sample its action from $\sigma_i\^t$ at each of the $h$ steps of the episode. It is natural to ask how the situation changes if we allow the agents to use shared random bits when sampling from their policies, which corresponds to allowing $\sigma\ind{1},\ldots,\sigma\^T$ to be non-product policies. In this case, \texttt{V-learning} yields a positive result via a standard ``batch-to-online'' conversion: by applying the first phase of \texttt{V-learning} during the first $T^{2/3}$ episodes and playing trajectories sampled i.i.d.~from the output policy produced by \texttt{V-learning} during the remaining $T-T^{2/3}$ episodes (which requires shared randomness), it is straightforward to see that a regret bound of order $\poly(H, S, \max_i A_i) \cdot T^{2/3}$ can be obtained. Similar remarks apply to \texttt{SPoCMAR} \citep{daskalakis2022complexity}, which can obtain a slightly worse regret bound of order $\poly(H, S, \max_i A_i) \cdot T^{3/4}$ in the same fashion. In fact, the batch-to-online conversion approach gives a generic solution for the setting in which shared randomness is available. That is, \emph{the assumption of shared randomness eliminates any distinction between no-regret algorithms and (non-sparse) equilibrium computation algorithms}, modulo slight loss in rates. For this reason, the shared randomness assumption is too strong to develop any sort of distinct theory of no-regret learning.

\dfcomment{We may want to add a remark in this subsection that shared randomness is essentially equivalent to removing the sparsity assumption, or allowing for $(\infty,\eps)$-\GenCCE. I think we had a remark about this earlier in the paper, but we may have commented it out. we can move it here}\noah{in the context of decentralized learning, it's not clear to me this is the case -- given a decentralized no-regret procedure without shared randomness, it's not clear how to convert it to a decentralized no-regret procedure with much smaller $T$ and shared randomness.}

\subsection{Comparison to lower bounds for finding stationary CCE}
A separate line of work \citet{daskalakis2022complexity,jin2022complexity} has recently shown \PPAD-hardness for the problem of finding stationary Markov CCE in infinite-horizon discounted stochastic games. These results are incomparable with our own: stationary Markov CCE are not sparse (in the sense of Definition \ref{def:markcce}), whereas we do not require stationarity of policies (as is standard in the finite-horizon setting).

\subsection{Proof of Proposition \ref{prop:nonmarkov-deviation}}
\label{sec:discussion-proofs}
Below we prove Proposition \ref{prop:nonmarkov-deviation}.
\begin{proof}[Proof of Proposition \ref{prop:nonmarkov-deviation}]
  We construct the claimed Markov game $\MG$ as follows. The single state is denoted by $\mf s$; as there is only a single state, the transitions are trivial. We denote each player's action space as $\MA_1 = \MA_2 = \{1,2\}$. The rewards to player 1 are given as follows: for all $(a_1, a_2) \in \MA$, 
  \begin{align}
R_{1,1}(\mf s, (a_1, a_2)) = \frac 12 \cdot \indic{a_2 = 1}, \qquad R_{1,2}(\mf s, (a_1, a_2)) = \frac 12 \cdot \indic{a_1 = a_2} \nonumber.
  \end{align}
  We allow the rewards of player 2 to be arbitrary; they do not affect the proof in any way.

  We let $\sigma_2  = (\sigma_{2,1}, \sigma_{2,2})\in \Pirndrnd_2$ be the policy which plays a uniformly random action at step 1 and then plays the same action at step 2: formally, $\sigma_{2,1}(s_1) = \Unif(\MA_2)$, and $\sigma_{2,2}((s_1, a_{2,1}, r_{2,1}), s_2) = \indic{a_{2,1}}$. Then for any Markov policy $\sigma_1 \in \PiMarkov_1$ of player 1, we must have $\BP_{\sigma_1 \times \sigma_2}(a_{1,2} = a_{2,2}) = 1/2$, which means that $V_1^{\sigma_1 \times \sigma_2} = \frac 12 \cdot \E_{\sigma_1 \times \sigma_2}[\indic{a_{2,1} = 1} + \indic{a_{1,2} = a_{2,2}}] = 1/2 \cdot (1/2 + 1/2) = 1/2$.

  On the other hand, any general (non-Markov) policy $\sigma_1 \in \Pirndrnd_1$ which satisfies
  \begin{align}
    \sigma_{1,2}((s_1, a_{1,1}, r_{1,1}), s_2) = \begin{cases}
    \indic{1}: & r_{1,1} = 1/2  \\
    \indic{2}: & r_{1,1} = 0
  \end{cases}\nonumber
  \end{align}
  has $V_1^{\sigma_1 \times \sigma_2} = 1/2 \cdot (1/2 + 1) = 3/4$. 
\end{proof}

}

\section*{Acknowledgements}
This work was performed in part while NG was an intern at Microsoft Research. NG is supported at MIT by a Fannie \& John Hertz Foundation Fellowship and an NSF Graduate Fellowship. 
This work has been made possible in part by a gift from the Chan Zuckerberg Initiative Foundation to establish the Kempner Institute for the Study of Natural and Artificial Intelligence. SK acknowledges funding from the Office of Naval Research under award N00014-22-1-2377 and the National Science Foundation Grant under award \#CCF-2212841.

\newpage

 \arxiv{
  \bibliographystyle{alpha}
  \bibliography{refs.bib}
   }

\icml{
\bibliography{refs.bib}

\newcommand{\etalchar}[1]{$^{#1}$}
\begin{thebibliography}{AVDG{\etalchar{+}}22}

\bibitem[AB06]{arora2006computational}
S.~Arora and B.~Barak.
\newblock {\em Computational Complexity: A Modern Approach}.
\newblock Cambridge University Press, 2006.

\bibitem[AFK{\etalchar{+}}22]{anagnostides2022uncoupled}
Ioannis Anagnostides, Gabriele Farina, Christian Kroer, Chung-Wei Lee, Haipeng
  Luo, and Tuomas Sandholm.
\newblock Uncoupled learning dynamics with \$o({\textbackslash}log t)\$ swap
  regret in multiplayer games.
\newblock In Alice~H. Oh, Alekh Agarwal, Danielle Belgrave, and Kyunghyun Cho,
  editors, {\em Advances in Neural Information Processing Systems}, 2022.

\bibitem[AVDG{\etalchar{+}}22]{agapiou2022melting}
John~P. Agapiou, Alexander~Sasha Vezhnevets, Edgar~A. Duéñez-Guzmán, Jayd
  Matyas, Yiran Mao, Peter Sunehag, Raphael Köster, Udari Madhushani, Kavya
  Kopparapu, Ramona Comanescu, DJ~Strouse, Michael~B. Johanson, Sukhdeep Singh,
  Julia Haas, Igor Mordatch, Dean Mobbs, and Joel~Z. Leibo.
\newblock Melting pot 2.0, 2022.

\bibitem[AYBK{\etalchar{+}}13]{abbasi2013online}
Yasin Abbasi~Yadkori, Peter~L Bartlett, Varun Kanade, Yevgeny Seldin, and Csaba
  Szepesvari.
\newblock Online learning in markov decision processes with adversarially
  chosen transition probability distributions.
\newblock In C.J. Burges, L.~Bottou, M.~Welling, Z.~Ghahramani, and K.Q.
  Weinberger, editors, {\em Advances in Neural Information Processing Systems},
  volume~26. Curran Associates, Inc., 2013.

\bibitem[Bab16]{babichenko2016query}
Yakov Babichenko.
\newblock Query complexity of approximate nash equilibria.
\newblock {\em J. ACM}, 63(4), oct 2016.

\bibitem[BBD{\etalchar{+}}22]{bakhtin2022human}
Anton Bakhtin, Noam Brown, Emily Dinan, Gabriele Farina, Colin Flaherty, Daniel
  Fried, Andrew Goff, Jonathan Gray, Hengyuan Hu, Athul~Paul Jacob, Mojtaba
  Komeili, Karthik Konath, Minae Kwon, Adam Lerer, Mike Lewis, Alexander~H.
  Miller, Sasha Mitts, Adithya Renduchintala, Stephen Roller, Dirk Rowe, Weiyan
  Shi, Joe Spisak, Alexander Wei, David Wu, Hugh Zhang, and Markus Zijlstra.
\newblock Human-level play in the game of <i>diplomacy</i> by combining
  language models with strategic reasoning.
\newblock {\em Science}, 378(6624):1067--1074, 2022.

\bibitem[BCI{\etalchar{+}}08]{borgs2008myth}
Christian Borgs, Jennifer Chayes, Nicole Immorlica, Adam~Tauman Kalai, Vahab
  Mirrokni, and Christos Papadimitriou.
\newblock The myth of the folk theorem.
\newblock In {\em Proceedings of the fortieth annual ACM symposium on Theory of
  computing}, pages 365--372, 2008.

\bibitem[BJY20]{bai2020nearoptimal}
Yu~Bai, Chi Jin, and Tiancheng Yu.
\newblock Near-optimal reinforcement learning with self-play.
\newblock In {\em Proceedings of the 34th International Conference on Neural
  Information Processing Systems}, NIPS'20, Red Hook, NY, USA, 2020. Curran
  Associates Inc.

\bibitem[Bla56]{blackwell1956analog}
David Blackwell.
\newblock An analog of the minimax theorem for vector payoffs.
\newblock {\em Pacific Journal of Mathematics}, 6:1--8, 1956.

\bibitem[BR17]{babichenko2017communication}
Yakov Babichenko and Aviad Rubinstein.
\newblock Communication complexity of approximate nash equilibria.
\newblock In {\em Proceedings of the 49th Annual ACM SIGACT Symposium on Theory
  of Computing}, STOC 2017, page 878–889, New York, NY, USA, 2017.
  Association for Computing Machinery.

\bibitem[Bro49]{brown1949some}
George~Williams Brown.
\newblock Some notes on computation of games solutions.
\newblock 1949.

\bibitem[BS18]{brown2018superhuman}
Noam Brown and Tuomas Sandholm.
\newblock Superhuman ai for heads-up no-limit poker: Libratus beats top
  professionals.
\newblock {\em Science}, 359(6374):418--424, 2018.

\bibitem[CBFH{\etalchar{+}}97]{cesa1997how}
Nicol{\`o} Cesa-Bianchi, Yoav Freund, David Haussler, David~P Helmbold,
  Robert~E Schapire, and Manfred~K. Warmuth.
\newblock How to use expert advice.
\newblock {\em Journal of the ACM}, 44(3):427--485, 1997.

\bibitem[CBL06]{cesa2006prediction}
Nicol{\`o} Cesa-Bianchi and G{\'a}bor Lugosi.
\newblock {\em Prediction, Learning, and Games}.
\newblock Cambridge University Press, New York, NY, USA, 2006.

\bibitem[CCT17]{chen2017wellsupported}
Xi~Chen, Yu~Cheng, and Bo~Tang.
\newblock {Well-Supported vs. Approximate Nash Equilibria: Query Complexity of
  Large Games}.
\newblock In Christos~H. Papadimitriou, editor, {\em 8th Innovations in
  Theoretical Computer Science Conference (ITCS 2017)}, volume~67 of {\em
  Leibniz International Proceedings in Informatics (LIPIcs)}, pages 57:1--57:9,
  Dagstuhl, Germany, 2017. Schloss Dagstuhl--Leibniz-Zentrum fuer Informatik.

\bibitem[CDT06]{chen2006computing}
Xi~Chen, Xiaotie Deng, and Shang-hua Teng.
\newblock Computing nash equilibria: Approximation and smoothed complexity.
\newblock In {\em 2006 47th Annual IEEE Symposium on Foundations of Computer
  Science (FOCS'06)}, pages 603--612, 2006.

\bibitem[CP20]{chen2020hedging}
Xi~Chen and Binghui Peng.
\newblock Hedging in games: Faster convergence of external and swap regrets.
\newblock In H.~Larochelle, M.~Ranzato, R.~Hadsell, M.F. Balcan, and H.~Lin,
  editors, {\em Advances in Neural Information Processing Systems}, volume~33,
  pages 18990--18999. Curran Associates, Inc., 2020.

\bibitem[DFG21]{daskalakis2021nearoptimal}
Constantinos~Costis Daskalakis, Maxwell Fishelson, and Noah Golowich.
\newblock Near-optimal no-regret learning in general games.
\newblock In A.~Beygelzimer, Y.~Dauphin, P.~Liang, and J.~Wortman Vaughan,
  editors, {\em Advances in Neural Information Processing Systems}, 2021.

\bibitem[DGP09]{daskalakis2009complexity}
Constantinos Daskalakis, Paul~W Goldberg, and Christos~H Papadimitriou.
\newblock The complexity of computing a nash equilibrium.
\newblock {\em SIAM Journal on Computing}, 39(1):195--259, 2009.

\bibitem[DGZ22]{daskalakis2022complexity}
Constantinos Daskalakis, Noah Golowich, and Kaiqing Zhang.
\newblock The complexity of markov equilibrium in stochastic games, 2022.

\bibitem[ELS{\etalchar{+}}22]{erez2022regret}
Liad Erez, Tal Lancewicki, Uri Sherman, Tomer Koren, and Yishay Mansour.
\newblock Regret minimization and convergence to equilibria in general-sum
  markov games, 2022.

\bibitem[FGGS13]{fearnley2013learning}
John Fearnley, Martin Gairing, Paul Goldberg, and Rahul Savani.
\newblock Learning equilibria of games via payoff queries.
\newblock In {\em Proceedings of the Fourteenth ACM Conference on Electronic
  Commerce}, EC '13, page 397–414, New York, NY, USA, 2013. Association for
  Computing Machinery.

\bibitem[FLM94]{fudenberg1994folk}
Drew Fudenberg, David Levine, and Eric Maskin.
\newblock The folk theorem with imperfect public information.
\newblock {\em Econometrica}, 62(5):997--1039, 1994.

\bibitem[FRSS22]{foster2022complexity}
Dylan~J Foster, Alexander Rakhlin, Ayush Sekhari, and Karthik Sridharan.
\newblock On the complexity of adversarial decision making.
\newblock {\em arXiv preprint arXiv:2206.13063}, 2022.

\bibitem[Han57]{hannan1957approximation}
J.~Hannan.
\newblock Approximation to {B}ayes risk in repeated play.
\newblock {\em Contributions to the Theory of Games}, 3:97--139, 1957.

\bibitem[HMC00]{hart2000simple}
Sergiu Hart and Andreu Mas-Colell.
\newblock A simple adaptive procedure leading to correlated equilibrium.
\newblock {\em Econometrica}, 68(5):1127--1150, 2000.

\bibitem[JKSY20]{jin2020reward}
Chi Jin, Akshay Krishnamurthy, Max Simchowitz, and Tiancheng Yu.
\newblock Reward-free exploration for reinforcement learning.
\newblock In {\em International Conference on Machine Learning (ICML)}, pages
  4870--4879. PMLR, 2020.

\bibitem[JLWY21]{jin2021vlearning}
Chi Jin, Qinghua Liu, Yuanhao Wang, and Tiancheng Yu.
\newblock V-learning--a simple, efficient, decentralized algorithm for
  multiagent {RL}.
\newblock {\em arXiv preprint arXiv:2110.14555}, 2021.

\bibitem[JMS22]{jin2022complexity}
Yujia Jin, Vidya Muthukumar, and Aaron Sidford.
\newblock The complexity of infinite-horizon general-sum stochastic games,
  2022.

\bibitem[Kak03]{kakade2003sample}
Sham~M Kakade.
\newblock On the sample complexity of reinforcement learning, 2003.

\bibitem[KECM21]{kwon2021rl}
Jeongyeol Kwon, Yonathan Efroni, Constantine Caramanis, and Shie Mannor.
\newblock {RL} for latent {MDPs}: Regret guarantees and a lower bound.
\newblock {\em Advances in Neural Information Processing Systems}, 34, 2021.

\bibitem[KEG{\etalchar{+}}22]{kramar2022negotiation}
János Kramár, Tom Eccles, Ian Gemp, Andrea Tacchetti, Kevin~R. McKee, Mateusz
  Malinowski, Thore Graepel, and Yoram Bachrach.
\newblock Negotiation and honesty in artificial intelligence methods for the
  board game of {Diplomacy}.
\newblock {\em Nature Communications}, 13(1):7214, December 2022.
\newblock Number: 1 Publisher: Nature Publishing Group.

\bibitem[KMN99]{kearns1999sparse}
Michael Kearns, Yishay Mansour, and Andrew~Y. Ng.
\newblock A sparse sampling algorithm for near-optimal planning in large markov
  decision processes.
\newblock In {\em Proceedings of the 16th International Joint Conference on
  Artificial Intelligence - Volume 2}, IJCAI'99, page 1324–1331, San
  Francisco, CA, USA, 1999. Morgan Kaufmann Publishers Inc.

\bibitem[LDGV{\etalchar{+}}21]{leibo2021scalable}
Joel~Z Leibo, Edgar~A Due{\~n}ez-Guzman, Alexander Vezhnevets, John~P Agapiou,
  Peter Sunehag, Raphael Koster, Jayd Matyas, Charlie Beattie, Igor Mordatch,
  and Thore Graepel.
\newblock Scalable evaluation of multi-agent reinforcement learning with
  melting pot.
\newblock In Marina Meila and Tong Zhang, editors, {\em Proceedings of the 38th
  International Conference on Machine Learning}, volume 139 of {\em Proceedings
  of Machine Learning Research}, pages 6187--6199. PMLR, 18--24 Jul 2021.

\bibitem[LS05]{littman2005polynomial}
Michael~L. Littman and Peter Stone.
\newblock A polynomial-time {N}ash equilibrium algorithm for repeated games.
\newblock {\em Decision Support Systems}, 39:55--66, 2005.

\bibitem[LS20]{lattimore2020bandit}
Tor Lattimore and Csaba Szepesv{\'a}ri.
\newblock {\em Bandit algorithms}.
\newblock Cambridge University Press, 2020.

\bibitem[LW94]{littlestone1994weighted}
Nick Littlestone and Manfred~K Warmuth.
\newblock The weighted majority algorithm.
\newblock {\em Information and computation}, 108(2):212--261, 1994.

\bibitem[LWJ22]{liu2022learning}
Qinghua Liu, Yuanhao Wang, and Chi Jin.
\newblock Learning {M}arkov games with adversarial opponents: Efficient
  algorithms and fundamental limits.
\newblock In Kamalika Chaudhuri, Stefanie Jegelka, Le~Song, Csaba Szepesvari,
  Gang Niu, and Sivan Sabato, editors, {\em Proceedings of the 39th
  International Conference on Machine Learning}, volume 162 of {\em Proceedings
  of Machine Learning Research}, pages 14036--14053. PMLR, 17--23 Jul 2022.

\bibitem[MB21]{mao2021provably}
Weichao Mao and Tamer Basar.
\newblock Provably efficient reinforcement learning in decentralized
  general-sum markov games.
\newblock {\em CoRR}, abs/2110.05682, 2021.

\bibitem[MF86]{maskin1986folk}
Eric Maskin and D~Fudenberg.
\newblock The folk theorem in repeated games with discounting or with
  incomplete information.
\newblock {\em Econometrica}, 53(3):533--554, 1986.
\newblock Reprinted in A. Rubinstein (ed.), Game Theory in Economics, London:
  Edward Elgar, 1995. Also reprinted in D. Fudenberg and D. Levine (eds.), A
  Long-Run Collaboration on Games with Long-Run Patient Players, World
  Scientific Publishers, 2009, pp. 209-230.

\bibitem[MK15]{malialis2015distributed}
Kleanthis Malialis and Daniel Kudenko.
\newblock Distributed response to network intrusions using multiagent
  reinforcement learning.
\newblock {\em Engineering Applications of Artificial Intelligence},
  41:270--284, 2015.

\bibitem[Nas51]{nash1951noncooperative}
John Nash.
\newblock Non-cooperative games.
\newblock {\em Annals of Mathematics}, 54(2):286--295, 1951.

\bibitem[Pap94]{papadimitriou1994complexity}
Christos~H. Papadimitriou.
\newblock On the complexity of the parity argument and other inefficient proofs
  of existence.
\newblock {\em J. Comput. Syst. Sci.}, 48(3):498--532, 1994.

\bibitem[Put94]{puterman_markov_1994}
Martin Puterman.
\newblock {\em Markov {Decision} {Processes}}.
\newblock John Wiley \& Sons, Ltd, 1 edition, 1994.

\bibitem[PVH{\etalchar{+}}22]{perolat2022mastering}
Julien Perolat, Bart~De Vylder, Daniel Hennes, Eugene Tarassov, Florian Strub,
  Vincent de~Boer, Paul Muller, Jerome~T. Connor, Neil Burch, Thomas Anthony,
  Stephen McAleer, Romuald Elie, Sarah~H. Cen, Zhe Wang, Audrunas Gruslys,
  Aleksandra Malysheva, Mina Khan, Sherjil Ozair, Finbarr Timbers, Toby Pohlen,
  Tom Eccles, Mark Rowland, Marc Lanctot, Jean-Baptiste Lespiau, Bilal Piot,
  Shayegan Omidshafiei, Edward Lockhart, Laurent Sifre, Nathalie Beauguerlange,
  Remi Munos, David Silver, Satinder Singh, Demis Hassabis, and Karl Tuyls.
\newblock Mastering the game of stratego with model-free multiagent
  reinforcement learning.
\newblock {\em Science}, 378(6623):990--996, 2022.

\bibitem[Rou15]{roughgarden2015intrinsic}
Tim Roughgarden.
\newblock Intrinsic robustness of the price of anarchy.
\newblock {\em Journal of the ACM}, 2015.

\bibitem[Rub16]{rubinstein2016settling}
Aviad Rubinstein.
\newblock Settling the complexity of computing approximate two-player {N}ash
  equilibria.
\newblock In {\em Annual Symposium on Foundations of Computer Science (FOCS)},
  pages 258--265. IEEE, 2016.

\bibitem[Rub18]{rubinstein2018inapproximability}
Aviad Rubinstein.
\newblock Inapproximability of {N}ash equilibrium.
\newblock {\em SIAM Journal on Computing}, 47(3):917--959, 2018.

\bibitem[SALS15]{syrgkanis2015fast}
Vasilis Syrgkanis, Alekh Agarwal, Haipeng Luo, and Robert~E Schapire.
\newblock Fast convergence of regularized learning in games.
\newblock In {\em Advances in Neural Information Processing Systems (NIPS)},
  pages 2989--2997, 2015.

\bibitem[Sha53]{shapley1953stochastic}
Lloyd Shapley.
\newblock Stochastic {Games}.
\newblock {\em PNAS}, 1953.

\bibitem[SHM{\etalchar{+}}16]{silver2016mastering}
David Silver, Aja Huang, Chris~J Maddison, Arthur Guez, Laurent Sifre, George
  Van Den~Driessche, Julian Schrittwieser, Ioannis Antonoglou, Veda
  Panneershelvam, Marc Lanctot, et~al.
\newblock Mastering the game of go with deep neural networks and tree search.
\newblock {\em nature}, 529(7587):484, 2016.

\bibitem[SMB22]{song2022when}
Ziang Song, Song Mei, and Yu~Bai.
\newblock When can we learn general-sum markov games with a large number of
  players sample-efficiently?
\newblock In {\em International Conference on Learning Representations}, 2022.

\bibitem[SS12]{shalevshwartz2012online}
Shai Shalev-Shwartz.
\newblock Online learning and online convex optimization.
\newblock {\em Found. Trends Mach. Learn.}, 4(2):107–194, feb 2012.

\bibitem[SSS16]{shalevshwartz2016safe}
Shai Shalev{-}Shwartz, Shaked Shammah, and Amnon Shashua.
\newblock Safe, multi-agent, reinforcement learning for autonomous driving.
\newblock {\em CoRR}, abs/1610.03295, 2016.

\bibitem[Vov90]{vovk1990aggregating}
Vladimir Vovk.
\newblock Aggregating strategies.
\newblock {\em Proc. of Computational Learning Theory, 1990}, 1990.

\bibitem[WAJ{\etalchar{+}}21]{weisz2021query}
Gellert Weisz, Philip Amortila, Barnab\'as Janzer, Yasin Abbasi-Yadkori, Nan
  Jiang, and Csaba Szepesvari.
\newblock On query-efficient planning in mdps under linear realizability of the
  optimal state-value function.
\newblock In Mikhail Belkin and Samory Kpotufe, editors, {\em Proceedings of
  Thirty Fourth Conference on Learning Theory}, volume 134 of {\em Proceedings
  of Machine Learning Research}, pages 4355--4385. PMLR, 15--19 Aug 2021.

\bibitem[YHAY{\etalchar{+}}22]{yin2022efficient}
Dong Yin, Botao Hao, Yasin Abbasi-Yadkori, Nevena Lazi{\'c}, and Csaba
  Szepesv{\'a}ri.
\newblock Efficient local planning with linear function approximation.
\newblock In Sanjoy Dasgupta and Nika Haghtalab, editors, {\em Proceedings of
  The 33rd International Conference on Algorithmic Learning Theory}, volume 167
  of {\em Proceedings of Machine Learning Research}, pages 1165--1192. PMLR, 29
  Mar--01 Apr 2022.

\bibitem[ZLY22]{zhan2022decentralized}
Wenhao Zhan, Jason~D Lee, and Zhuoran Yang.
\newblock Decentralized optimistic hyperpolicy mirror descent: Provably
  no-regret learning in markov games.
\newblock {\em arXiv preprint arXiv:2206.01588}, 2022.

\bibitem[ZTS{\etalchar{+}}22]{zheng2022ai}
Stephan Zheng, Alexander Trott, Sunil Srinivasa, David~C. Parkes, and Richard
  Socher.
\newblock The ai economist: Taxation policy design via two-level deep
  multiagent reinforcement learning.
\newblock {\em Science Advances}, 8(18):eabk2607, 2022.

\end{thebibliography}
\bibliographystyle{icml2023}
}

\newpage

\appendix
\icml{\onecolumn}

\renewcommand{\contentsname}{Contents of Appendix}
\addtocontents{toc}{\protect\setcounter{tocdepth}{2}}
{
  \hypersetup{hidelinks}
  \tableofcontents
}

\icml{
  \part{Additional results and discussion}

  \icml{
    \section{Tighter computational lower bounds under ETH for \PPAD}
    \label{sec:eth}
    Recall that Corollary \ref{cor:markov-formal-ppad} states that if $\PPAD \neq \PP$, then there is no constant $C> 4$ and $\poly(|\MG|)$-time algorithm which solves the $(|\MG|^C, |\MG|^{-1/C})$-\MarkCCE problem for any 2-player Markov game $\MG$. 
Using a stronger complexity-theoretic assumption, the Exponential Time Hypothesis for \PPAD \cite{rubinstein2016settling}, we can obtain a hardness result which rules out efficient algorithms even when 1)  the accuracy $\eps$ is constant, as opposed to being $|\MG|^{-1/C}$, and 2) $T$ is quasipolynomially large, as opposed to only being of polynomial size, i.e., $|\MG|^C$. 
\begin{corollary}[ETH-hardness of \MarkCCE]
  \label{cor:eth-hardness}
There is a constant $\ep_0 > 0$ such that if there exists an algorithm that solves the $(|\MG|^{o(\log |\MG|)}, \ep_0)$-\MarkCCE problem in $|\MG|^{o(\log |\MG|)}$ time, then the Exponential Time Hypothesis for \PPAD fails to hold.
\end{corollary}
Corollary \ref{cor:eth-hardness}  is an immediate consequence of \cref{thm:markov-formal} and the fact that for some absolute constant $\ep_0 > 0$, there are no polynomial-time algorithms for computing $\ep_0$-Nash equilibria in 2-player normal-form games under the Exponential Time Hypothesis for \PPAD (as shown in \cite{rubinstein2016settling}).

    }

\section{Multi-player games: Unconditional lower bounds}
\label{sec:multiplayer}

\section{Discussion and interpretation}
\label{sec:discussion}

}

\icml{
  \part{Proofs}
  }
  \section{Additional preliminaries}

  \icml{
  \subsection{Additional preliminaries for Markov games}
  
     \paragraph{Deterministic policies.}
  It will be helpful to introduce notation for
  \emph{deterministic} general (non-Markov) policies, which correspond
  to the special case of randomized policies where each policy $\sigma_{i,h}$ exclusively maps to singleton distributions. {In particular, a deterministic general policy of agent $i$ is 
  a collection of mappings $\pi_i =
  (\pi_{i,1}, \ldots, \pi_{i,H})$, where $\pi_{i,h} : \CH_{i,h-1} \times
  \MS \ra \MA_i$.} %
  We denote by $\Pidet_i$ the space of deterministic
  general policies of agent $i$, and further write $\Pidet:= \Pidet_1
  \times \cdots \times \Pidet_m$ to denote the space of \emph{joint
    deterministic policies}. {We use the convention throughout that
  deterministic policies are denoted by the letter $\pi$, whereas
  randomized policies are denoted by $\sigma$.}

\paragraph{Additional facts on regret and CCE.}
The following facts regarding deterministic policies and the definition of coarse correlated equilibria and regret are well-known:
\begin{itemize}
\item In the context of Definition \ref{def:cce} (defining an $\ep$-CCE), the maximizing policy $\sigma_i'$ can always be chosen to be determinimistic, so $\distp \in \Delta(\Pirndrnd)$ is an $\ep$-CCE if and only if $\max_{\pi_i \in \Pidet_i} V_i^{\pi_i \times \distp_{-i}} - V_i^\distp \leq \ep$.
  \item  In the context of (\ref{eq:reg-defn}) in the definition of regret, the maximum over $\sigma_i \in \Pirndrnd_i$ is always achieved by a deterministic general policy, so we have $\Reg_{i,T} = \max_{\pi_i \in \Pidet_i} \sum_{t=1}^T \prn[\big]{ V_i^{\pi_i \times \sigma_{-i}\^t} - V_i^{\sigma\^t} }$.
\end{itemize}
Next, the following standard result
shows that the uniform average of any no-regret sequence forms an
approximate coarse correlated equilibrium.
\begin{fact}[No-regret is equivalent to CCE]
  \label{fac:no-regret-cce}
 Suppose that a sequence of policies $\sigma\^1, \ldots, \sigma\^T\in
 \Pirndrnd$ satisfies $\Reg_{i,T}(\sigma\^1, \ldots, \sigma\^T ) \leq
 \ep \cdot T$ for each $i \in [m]$. Then the uniform average of these
 $T$ policies, namely the distributional policy $\ol \sigma :=
 \frac{1}{T} \sum_{t=1}^T \indic{\sigma\^t} \in \Delta(\Pirndrnd)$, is
 an $\ep$-CCE.

Likewise if a sequence of policies $\sigma\^1, \ldots, \sigma\^T\in
\Pirndrnd$ has the property that the distributional policy $\ol \sigma
:= \frac{1}{T} \sum_{t=1}^T \indic{\sigma\^t} \in \Delta(\Pirndrnd)$,
is an $\ep$-CCE, then we have $\Reg_{i,T}(\sigma\^1, \ldots, \sigma\^T ) \leq \ep \cdot T$ for all $i \in [m]$.
\end{fact}
Fact \ref{fac:no-regret-cce} is an immediate consequence of Definitions \ref{def:cce} and \ref{def:regret}. 
 }

\subsection{Nash equilibria and computational hardness.}
\label{sec:nash-prelims}
\noah{I moved this permanently here (not just icml), in part due to laziness but also b/c I think it's pretty standard}
The most foundational and well known solution concept for normal-form games is the \emph{Nash equilibrium} \cite{nash1951noncooperative}. 
\begin{defn}[$(n,\ep)$-\Nash problem]
    \label{def:nash}
    For a normal-form game $G = (M_1, \ldots, M_m)$ and $\ep > 0$, a product distribution $p \in \prod_{j=1}^m \Delta([n])$ is said to be an $\ep$-Nash equilibrium for $G$ if for all $i \in [n]$,
  \begin{align}
\max_{a_i' \in [n]} \E_{\ba \sim p} [(M_i)_{a_i', \ba_{-i}} ] - \E_{\ba \sim p}[(M_i)_\ba] \leq \ep\nonumber.
  \end{align}

  We define the \emph{$m$-player $(n,\ep)$-\Nash problem} to be the
  problem of computing an $\ep$-Nash equilibrium of a given $m$-player
  $n$-action normal-form game.\footnote{One must also take care to
    specify the bit complexity of representing a normal-form game. We
    assume that the payoffs of any normal-form game given as an
    instance to the $(n, \ep)$-\Nash problem can each be expressed
    with $\max\{n,m\}$ bits; this assumption is without loss of
    generality as long as $\ep \geq 2^{-\max\{n,m\}}$ (which it will
    be for us). \label{fn:nash-bits}}
\end{defn}
Informally, $p$ is an $\ep$-Nash equilibrium if no player $i$ can gain
more than $\ep$ in \utility by deviating to a single fixed action
$a_i'$, while all other players randomly choose their actions
according to $p$. Despite the
  intuitive appeal of Nash equilibria, they are intractable to compute:
  for any $c > 0$, it is \PPAD-hard to solve the $(n, n^{-c})$-\Nash problem, namely, to compute $n^{-c}$-approximate Nash equilibria in 2-player $n$-action normal-form games \cite{daskalakis2009complexity,chen2006computing,rubinstein2018inapproximability}. We recall that the complexity class \PPAD consists of all total search
  problems which have a polynomial-time reduction to the
  \texttt{End-of-The-Line (EOTL)} problem. \PPAD is the most well-studied
  complexity class in algorithmic game theory, and it is widely
  believed that $\PPAD\neq\PP$. We refer the reader to \cite{daskalakis2009complexity,chen2006computing,rubinstein2018inapproximability,papadimitriou1994complexity} for further background on the class \PPAD and the \texttt{EOTL} problem.

\subsection{Query complexity of Nash equilibria}
\label{sec:query-complexity-proof}
Our statistical lower bound for the \GenCCE problem in \cref{thm:statistical-lb} relies on existing query complexity lower  bounds for computing approximate Nash equilibria in $m$-player normal-form games. We first review the query complexity model for normal-form games.

\paragraph{Oracle model for normal-form games.} For $m,n \in \BN$, consider an $m$-player $n$-action normal form game $G$, specified by payoff tensors $M_1, \ldots, M_m$. Since the tensors $M_1, \ldots, M_m$ contain a total of $mn^m$ real-valued payoffs, in the setting when $m$ is large, it is unrealistic to assume that an algorithm is given the full payoff tensors as input. Therefore, prior work on computing equilibria in such games has studied the setting in which the algorithm makes adaptive \emph{oracle queries} to the payoff tensors.

In particular, the algorithm, which is allowed to be randomized, has access to a \emph{payoff oracle} $\MO_G$ for the game $G$, which works as follows. At each time step, the algorithm can choose to specify an action profile $\ba \in [n]^m$ and then query $\MO_G$ at the action profile $\ba$. The oracle $\MO_G$ then returns the payoffs $(M_1)_{\ba}, \ldots, (M_m)_{\ba}$ for each player if the action profile $\ba$ is played.

\paragraph{Query complexity lower bound for approximate Nash equilibrium.} The following theorem gives a lower bound on the number of queries any randomized algorithm needs to make to compute an approximate Nash equilibrium in an $m$-player game.
\begin{theorem}[Corollary 4.5 of \cite{rubinstein2016settling}]
  \label{thm:query-lbs}
  There is a constant $\ep_0 > 0$ so that any randomized algorithm which solves the $(2, \ep_0)$-\Nash problem for $m$-player normal-form games %
  with probability at least $2/3$ must use at least $2^{\Omega(m)}$ payoff queries. 
\end{theorem}
We remark that \cite{babichenko2016query,chen2017wellsupported} provide similar, though quantitatively weaker, lower bounds to that in Theorem \ref{thm:query-lbs}. We also emphasize that the lower bound of Theorem \ref{thm:query-lbs} applies to \emph{any} algorithm, i.e., including those which require extremely large computation time.

  \section{Proofs of lower bounds for \MarkCCE (\cref{sec:markov})}

\subsection{Preliminaries: Online density estimation}

\label{sec:online-density}
Our proof makes use of tools for online learning with the logarithmic
loss, also known as conditional density estimation. In particular, we use \dfedit{a variant of the exponential weights
  algorithm known as \emph{Vovk's aggregating algorithm} in the context
  of density estimation \cite{vovk1990aggregating,cesa2006prediction}}. We consider the following setting with two players, a \emph{Learner} and \emph{Nature}. Furthermore, there is a set $\MY$, called the \emph{outcome space}, and a set $\MX$, called the context space; for our applications it suffices to assume $\MY$ and $\MX$ are finite. For some $T \in \BN$, there are $T$ time steps $t = 1, 2, \ldots, T$. At  each time step $t \in [T]$:
\begin{itemize}
\item Nature reveals a context $x\^t \in \MX$; %
\item Having seen the context $x\^t$, the learner predicts a distribution $\wh q\^t \in \Delta(\MY)$;
\item Nature chooses an outcome $y\^t \in \MY$, and the learner suffers loss
$
\lgls\^t(\wh q\^t) := \log \left( \frac{1}{\wh q\^t(y\^t)} \right).
$
\end{itemize}
For each $t \in [T]$, we let $\MH\^t = \{ (x\^1, y\^1, \wh q\^1),
\ldots, (x\^t, y\^t, \wh q\^t) \}$ denote the history of interaction
up to step $t$; we emphasize that each context $x\^t$ may be chosen adaptively as a function of $\MH\^{t-1}$.   Let $\CF\^t$ denote the sigma-algebra generated by $(\MH\^t, x\^{t+1})$. 
We measure performance in terms of regret against a set $\MI$ of
\emph{experts}, also known as the \emph{expert setting}. Each expert $i \in \MI$ consists of a function $p_i : \MX \ra \Delta(\MY)$. The \emph{regret} of an algorithm against the expert class $\MI$ when it receives contexts $x\^1, \ldots ,x\^T$ and observes outcomes $y\^1, \ldots, y\^T$ is defined as
\begin{align}
\Reg_{\MI, T} = \sum_{t=1}^T \lgls\^t(\wh q\^t) - \min_{i \in \MI} \sum_{t=1}^T \lgls\^t(p_i(x\^t))\nonumber.
\end{align}
Note that the learner can observe the expert predictions
$\crl{p_i(x\ind{t})}_{i\in\cI}$ and use them to make its own prediction at each round $t$.
\begin{proposition}[Vovk's aggregating algorithm]
  \label{prop:vovk}
  Consider Vovk's aggregating algorithm, which predicts via
  \begin{align}
\wh q\^t(y) := \E_{i \sim \til q\^t}[p_i(x\^t)], \quad\text{where}\quad \til q\^t(i) \ldef \frac{\exp \left( -\sum_{s=1}^{t-1} \lgls\^s(p_i(x\^s))\right)}{\sum_{j \in \MI} \exp \left(-\sum_{s=1}^{t-1} \lgls\^s(p_j(x\^s))\right)}\label{eq:aggregation}.
  \end{align}
  This algorithm guarantees a regret bound of $\Reg_{\MI, T} \leq \log|\MI|$. 
\end{proposition}
Recall that for probability distributions $p,q$ on a finite set $\MB$, their total variation distance is defined as
\begin{align}
\tvd{p}{q} = \max_{\ME \subset \MB} |p(\ME) - q(\ME)|.
\end{align}
As a (standard) consequence of Proposition \ref{prop:vovk},
in the \emph{realizable} setting in which the distribution of $y\^t | x\^t$ follows
$p_{i^\st}(x\^t)$ for some fixed (unknown) expert $i^\st \in \MI$, we can obtain
a bound on the total variation distance between the algorithm's
predictions and those of $p_{i^\st}(x\^t)$.
\begin{proposition}
  \label{prop:online-tvd}
  If the distribution of outcomes is \emph{realizable}, i.e., there exists an expert $i^\st \in \MI$ so that $y\^t \sim p_{i^\st}(x\^t)  \ | \ x\^t, \MH\^{t-1}$ for all $t \in [T]$, then the predictions $\wh q\^t$ of the aggregation algorithm (\ref{eq:aggregation}) satisfy
  \begin{align}
\sum_{t=1}^T \E \left[ \tvd{\wh q\^t}{p_{i^\st}(x\^t)} \right] \leq \sqrt{T \log |\MI|}\nonumber.
  \end{align}
\end{proposition}

For completeness, we provide the proof of Proposition \ref{prop:online-tvd} here.
\begin{proof}[Proof of Proposition \ref{prop:online-tvd}]
To simplify notation, for an expert $i \in \MI$, a context $x \in \MX$, and an outcome $y \in \MY$, we write $p_i(y | x)$ to denote $p_i(x)(y)$. 
  
  Proposition \ref{prop:vovk} gives that the following inequality holds (almost surely):
  \begin{align}
\Reg_{\MI, T} = \sum_{t=1}^T \log \left( \frac{1}{\wh q\^t(y\^t)} \right) - \sum_{t=1}^T \log \left( \frac{1}{p_{i^\st}(y\^t  | x\^t)} \right) \leq \log |\MI|\nonumber.
  \end{align}
  For each $t \in [T]$, note that $\wh q\^t$ and $x\^t$ are  $\CF\^{t-1}$-measurable (by definition). Then %
  \begin{align}
    \sum_{t=1}^T   \tvd{\wh q\^t}{p_{i^\st}(x\^t)}^2 \leq & \sum_{t=1}^T   \kld{p_{i^\st}(x\^t)}{\wh q\^t}  \nonumber\\
    = & \sum_{t=1}^T  \sum_{y \in \cY} p_{i^\st}(y | x\^t) \cdot \log \left( \frac{p_{i^\st}(y | x\^t)}{\wh q\^t(y)} \right) \nonumber\\
    = &  \sum_{t=1}^T \E \left[ \log \left( \frac{1}{\wh q\^t(y\^t)} \right) - \log \left( \frac{1}{p_{i^\st}(y\^t | x\^t)}\right) \ | \ \CF\^{t-1} \right]\nonumber,
  \end{align}
  where the first inequality uses Pinsker's inequality and the final
  equality uses the fact that $y\^t \sim p_{i^\st}(x\^t) | x\^t,
  \MH\^{t-1}$. It follows that
  \[
    \En\brk*{\sum_{t=1}^T   \tvd{\wh q\^t}{p_{i^\st}(x\^t)}^2} \leq
    \E[\Reg_{\MI, T}] \leq\log\abs{\cI}.
    \]
    Jensen's inequality now gives that
  \begin{align}
    \E \left[ \sum_{t=1}^T \tvd{\wh q\^t}{p_{i^\st}(x\^t)} \right] \leq & \sqrt{T} \cdot \sqrt{ \E\left[\sum_{t=1}^T \tvd{\wh q\^t}{p_{i^\st}(x\^t)}^2 \right]}
    \leq  \sqrt{T \log |\MI|}\nonumber.
  \end{align}
\end{proof}

\subsection{Proof of Theorem \ref{thm:markov-formal}}
\label{sec:markov-proof}
\begin{proof}[Proof of Theorem \ref{thm:markov-formal}]
  Fix $n \in \BN$, which we recall represents an upper bound on the description length of the Markov game. 
  Assume that we are given an algorithm $\CB$ that solves the $(\Tn, \epn)$-\MarkCCE problem for Markov games $\MG$ satisfying $|\MG| \leq n$ in time $U$. %
  We proceed to describe an algorithm which solves the 2-player $(
  \lfloor n^{1/2}/2\rfloor, 4 \cdot \epn)$-\Nash problem in time
  $(n\Tn U)^{C_0}$, as long as $\Tn < \exp(\epn^2 \cdot
  n^{1/2}/2^{5})$. First, define $n_0 := \lfloor n^{1/2}/2\rfloor$,
  and consider an arbitrary 2-player $n_0$-action normal form $G$,
  which is specified by payoff matrices $M_1, M_2 \in [0,1]^{n_0
    \times n_0}$, so that all entries of the game can be written in
  binary using at most $n_0$ bits (recall, per footnote \ref{fn:nash-bits}, that we may assume that the entries of an instance of $(n_0, 4 \cdot \ep)$-\Nash can be specified with $n_0$ bits). Based on $G$, we construct a 2-player
  Markov game $\MG := \MG(G)$ as follows:
  \begin{defn}
    \label{def:mg-g}
    We define the game $\MG(G)$ to consist of the tuple $\MG(G) = (\MS, H, (\MA_i)_{i \in [2]}, \BP, (R_i)_{i \in [2]}, \mu)$, where:
  \begin{itemize}
\item The horizon of $\MG$ is $H = 2 \lfloor n_0/2 \rfloor$ (i.e., the largest even number at most $n_0$). %
\item Let $A = n_0$; the action spaces of the 2 agents are given by $\MA_1 = \MA_2 = [A]$.
\item There are a total of $A^2 + 1$ states: in particular, there is a state $\mf s_{(a_1, a_2)}$ for each $(a_1, a_2) \in [A]^2$, as well as a distinguished state $\mf s$, so we have:
  \begin{align}
\MS = \{ \mf s \} \cup \{ \mf s_{(a_1, a_2)} \ : \ (a_1, a_2) \in [A]^2 \}\nonumber.
  \end{align}
\item For all odd $h \in [H]$, the reward to agents $j \in [2]$ given that the action profile $(a_1, a_2)$ is played at step $h$ is given by $R_{j,h}(s, (a_1, a_2)) := \frac{1}{H} \cdot (M_j)_{a_1, a_2}$, for all $s \in \MS$. %
  All agents receive 0 reward at even steps $h \in [H]$.
\item At odd steps $h \in [H]$, if actions $a_1, a_2 \in [A]$ are taken, the game transitions to the state $\mf s_{ (a_1, a_2)}$. At even steps $h \in [H]$, the game always transitions to the state $\mf s$.
\item The initial state (i.e., at step $h=1$) is $\mf s$ (i.e., $\mu$
  is a singleton distribution supported on $\mf s$).
\end{itemize}
It is evident that this construction takes polynomial
time, and satisfies $|\MG| \leq A^2+1 \leq
n_0^2+1 \leq  n$. We will now show by applying the algorithm $\scrB$
to $\cG$, we can efficiently compute $4 \cdot \epn$-approximate Nash
equilibrium for the original game $G$. To do so, we appeal to \cref{alg:2nash}.

\end{defn}
        \begin{algorithm}[ht]
    \setstretch{1.3}
     \begin{algorithmic}[1]
       \State \textbf{Input:} 2-player, $n_0$-action normal form game $G$. 
       \State Construct the 2-player Markov game $\MG = \MG(G)$ per Definition \ref{def:mg-g}, which satisfies $|\MG| \leq n$.
       \State Call the algorithm $\CB$ on the game $\MG$, which produces a sequence $\sigma\^1, \ldots, \sigma\^T$, where each $\sigma\^t \in \PiMarkov$.
       \For{$t \in [T]$ and odd $h \in [H]$:}
       \If{$\sigma\^t_h(\mf s) \in \Delta(\MA_1) \times \Delta(\MA_2)$ is a $(4 \cdot \ep,n)$-Nash equilibrium of $G$:}
       \Return $\sigma\^t_h(\mf s)$.
       \EndIf
       \EndFor
       \State \textbf{if} the for loop terminates without returning:
       return \textbf{fail}. 
     \end{algorithmic}
     \caption{Algorithm to compute Nash equilibrium used in proof of Theorem \ref{thm:markov-formal}.}
     \label{alg:2nash}
     \end{algorithm}

Algorithm \ref{alg:2nash} proceeds as follows. First, it constructs
the 2-player Markov game $\MG(G)$ as defined above, and calls the
algorithm $\CB$, which returns a sequence $\sigma\^1, \ldots,
\sigma\^T \in \PiMarkov$ of product Markov policies with the property
that the average $\ol \sigma := \frac{1}{T} \sum_{t=1}^T
\indic{\sigma\^t}$ is an $\epn$-CCE of $\MG$. It then enumerates over
the distributions $\sigma\^t_h(\mf s) \in \Delta(\MA_1) \times
\Delta(\MA_2)$ for each $t \in [T]$ and $h \in [H]$ odd,
and checks whether each one is a $4 \cdot \epn$-approximate Nash
equilibrium of $G$. If so, the algorithm outputs such a Nash
equilibrium, and otherwise, it fails. The proof of Theorem
\ref{thm:markov-formal} is thus completed by the following lemma, which states that as long as $\ol \sigma$ is an $\epn$-CCE of $\MG$, Algorithm \ref{alg:2nash} never fails. %
\begin{lemma}[Correctness of Algorithm \ref{alg:2nash}]
  \label{lem:2nash-main}
Consider the normal form game $G$ and the Markov game $\MG = \MG(G)$ as constructed above, which has horizon $H$. For any $\ep_0 > 0$, $T \in \BN$, if $T < \exp(H \cdot \ep_0^2/ 2^8)$ and $\sigma\^1, \ldots, \sigma\^T \in \PiMarkov$ are product Markov policies so that $\frac{1}{T} \sum_{t=1}^T \indic{\sigma\^t}$ is an $(\ep_0/4)$-CCE of $\MG$, then there is some odd $h \in [H]$ and $t \in [T]$ so that $\sigma_h\^t(\mf s)$ is an $\ep_0$-Nash equilibrium of $G$. 
\end{lemma}
The proof of Lemma \ref{lem:2nash-main} is given below. Applying Lemma \ref{lem:2nash-main} with $\ep_0 = 4 \epn$ (which is a valid application since $T < \exp(n_0 \cdot (4\epn)^2 / 2^{8})$ by our assumption on $\Tn, \epn$), yields that Algorithm \ref{alg:2nash} always finds a $4\epn$-Nash equilibrium of the $n_0$-action normal form game $G$, thus solving the given instance of the $(n_0, 4\cdot \epn)$-\Nash problem. Furthermore, it is straightforward to see that Algorithm \ref{alg:2nash} runs in time $U + (nT)^{C_0} \leq  (UnT)^{C_0}$, for some constant $C_0 \geq 1$.

   \end{proof}

   \begin{proof}[Proof of Lemma \ref{lem:2nash-main}]
  Consider a sequence of product Markov policies $\sigma\^1, \ldots,
  \sigma\^T$ with the property that the average $\ol \sigma =
  \frac{1}{T} \sum_{t=1}^T\indic{ \sigma\^T}$ is an $(\ep_0/4)$-CCE of
  $\MG$. For all odd $h \in [H]$ and $j \in [2]$, let $p\^t_{j,h} :=
  \sigma\^t_{j,h}(\mf s) \in \Delta(\MA_j)$, which is the distribution
  played under $\sigma\^t$ by player $j$ at step $h$ (at the unique
  state $\mf s$ with positive probability of being reached at step
  $h$). For odd $h$, we have $\sigma_h\^t(\mf s) = p_{1,h}\^t \times
  p_{2,h}\^t$, and our goal is to show that for some odd $h \in [H]$
  and $t \in [T]$, $p_{1,h}\^t \times p_{2,h}\^t$ is an $\ep_0$-Nash
  equilibrium of $G$. To proceed, suppose for the sake of contradiction that this is not the case.
  
Let  us write $\MO_H := \{ h \in [H] : h\ \rm{ odd} \}$ to denote the
set of odd-numbered steps, and $\ME_H = [H] \backslash \MO_H$ to
denote the set of even-numbered steps. Let $H_0 = |\MO_H| = |\ME_H| =
H/2$. We first note that for $j \in [2]$, agent $j$'s value under the mixture policy $\ol \sigma$ is given as follows:
\begin{align}
V_{j}^{\ol \sigma} = \frac{1}{TH} \sum_{t=1}^T \sum_{h \in \MO_H} \E_{a_1 \sim p_{1,h}\^t, a_2 \sim p_{2,h}\^t} \left[ (M_j)_{a_1, a_2} \right]\nonumber.
\end{align}
 For each $j \in [2]$, we will derive a contradiction by constructing
 a (non-Markov) deviation policy for player $j$ in $\MG$, denoted
 $\pi_j^\dagger \in \Pidet_j$, which will give player $j$ a
 significant gain in value against the policy $\ol \sigma$. To do so,
 we need to specify $\pi_{j,h}^\dagger(\tau_{j,h-1}, s_h) \in \MA_j$,
 for all $\tau_{j,h-1} \in \CH_{j,h-1}$ and $s_h \in \MS$; note that
 we may restrict our attention only to histories $\tau_{j,h_0-1}$ that occur with positive probability under the transitions of $\MG$.

Fix any $h_0 \in [H]$, $\tau_{j,h_0-1} \in \CH_{j,h_0-1}$, and
$s_{h_0} \in \MS$. If $\tau_{j,h_0-1}$ occurs with positive
probability under the transitions of $\MG$, then for each $h \in
\MO_H$, $h < h_0-1$ and both $j' \in [2]$, the action played by agent $j'$
at step $h$ is determined by $\tau_{j,h}$. Namely, if the state at
step $h+1$ of $\tau_{j,h_0-1}$ is $\mf s_{ (a_1', a_2')}$, then player
$j'$ played action $a_{j}'$ at step $h$. So, for each $h \in
\MO_H$ with $h < h_0-1$, we may define $(a_{1,h}, a_{2,h})$ as the
action profile played at step $h$, which is a measurable function of $\tau_{j,h_0-1}$. 
  With this in mind, we define $\pi_{j,h_0}^\dagger(\tau_{j,h_0-1},
  s_{h_0})$ by applying Vovk's aggregating algorithm (Proposition
  \ref{prop:online-tvd}) as follows.
\begin{enumerate}
\item If $h_0$ is even, play an arbitrary action (note that the actions at even-numbered steps have no influence on the transitions or rewards).
\item If $h_0$ is odd, define $\wh q_{j,h_0} \in \Delta(\MA_j)$, by $\wh q_{j,h_0} := \E_{t \sim \til q_{j,h_0}} [ p_{-j, h}\^t]$, where $\til q_{j,h_0} \in \Delta([T])$ is defined as follows: for $t \in [T]$,
  \begin{align}
\til q_{j,h_0}(t) := \frac{\exp \left( - \sum_{h < h_0: \ h \in \MO_H} \log \left( \frac{1}{p\^t_{-j, h}(a_{-j,h})} \right)\right)}{\sum_{t'=1}^T\exp \left( - \sum_{h < h_0:\ h \in \MO_H} \log \left( \frac{1}{p\^{t'}_{-j, h}(a_{-j,h})} \right)\right)}\nonumber.
  \end{align}
  Note that $\wh q_{j,h_0}$ is a function of $\tau_{j,h_0-1}$ via the
  action profiles $\crl*{(a_{1,h}, a_{2,h})}_{h<h_0:h\in\cO_H}$; to
  simplify notation, we suppress this
  dependence. %
  \item Then for any state $s_{h_0} \in \MS$, define $\pi_{j,h_0}^\dagger(\tau_{j,h_0-1}, s_{h_0})$ to be  a best response to $\wh q_{j,h_0}$, namely 
    \begin{align}
      \pi_{j,h_0}^\dagger(\tau_{j,h_0-1}, s_{h_0}) := \argmax_{a_j \in \MA_j} \E_{a_{-j} \sim \wh q_{j,h_0}} \left[ R_{j,h}(\mf s_{h_0}, (a_1, a_2)) \right]=\argmax_{a_j \in \MA_j} \E_{a_{-j} \sim \wh q_{j,h_0}} \left[ (M_j)_{a_1, a_2} \right].\label{eq:define-br-ol}
    \end{align}
  
\end{enumerate}
Note that, for odd $h_0$, the distribution $\wh q_{j,h_0} \in
\Delta(\MA_j)$ defined above can be viewed as an application of Vovk's online aggregation algorithm at step $(h_0+1)/2$ in the following setting:
the number of steps ($T$, in the notation of Proposition \ref{prop:online-tvd};
note that $T$ plays a different role in the present proof) is
$H_0=H/2$, the context space is $\MO_H$, and the outcome space is
$\MA_{-j}$.\footnote{Here $-j$ denotes the index of the player who is
  not $j$.} There are $T$ experts $\til p\^1, \ldots, \til p\^T$
(i.e., we have $\cI=\crl*{\til p\ind{t}}_{t\in\brk{T}}$), whose predictions on a context $h \in \MO_H$ are defined as follows: the expert $\til p\^t$ predicts $\til p\^t(h) := p_{-j,h}\^t$.  %
Then, the distribution $\wh q_{j,h_0}$ is obtained by updating the aggregation algorithm with the context-observation pairs $(h, a_{-j, h})$, for \emph{odd} values of $h < h_0$. %

We next analyze the value of $V_{j}^{\pi_j^\dagger, \ol \sigma_{-j}}$
for $j \in [2]$ to show that the deviation strategy we have defined indeed obtains
significant gain. To do so, recall that this value represents the
payoff for player $j$ under the process in which we draw an index
$\tstar\in\brk*{T}$ uniformly at random, then for each step
$h\in\brk{H}$, player $j$ plays according to $\pi_j^\dagger$ and player $-j$ plays according to $\sigma_{-j}\^{t^\st}$. (In particular, at odd-numbered steps, player $-j$ plays according to $p_{-j,h}\^{t^\st}$.) %
We recall that $\E_{\pi_j^\dagger \times \ol
  \sigma_{-j}} \brk*{\cdot}$ denotes the expectation under this process. We let $\tau_{j,h-1} \in \CH_{j,h-1}$ denote the random variable which
is the history observed by player $j$ in this setup, i.e., when the
policy played is $\pi_j^\dagger \times \ol \sigma_{-j}$, and let
$\crl*{(a_{1,h}, a_{2,h})}_{h\in\cO_H}$ denote the action
profiles for odd rounds, which are a measurable function of each
player's trajectory. 

We apply Proposition \ref{prop:online-tvd} with the time horizon as
$H_0$, and with the set of experts set to $\cI\ldef{}\{ \til p\^1,
\ldots, \til p\^T \}$ as defined above. The context sequence the
sequence of increasing values of $h \in \MO_H$, and for each $h \in
\MO_H$, the outcome at step $(h+1)/2$ (for which the context is $h$)
is distributed as $a_{-j,h} \sim \til p\^{t^\st}(h) = p\^{t^\st}_{-j,
  h}$ conditioned on $\tstar$, which in particular satisfies the realizability assumption stated in Proposition \ref{prop:online-tvd}.  %
Then, since (as remarked above), the distributions $\wh q_{j,h}$, for
$h \in \MO_H$, are exactly the predictions made by Vovk's aggregating
algorithm, Proposition \ref{prop:online-tvd} gives that\footnote{In
  fact, Proposition \ref{prop:online-tvd} implies that a similar bound holds uniformly
  for each possible realization of $\tstar$, but
  \cref{eq:sum-tv-bound} suffices for our purposes.}
\begin{align}
\E_{\pi_j^\dagger \times \ol \sigma_{-j}} \left[ \sum_{h \in \MO_H}\tvd{\wh q_{j,h}}{p_{-j,h}\^{t^\st}} \right] =\E_{\pi_j^\dagger \times \ol \sigma_{-j}} \left[ \sum_{h \in \MO_H}\tvd{\wh q_{j,h}}{\til p\^{t^\st}(h)} \right] \leq \sqrt{H_0 \log T}\label{eq:sum-tv-bound}.
\end{align}
Recall that we have assumed for the sake of contradiction that
$p_{1,h}\^{t}\times p_{2,h}\^{t}$ is not an $\ep_0$-Nash equilibrium
of $G$ for each $h \in [H]$ and $t\in\brk{T}$. Consider a fixed draw
of the random variable $\tstar\in\brk{T}$ defined above. Then it holds
that for $j \in [2]$ and $h \in [H]$, defining
\begin{align}
\ep_{0,j,h} := \max_{a_j \in [A]} \E_{a_{-j} \sim p_{-j,h}\^{t^\st}} \left[ (M_j)_{a_1, a_2} \right] - \E_{a_1 \sim p_{1,h}\^{t^\st}, a_2 \sim p_{2,h}\^{t^\st}} \left[ (M_j)_{a_1, a_2} \right],\label{eq:disp0}
\end{align}
we have $\ep_{0,1,h} + \ep_{0,2,h} \geq \ep_0$. 
Consider any $j \in [2]$, $h \in \MO_H$, and a history $\tau_{j,h-1}
\in \CH_{j,h-1}$ of agent $j$ up to step $h-1$ (conditioned on
$\tstar$). Let us write $\delta_{-j,h}\^{t^\st} :=
\tvd{p_{-j,h}\^{t^\st}}{\wh q_{j,h}}$; note that
$\delta_{-j,h}\^{t^\st}$ is a function of $\tau_{j,h-1}$, through its
dependence on $\wh q_{j,h}$. We have, by the definition of
$\pi_{j,h}^\dagger(\tau_{j,h-1}, s_h)$ in (\ref{eq:define-br-ol}) and
the definition of $\delta_{-j,h}\^{t^\st}$,
\begin{align}
  \E_{a_{-j} \sim p_{-j,h}\^{t^\st}} \left[ (M_j)_{\pi_{h,j}^\dagger(\tau_{j,h-1}, \mf s), a_{-j}}\ | \ t^\st,\ \tau_{j,h-1} \right] \geq & \E_{a_{-j} \sim \wh q_{j,h}} \left[ (M_j)_{\pi_{h,j}^\dagger(\tau_{j,h-1}, \mf s), a_{-j}}\ | \ t^\st,\ \tau_{j,h-1} \right] - \delta_{-j,h}\^{t^\st}\nonumber\\
  = & \max_{a_j \in [A]} \E_{a_{-j} \sim \wh q_{j,h}} \left[ (M_j)_{a_j, a_{-j}} \ | \ t^\st,\ \tau_{j,h-1} \right] - \delta_{h,-j}\^{t^\st}\nonumber\\
  \geq & \max_{a_j \in [A]} \E_{a_{-j} \sim p_{h,-j}\^{t^\st}} \left[ (M_j)_{a_j, a_{-j}} \right] - 2 \delta_{-j,h}\^{t^\st}\label{eq:disp1}.
\end{align}
Combining \cref{eq:disp0} and \cref{eq:disp1}, we get that for any fixed $h \in \MO_H$, $j \in [2]$, and $\tau_{j,h-1} \in \CH_{j,h-1}$,
\begin{align}
\E_{a_{-j} \sim p_{-j,h}\^{t^\st}} \left[ (M_j)_{\pi_{j,h}^\dagger(\tau_{j,h-1}, \mf s), a_{-j}}\ | \ t^\st,\ \tau_{j,h-1} \right] - \E_{a_1 \sim p_{1,h}\^{t^\st}, a_2 \sim p_{2,h}\^{t^\st}} \left[ (M_j)_{a_1, a_2} \right]  > \ep_{0,j,h} - 2 \delta_{-j,h}\^{t^\st}\label{eq:delta-deviation}.
\end{align}
Averaging over the draw of $t^\st \in [T]$, which we recall is chosen uniformly, we see that %
\begin{align}
  & \sum_{j \in [2]} V_{j}^{\pi_j^\dagger \times \ol \sigma_{-j}} - V_{j}^{\ol \sigma} \nonumber\\
  = &  \frac{1}{T} \sum_{t=1}^T \sum_{j \in [2]} V_{j}^{\pi_j^\dagger\times \sigma\^t_{-j}} - V_{j}^{\sigma\^t}\label{eq:split-into-k}\\
  = & \frac{1}{T} \sum_{t=1}^T \sum_{j \in [2]} \E_{\pi_j^\dagger \times \sigma_{-j}\^t} \left[ \sum_{h \in \MO_H} \E_{a_{-j} \sim p_{-j,h}\^t}[R_{j,h}(\mf s, (\pi_{j,h}^\dagger(\tau_{j,h-1}, \mf s), a_{-j})) \ | \ t,\ \tau_{j,h-1}] - \E_{a_1 \sim p_{1,h}\^t, a_2 \sim p_{2,h}\^t} [ R_{j,h}(\mf s, (a_1, a_2)) ]\right]\nonumber\\
  = & \frac{1}{TH} \sum_{t=1}^T \sum_{j \in [2]} \E_{\pi_j^\dagger \times \sigma_{-j}\^t} \left[ \sum_{h \in \MO_H} \E_{a_{-j} \sim p_{-j,h}\^t} [ (M_j)_{\pi_{j,h}^\dagger(\tau_{j,h-1}, \mf s), a_{-j}} \ | \ t, \ \tau_{j,h-1}] - \E_{a_1 \sim p_{1,h}\^t, a_2 \sim p_{2,h}\^t} [ (M_j)_{a_1, a_2} ]\right]\nonumber\\
  \geq & \frac{1}{TH} \sum_{t=1}^T \sum_{j \in [2]} \E_{\pi_j^\dagger \times \sigma_{-j}\^t} \left[ \sum_{h \in \MO_H} \left( \ep_{0,j,h} - 2 \delta_{-j,h}\^{t}  \right)\right]\label{eq:use-deviation}\\
  \geq & \frac{\ep_0}{2} - \frac{2}{TH} \sum_{t=1}^T 2 \sqrt{H_0 \log T} \geq \frac{\ep_0}{2} - 4 \sqrt{ \log (T)/H}\label{eq:use-phk-tvd},
\end{align}
where (\ref{eq:split-into-k}) follows from the definition $\ol \sigma
= \frac{1}{T} \sum_{t=1}^T \indic{\sigma\^t}$,
(\ref{eq:use-deviation}) follows from (\ref{eq:delta-deviation}), and
(\ref{eq:use-phk-tvd}) uses (\ref{eq:sum-tv-bound}). As long as $T <
\exp(H \cdot(\ep_0/16)^2)$, the this expression is bounded below by
$\ep_0/4$, meaning that $\ol\sigma$ is not an $\ep_0/4$-approximate
CCE. This completes the contradiction.
   \end{proof}

   \section{Proofs of lower bounds for \GenCCE (\cref{sec:nonmarkov,sec:multiplayer})}
   \label{sec:nonmarkov-proof}

   In this section we prove our computational lower bounds for solving the \GenCCE problem with $m = 3$ players (Theorem \ref{thm:nonmarkov-formal} and Corollary \ref{cor:nonmarkov-formal-ppad}), as well as our statistical lower bound for solving the \GenCCE problem with a general number $m$ of players (Theorem \ref{thm:statistical-lb}).

Both theorems are proven as consequences of a more general result
given in Theorem \ref{thm:nonmarkov-multiplayer} below, which reduces the \Nash problem in $m$-player normal-form games to the \GenCCE problem in $(m+1)$-player Markov games. %
In more detail, the theorem shows that (a) if an algorithm for \GenCCE makes few calls to a generative model oracle, then we get an algorithm for the \Nash problem with few calls to a payoff oracle (see Section \ref{sec:query-complexity-proof} for background on the payoff oracle for the \Nash problem), and (b) if the algorithm for \GenCCE is \emph{computationally} efficient, then so is the algorithm for the \Nash problem.
\begin{theorem}
  \label{thm:nonmarkov-multiplayer}
  There is a constant $C_0 > 0$ so that the following holds. 
  Consider $n,m \in \BN$, and 
suppose $\Tnm, \Nnm, \Qnm \in \BN$ and $\epnm > 0$ satisfy $1 < \Tnm < \exp \left( \frac{\epnm^2 \cdot \lfloor n/m \rfloor}{m^2} \right)$.
  Suppose there is an algorithm $\mathscr{B}$ which, given a generative model oracle for a $(m+1)$-player Markov game $\MG$ with $|\MG| \leq n$, solves the $(\Tnm, \epnm, \Nnm)$-\GenCCE problem for $\MG$ using $\Qnm$ generative model oracle queries. Then the following conclusions hold:
  \begin{itemize}
  \item   For any $\delta > 0$, the $m$-player $(\lfloor n/m \rfloor, 16(m+1) \cdot \epnm )$-\Nash problem for any normal-form game $G$ can be solved, with failure probability $\delta$, using at most $C_0 \cdot (\Qnm \cdot \log(1/\delta)) + (\log(1/\delta) \cdot nm/\epnm)^{C_0}$ queries to a payoff oracle $\MO_G$ for $G$. 
  \item If the algorithm $\mathscr{B}$ additionally runs in time $U$ for some $U \in \BN$, then the algorithm solving \Nash from the previous bullet point runs in time $(nm\Tnm \Nnm U \log(1/\delta) / \epnm)^{C_0}$. %
  \end{itemize}
\end{theorem}
Theorem \ref{thm:nonmarkov-formal} follows directly from Theorem \ref{thm:nonmarkov-multiplayer} by taking $m=2$.
\begin{proof}[Proof of Theorem \ref{thm:nonmarkov-formal}]
  Suppose there is an algorithm which, given the description of any 3-player Markov game $\MG$ with $|\MG| \leq n$, solves the $(\Tnm, \epnm, \Nnm)$-\GenCCE problem in time $U$. Such an algorithm immediately yields an algorithm which can solve the $(\Tnm, \epnm, \Nnm)$-\GenCCE problem in time $U + |\MG|^{O(1)}$ using only a generative model oracle, since the exact description of the Markov game can be obtained with $HS |\MA| \leq HS (\max_i A_i)^3 \leq |\MG|^5$ queries to the generative model (across all $(h,s,\ba)$ tuples). 
  We can now solve the problem of computing a $50\cdot \epnm$-Nash equilibrium of a given 2-player $\lfloor n/2 \rfloor$-action normal form game $G$ as follows. %
  We simply apply the algorithm of Theorem \ref{thm:nonmarkov-multiplayer} with $m=2$, noting that the oracle $\MO_G$ in the theorem statement can be implemented by reading the corresponding bits of input of the input game $G$. The second bullet point yields that this algorithm takes time $(nTNU\log(1/\delta)/\ep)^{C_0}$, for some constant $C_0$. Furthermore, the assumption $T < \exp(\ep^2 \cdot \lfloor n/m \rfloor / m^2)$ of Theorem \ref{thm:nonmarkov-multiplayer} is implied by the assumption that $T < \exp(\ep^2 n / 16)$ of Theorem \ref{thm:nonmarkov-formal}. 
\end{proof}
In a similar manner, Theorem \ref{thm:statistical-lb} follows from Theorem \ref{thm:nonmarkov-multiplayer} by applying Theorem \ref{thm:query-lbs}, which states that there is no randomized algorithm that finds approximate Nash equilibria of $m$-player, 2-action normal form games in time $2^{o(m)}$.
\begin{proof}[Proof of Theorem \ref{thm:statistical-lb}]
    Let $\ep_0$ be the constant from Theorem \ref{thm:query-lbs}, and consider any $m \geq 3$. 
    Suppose there is an algorithm which, for any $m$-player Markov
    game $\MG$ with $|\MG| \leq 2m^6$, makes $Q$ oracle queries to a
    generative model oracle for $\MG$, and solves the $(T,\ep_0/(10m),
    N)$-\GenCCE problem for $\MG$ for some $T, N \in \BN$ so that $T <
    \exp(cm)$, for a sufficiently small absolute constant $c$. %
    Then, by Theorem \ref{thm:nonmarkov-multiplayer} with $\ep = \ep_0/(10m)$ and $n = m^6$ (which ensures that $\Tnm < \exp((\ep_0/(10m))^2 \cdot \lfloor n/m \rfloor / m^2)$ as long as $c$ is sufficiently small), there is an algorithm which solves the $(m^5, \ep_0)$-\Nash problem---and thus the $(2, \ep_0)$-\Nash problem---for $(m-1)$-player games with failure probability $1/3$, using $O(\Qnm) + m^{O(1)}$ queries to a payoff oracle. But by Theorem \ref{thm:query-lbs}, any such algorithm requires $2^{\Omega(m)}$ queries to a payoff oracle. It follows that $Q \geq 2^{\Omega(m)}$, as desired. 
\end{proof}

\subsection{Proof of Theorem \ref{thm:nonmarkov-multiplayer}}

\begin{proof}[Proof of Theorem \ref{thm:nonmarkov-multiplayer}]
  Fix any $m \geq 2$, $n \in \BN$.  Suppose we are given an algorithm
  $\CB$ that solves the $(m+1)$-player $(\Tnm, \epnm, \Nnm)$-\GenCCE
  problem for Markov games $\MG$ satisfying $|\MG| \leq n$, running in
  time $U$ and using at most $Q$ generative model queries.
  We proceed to describe an algorithm which solves the $m$-player
  $(\lfloor n/m \rfloor, 16(m+1)\cdot \epnm)$-\Nash problem using $C_0
  \cdot (\Qnm \cdot \log(1/\delta)) + (\log(1/\delta) \cdot
  nm/\epnm)^{C_0}$ queries to a payoff oracle, and
  running in time $(nm\Tnm \Nnm U \log(1/\delta) / \epnm)^{C_0}$,
  where $\delta$ represents the failure probability. Define $n_0 :=
  \lfloor n/m \rfloor$, and assume we are given an arbitrary $m$-player $n_0$-action normal form $G$, which is specified by payoff matrices $M_1, \ldots, M_m \in [0,1]^{n_0 \times \cdots \times n_0}$. %
  We assume that all entries of each of the matrices $M_j$ have only
  the most significant $\max\{ n_0, \lceil \log 1/\epnm \rceil\}$ bits
  nonzero; this assumption is without loss of generality, since by
  truncating the utilities to satisfy this assumption, we change all
  payoffs by at most $\epnm$, which degrades the quality of any
  approximate equilibrium by at most $2\epnm$ (in addition, we have $\lceil \log 1/\ep \rceil \leq n_0$ since we have assumed $1 < T < \exp(\ep^2 n_0/m^2)$).
  We assume $\ep \leq 1/2$ without loss of generality. %
  Based on $G$, we construct an $(m+1)$-player Markov game $\MG := \MG(G)$ as follows.
\begin{defn}
  \label{def:mg-g-multiplayers}
      We define the Markov game $\MG(G)$ as the tuple $\MG(G) = (\MS, H, (\MA_i)_{i \in [2]}, \BP, (R_i)_{i \in [2]}, \mu)$, where:
\begin{itemize}
\item The horizon of $\MG$ is chosen to be the power of $2$ satisfying
   $n_0 \leq H < 2n_0$.%
\item Let $A \ldef n_0$. The action spaces of agents $1, 2, \ldots, m$ are given by $\MA_1 = \cdots = \MA_m = [A]$. The action space of agent $m+1$ is %
  \begin{align}
\MA_{m+1} = \{ (j, a_j) \ : \ j \in [m], a_j\in \MA_j \}\nonumber,
  \end{align}
  so that $|\MA_{m+1}| = Am \leq n$. 

 We write $\MA = \prod_{j=1}^m \MA_j$ to denote the joint action space of the first $m$ agents, and $\ol \MA  := \prod_{j=1}^{m+1} \MA_j$ to denote the joint action space of all agents. %

\item There is a single state, denoted by $\mf s$, i.e., $\MS = \{ \mf
  s \}$ (in particular, $\mu$ is a singleton distribution supported on
  $\mf s$).

\item For all $h \in [H]$, the reward for agent $j \in [m+1]$, given
  an action profile $\ba = (a_1, \ldots, a_{m+1})$ at the unique state
  $\mf s$, is as follows: writing $a_{m+1} = (j', a_{j'}')$, we have%
  \begin{align}
R_{j,h}(\mf s, \ba) =  \ol R_{j,h}(\mf s, \ba) + \frac{1}{H} \cdot 2^{-3 \lceil \log 1/\ep \rceil} \cdot \enc(\ba)\label{eq:rewards-multiplayers-true},
  \end{align}
  where $\ol R_{j,h}(\mf s, \ba)$ is defined per the kibitzer construction of \cite{borgs2008myth}:
\begin{align}
   \ol R_{j,h}(\mf s, \ba) := \begin{cases}
    0 &: j \not \in \{j', m+1\} \\
  \frac{1}{H} \cdot \left((M_j)_{a_1, \ldots, a_m} - (M_j)_{a_1, \ldots, a_{j'}', \ldots, a_m}\right) &: j=j' \\
  \frac{1}{H} \cdot \left( (M_j)_{a_1, \ldots, a_{j'}', \ldots, a_m} - (M_j)_{a_1, \ldots, a_m} \right) &: j=m+1.
  \end{cases}\label{eq:rewards-multiplayers}
\end{align}
In (\ref{eq:rewards-multiplayers-true}) above,  $\enc(\ba) \in [0,1]$ is the binary representation of a binary
encoding of the action profile $\ba$. In particular, if the binary encoding of $\ba$ is $(b_1, \ldots, b_N)$, with $b_i \in \{0,1\}$, then $\enc(\ba) = \sum_{i=1}^N 2^{-i} \cdot b_i$. Note that $\enc(\ba)$ takes $N = O(m
\log n_0) \leq O(m \log n)$ bits to specify.
\end{itemize}
\end{defn}

      \begin{algorithm}[H]
    \setstretch{1.3}
     \begin{algorithmic}[1]
       \State \textbf{Input:}
       \Statex \quad Parameters $n,n_0, m, T \in \BN$, $\delta = \ep / (6H)$, $K = \lceil 4  \log(mn_0/\delta) / \ep^2 \rceil$. 
       \Statex \quad An $m$-player, $n_0$-action normal form game $G$, with utilies accessible by oracle $\MO_G$. 
       \Statex \quad An algorithm $\CB$ for computing approximate CCE of Markov games.
       \State \label{line:call-b}Call the algorithm $\CB$ on the
       $(m+1)$-player Markov game $\MG = \MG(G)$ constructed as in Definition \ref{def:mg-g-multiplayers}, which produces a sequence $\sigma\^1, \ldots, \sigma\^T$, where each $ \sigma\^t = ( \sigma\^t_1, \ldots,  \sigma\^t_{m+1})$ with $ \sigma\^t_j \in \Pirndrnd_j$. Here, we use the oracle $\MO_G$ to simulate generative model oracle queries made by $\CB$.  
       \State Draw $t^\st \in [T]$ uniformly at random.
       \State For each $j \in [m]$, initialize $\tau_{j,0}$ to be an empty trajectory. 
       \For{$h \in [H]$:} \arxiv{\Comment{\emph{Simulate a trajectory from $\MG$}}}
       \State Set $s_h = \mf s$ (per the transitions of $\MG$). %
       \State \label{line:compute-hatq-multiplayer}\multiline{ For each $j \in [m]$, define $\wh q_{j,h} := \E_{t \sim \til q_{j,h}} \left[ \sigma_{j,h}\^t(\tau_{j,h-1}, s_h) \right] \in \Delta(\MA_j)$, where $\til q_{j,h} \in \Delta([T])$ is defined as follows: for $t \in [T]$,}
       \begin{align}
\til q_{j,h}(t) :=  \frac{\exp \left( - \sum_{g < h} \log \left( \frac{1}{\sigma\^t_{j, g}(a_{j,g} | \tau_{j,g-1},  s_g)} \right)\right)}{\sum_{t'=1}^T\exp \left( - \sum_{g < h} \log \left( \frac{1}{\sigma\^{t'}_{j, g}(a_{j,g} | \tau_{j,g-1},  s_g)} \right)\right)}\nonumber.
       \end{align}
       \vspace{-5pt}
       \State Draw $K$ i.i.d.~samples $\ba_h^1, \ldots, \ba_h^K \sim \bigtimes_{j \in [m]} \wh q_{j,h}$. \label{line:draw-samples}
       \State \multiline{For each $a' \in \MA_{m+1}$, define $\wh R_{m+1,h}(a'):= \frac{1}{K} \sum_{k=1}^K R_{m+1,h}(s_h, (\ba_h^k, a'))$. Here, we use the oracle $\MO_G$ to compute $R_{m+1,h}(s_h, (\ba_h^k, a'))$ for each tuple $(\ba_h^k, a')$. \label{line:mp1-br}}

       \State For each $j \in [m]$, draw $a_{j,h} \sim  \sigma_{j,h}\^{t^\st}(\cdot  | \tau_{j,h-1}, s_h)$.\label{line:draw-ajh}
       \State \label{line:mp1-br} \multiline{Choose the action $a_{m+1,h}$ of player $m+1$ as follows: \emph{(Action $a_{m+1,h}$ is corresponds to the action selected by the policy $\pi_{m+1}^\dagger$ of player $m+1$ defined within the proof of Lemma \ref{lem:algorithmic-lemma}; this policy is well-defined because the action profiles of all players $i\in\brk{m}$ can be extracted from the lower-order bits of player $m+1$'s reward)}}
       \begin{align}
         a_{m+1,h} := \argmax_{a' \in \MA_{m+1}}\left\{\wh
         R_{m+1,h}(a') \right\} \label{eq:choose-mp1-action}.
       \end{align}         \vspace{-10pt}
       \State \multiline{ For each $j \in [m+1]$, let $r_{j,h} = R_{j,h}(s_h, (a_{1,h}, \ldots, a_{m+1,h}))$. \label{line:reward-simulate}} 
       \State \multiline{ Each player $j$ constructs $\tau_{j,h}$ by updating $\tau_{j,h-1}$ with $(s_h, a_{j,h}, r_{j,h})$.}
\If{$\wh R_{m+1,h}(a_{m+1,h}) \leq 14(m+1) \cdot \ep/H$}\label{line:check-nash}
\Return $\wh q_h := \bigtimes_{j \in [m]} \wh q_{j,h}$ as a candidate approximate Nash equilibrium for $G$.
       \EndIf       
       \EndFor
       \State \textbf{if} the for loop terminates without returning: return \textbf{fail}. 
     \end{algorithmic}
     \caption{Algorithm to compute Nash equilibrium used in proof of Theorem \ref{thm:nonmarkov-multiplayer}.}
     \label{alg:3nash}
   \end{algorithm}

   It is evident that this construction takes
   polynomial time and satisfies $|\MG| \leq mn_0
   \leq n$. Furthermore, it is clear that a single generative model
   oracle call for the Markov game $\MG$ (per Definition
   \ref{def:generative-model}) can be implemented using at most 2
   calls to the oracle $\MO_G$ for the normal-form game $G$. %
   We will now show by applying the algorithm $\scrB$
to $\cG$, we can efficiently (in terms of runtime and oracle calls) compute a $16(m+1)\cdot \epn$-approximate Nash
equilibrium for the original game $G$. To do so, we appeal to \cref{alg:3nash}.

   Algorithm \ref{alg:3nash} proceeds as follows. First, it calls the
   algorithm $\CB$ on the $(m+1)$-player Markov game $\MG(G)$, using the oracle $\MO_G$ to simulate $\CB$'s calls to the generative model oracle for $\MG$. 
By assumption, the algorithm $\CB$ returns a sequence $\sigma\^1, \ldots, \sigma\^T$ of product policies of the form $ \sigma\^t = ( \sigma\^t_1, \ldots,  \sigma\^t_{m+1})$, so that each $ \sigma\^t_j \in \Pirndrnd_j$ is $N$-computable, and so that the average $\ol \sigma := \frac{1}{T} \sum_{t=1}^T \indic{\sigma\^t}$ is an $\ep$-CCE of $\MG$.
   Next, Algorithm \ref{alg:3nash} samples a trajectory from $\MG$ in
   which:
   \begin{itemize}
   \item Players $1, \ldots, m$ each play according to a policy
     $ \sigma\^{t^\st}$ for an index $t^\st \in [T]$ chosen uniformly
     at the start of the episode.
   \item Player $m+1$ plays
     according to a strategy that, at each step $h \in [H]$, computes
     distributions $\wh q_{j,h}$ representing its ``belief'' of what action
     each player $j \in [m]$ will play at step $h$ (Line
     \ref{line:compute-hatq-multiplayer}), and plays an approximate
     best response to the product of the strategies $\wh q_{j,h}$,
     $j \in [m]$ (Line \ref{line:mp1-br}).
   \end{itemize}
   In order avoid exponential dependence on the number of players $m$
   when computing an approximate best response to $\bigtimes_{j \in
     [m]} \wh q_{j,h}$, we draw $K := \lceil 4 \log
   (mn_0/\delta)/\ep^2 \rceil$ (for $\delta = \ep/(6H)$) samples from
   $\bigtimes_{j \in [m]} \wh q_{j,h}$ and use these samples to
   compute the best response. In particular, letting
   $\ba_h^K\in\cA$ denote the $k$th sampled action profile, we  construct a function
   $\wh R_{m+1,h} : \MA_{m+1} \ra \BR$ in Lines
   \ref{line:draw-samples} and \ref{line:mp1-br} which, for each $a'
   \in \MA_{m+1}$, is defined as the average over samples
   $\crl{\ba_h^k}_{k\in\brk{K}}$ of the realized payoffs $R_{m+1,h}(s_h, (\ba_h^k, a'))$;
   note that to compute the payoffs for each sample, Algorithm \ref{alg:3nash} needs only two oracle calls to $\MO_G$. 

   The following lemma, proven in the sequel, gives a correctness guarantee for Algorithm \ref{alg:3nash}.
\begin{lemma}[Correctness of Algorithm \ref{alg:3nash}]
  \label{lem:algorithmic-lemma}
  Given any $m$-player $n_0$-action normal form game $G$, if the
  algorithm $\CB$ solves the $(T, \ep, N)$-\GenCCE problem for the
  game $\MG(G)$ with $T, \ep, N$ satisfying $T \leq
  \exp(n_0\ep^2/m^2)$, then Algorithm \ref{alg:3nash} outputs a $16(m+1) \cdot \ep$-approximate Nash equilibrium of $G$  with probability at least $1/3$, and otherwise fails.
\end{lemma}

The assumption that $\Tnm < \exp \left( \frac{\epnm^2 \cdot \lfloor
    n/m \rfloor}{m^2} \right)$ from the statement of Theorem
\ref{thm:nonmarkov-multiplayer} yields that $T  \leq \exp(n_0\ep^2 /
m^2)$, so Lemma \ref{lem:algorithmic-lemma} yields that Algorithm
\ref{alg:3nash} outputs a $16(m+1) \cdot \ep$-Nash equilibrium of $G$
with probability at least $1/3$ (and otherwise fails). By iterating
Algorithm \ref{alg:3nash} for $\log(1/\delta)$ times, we may thus compute a $16(m+1) \cdot \ep$-Nash equilibrium of $G$ with failure probability $1-\delta$. 

We now analyze the oracle cost and computational cost of Algorithm \ref{alg:3nash}. It takes $2 Q$ oracle calls to $\MO_G$ to simulate the $Q$ generative model oracle calls of $\CB$, and therefore, if $\CB$ runs in time $U$, then the call to $\CB$ on Line \ref{line:call-b}, using oracle calls to $\MO_G$ to simulate simulate the generative model oracle calls, runs in time $O(U)$. 
Next, the computations of $\til q_{j,h}$ (and thus $\wh q_{j,h}$) in Line \ref{line:compute-hatq-multiplayer} can be performed in $(nmTN)^{O(1)}$ time, the computation of $\wh R_{m+1,h} : \MA_{m+1} \ra \BR$ in Line \ref{line:mp1-br} requires time (and oracle calls to $\MO_G$) bounded above by $O(|\MA_{m+1}| \cdot K) \leq (nm\log(1/\delta)/\ep)^{O(1)}$,  constructing the actions $a_{j,h}$ (for $j \in [m+1]$) in Lines \ref{line:draw-ajh} and \ref{line:mp1-br} takes time $(Nmn)^{O(1)}$ (using the fact that the policies $\sigma_{j,h}\^{t^\st}$ are $N$-computable), and constructing the rewards $r_{j,h}$ on Line \ref{line:reward-simulate} requires another $2(m+1)$ oracle calls to $\MO_G$. 
Altogether, Algorithm \ref{alg:3nash} requires $2Q + (nm \log(1/\delta)/\ep)^{C_0}$ oracle calls to $\MO_G$ and, if $\CB$ runs in time $U$, then Algorithm \ref{alg:3nash} takes time $(nmTNU \log(1/\delta)/\ep)^{C_0}$, for some absolute constant $C_0$. 
  
\end{proof}

\begin{remark}[Bit complexity of exponential weights updates]
In the above proof we have noted that $\til q_{j,h}$ (as defined in
Line \ref{line:compute-hatq-multiplayer} of \cref{alg:3nash}) can be
computed in time $(nmTN)^{O(1)}$. A detail we do not handle formally
is that, since the values of
$\til q_{j,h}(t)$ are in general irrational, only the $(nmTN)^{O(1)}$
most significant bits of each real number $\til q_{j,h}(t)$ can be
computed in time $(nmTN)^{O(1)}$. To give a truly polynomial-time implementation of
\cref{alg:3nash}, one can compute only the $(nmTN)^{O(1)}$ most
significant bits of each distribution $\til q_{j,h}$, which is
sufficient to approximate the true value of $\wh q_{j,h}$ to within
$\exp(-(nmTN)^{O(1)})$ in total variation distance. Since $\wh
q_{j,h}$ only influences the subsequent execution of \cref{alg:3nash}
via the samples $\ba_h^1, \ldots, \ba_h^K \sim \bigtimes_{j \in [m]}
\wh q_{j,h}$ drawn in Line \ref{line:draw-samples}, by a union bound,
the approximation of $\wh q_{j,h}$ we have described perturbs the
execution of the algorithm by at most $O(KH) \cdot
\exp(-(nmTN)^{O(1)})$ in total variation distance. In particular, the
correctness guarantee of Lemma \ref{lem:algorithmic-lemma} still
holds, with sucess probability at least  $1/3 - \exp(-(nmTN)^{O(1)}) > 1/4$. 
\end{remark}

It remains to prove Lemma \ref{lem:algorithmic-lemma}, which is the bulk of the proof of Theorem \ref{thm:nonmarkov-multiplayer}.
\begin{proof}[Proof of Lemma \ref{lem:algorithmic-lemma}]
   We will establish the following two facts:
   \begin{enumerate}
   \item First, the choices of $a_{m+1, h}$ in Line \ref{line:mp1-br}
     (i.e., Eq. \ref{eq:choose-mp1-action}) of Algorithm
     \ref{alg:3nash} correspond to a valid policy
     $\pi_{m+1}^\dagger \in \Pirndrnd$ for player $m+1$ (representing
     a strategy for deviating from the equilibrium $\ol \sigma$), in that they
     can be expressed as a function of player $(m+1)$'s history,
     $(\tau_{m+1,h-1}, s_h)$ at each step $h$.
   \item Second, we will show
     that, since $\ol \sigma$ is an $\ep$-CCE of $\MG$, the strategy
     $\pi_{m+1}^\dagger$ cannot not lead to a large increase of value
     for player $m+1$, which will imply that Algorithm \ref{alg:3nash}
     must return a Nash equilibrium with high enough
     probability. %
   \end{enumerate}

   \paragraph{Defining $\pi_{i}^\dagger$ for $i \in [m+1]$.} We begin
   by constructing the policy $\pi_{m+1}^\dagger$ described; for later
   use in the proof, it will be convenient to construct
   a collection of closely related policies $\pi_i^\dagger \in
   \Pirndrnd$ for $i\in\brk{m}$, also representing
   strategies for deviating from the equilibrium $\ol \sigma$.

Let $i\in\brk{m+1}$ be fixed. For $h \in [H]$, the mapping $\pi_{i,
  h}^\dagger : \CH_{i, h-1} \times \MS \ra \MA_{i}$ is defined as
follows. Given a history $\tau_{i, h-1} = (s_1, a_{i,1}, r_{i,1},
\ldots, s_{h-1}, a_{i,h-1}, r_{i,h-1}) \in \CH_{i,h-1}$ (we assume
without loss of generality that $\tau_{i,h-1}$ occurs with positive
probability under some sequence of general policies) and a current
state $s_h$, we define $\pi_{i,h}^\dagger(\tau_{i,h-1}, s_h) \in
\MA_{i}$ through the following process.
   \begin{enumerate}
   \item First, we claim that for all players
     $j\in\brk{m+1}\setminus\crl{i}$, it is possible to extract the
     trajectory $\tau_{j,h-1}$ from the trajectory $\tau_{i,h-1}$ of
     player $i$.
     \begin{enumerate}
     \item Recall that for each $g < h$, %
       from the definition in \cref{eq:rewards-multiplayers-true} and
       the function $\enc(\ba)$, the bits following position
       $3 \lceil \log 1/\ep \rceil$ of the reward $r_{i,g}$ given to
       player $i$ at step $g$ of the trajectory $\tau_{i,g-1}$ encode
       an action profile $\ba_g \in \ol \MA$. Since $\tau_{i,h-1}$
       occurs with positive probability, this is precisely the action
       profile which was played by agents at step $g$. Note we also
       use here that by definition of the rewards $R_{j,h}(s, \ba)$ in
       (\ref{eq:rewards-multiplayers-true}), the component
       $\ol R_{j,h}(s, \ba)$ of the reward only affects the first
       $2 \lceil \log 1/\ep \rceil$ bits.

     \item For $g < h$ and $j \in [m+1]\backslash \{ i\}$, define
       $r_{j,g} := R_{j,g}(s_{g}, \ba_{g})$.

     \item For $j \in [m+1]\backslash \{ i\}$, write
       $\tau_{j,h-1} := (s_1, a_{j,1}, r_{j,1}, \ldots, s_{h-1},
       a_{j,h-1}, r_{j,h-1})$; in particular, $\tau_{j,h-1}$ is a
       deterministic function of $(\tau_{i,h-1}, s_h)$.  (Note that,
       since $\tau_{i,h-1}$ occurs with positive probability, the
       history $\tau_{j,h-1}$ observed by player $j$ up to step $h-1$
       can be computed from it via Steps (a) and
       (b)). %
       Going forward, for $g < h-1$, we let $\tau_{j,g}$ denote the prefix of
       $\tau_{j,h-1}$ up to step $g$.
     \end{enumerate}
   \item Now, using that player $i$ can compute all players'
     trajectories, for each $j \in [m+1]$ we define
     \begin{align}
       \wh q_{j,h} := \E_{t \sim \til q_{j,h}} \left[ \sigma_{j,h}\^t(\tau_{j,h-1}, s_h) \right] \in \Delta(\MA_j)\label{eq:define-whq-real},
     \end{align}
      where $\til q_{j,h} \in \Delta([T])$ is defined as follows: for $t \in [T]$, %
       \begin{align}
          \til q_{j,h}(t) :=
         \frac{\exp \left( - \sum_{g < h} \log \left( \frac{1}{\sigma\^t_{j, g}(a_{j,g} | \tau_{j,g-1},  s_g)} \right)\right)}{\sum_{t'=1}^T\exp \left( - \sum_{g < h} \log \left( \frac{1}{\sigma\^{t'}_{j, g}(a_{j,g} | \tau_{j,g-1},  s_g)} \right)\right)}\label{eq:define-tilq-real}.
       \end{align}
       Note that $\wh q_{j,h}$ is a random variable which depends on the trajectory $(\tau_{j,h-1}, s_h)$ (which can be computed from $(\tau_{i,h-1}, s_h)$). 
       In addition, the definition of $\wh q_{j,h}$
    (for each $j \in [m]$) %
	   is exactly as is defined in Line \ref{line:compute-hatq-multiplayer}
    of Algorithm \ref{alg:3nash}. 
  \item For $i\in\brk{m}$, define $\pi_{i,h}^\dagger(\tau_{i, h-1}, s_h)$ as follows:
    \begin{align}
  \pi_{i,h}^\dagger(\tau_{i, h-1}, s_h) :=     \argmax_{a' \in \MA_{i}} \E_{\ba_{-i} \sim \bigtimes_{j \neq i} \wh q_{j,h}} \left[R_{m+1,h}(s_h, (a', \ba_{-i}))\right].\label{eq:define-pii-dagger}
    \end{align}
    For the case $i =m+1$, define $\pi_{m+1,h}^\dagger(\tau_{m+1,h-1},
    s_h) \in \Delta(\MA_{m+1})$ \dfedit{(implicitly)} to be the following distribution over $a_{m+1,h}^\dagger \in \MA_{m+1}$: draw $\ba_h^1, \ldots, \ba_h^K \sim \bigtimes_{j \in [m]} \wh q_{j,h}$, define $\wh R_{m+1,h}(a') := \frac 1K \sum_{k=1}^K R_{m+1,h}(s_h, (\ba_h^k, a'))$ for $a' \in \MA_{m+1}$, and finally set
    \begin{align}
      a_{m+1,h}^\dagger:= \argmax_{a' \in \MA_{m+1}} \left\{ \wh R_{m+1,h}(a') \right\}\label{eq:argmax-rhat-policy}.
    \end{align}
  Note that, for each choice of $(\tau_{m+1,h-1}, s_h)$, the
  distribution  $\pi_{m+1,h}^\dagger(\tau_{m+1,h-1}, s_h)$ as defined
  above coincides with the distribution of the action
  $a_{m+1,h}^\dagger$ defined in Eq. \ref{eq:choose-mp1-action} in Algorithm \ref{alg:3nash}, when player $m+1$'s history is $\tau_{m+1,h-1}$ and the state at step $h$ is $s_h$. %
  The following lemma, for use later in the proof, bounds the
  approximation error incurred in sampling $\ba_h^1, \ldots, \ba_h^K \sim \bigtimes_{j \in [m]} \wh q_{j,h}$.
  \begin{lemma}
    \label{lem:sampling-qhat}
    Fix any $(\tau_{m+1,h-1}, s_h) \in \CH_{j,h-1}$. With probability
    at least $1-\delta$ over the draw of $\ba_h^1, \ldots, \ba_h^K \sim \bigtimes_{j \in [m]} \wh q_{j,h}$, it holds that for all $a' \in \MA_{m+1}$,
     \begin{align}
 \left|\wh R_{m+1,h}(a') - \E_{a_j \sim \wh q_{j,h} \ \forall j \in [m]} [R_{m+1,h}(s_h, (a_1, \ldots, a_m, a'))]\right| \leq \frac{\ep}{H}\nonumber,
     \end{align}
     which implies in particular that with probability at least $1-\delta$ over the draw of $a_{m+1,h}^\dagger \sim \pi_{m+1,h}^\dagger(\tau_{m+1,h-1}, s_h)$,
     \begin{align}
\max_{a' \in \MA_{m+1}} \left\{ \E_{a_j \sim \wh q_{j,h} \ \forall j \in [m]} [R_{m+1,h}(s_h, (a_1, \ldots, a_m, a'))] \right\}-\frac{2\ep}{H} \leq  & \E_{a_j \sim \wh q_{j,h} \ \forall j \in [m]} [R_{m+1,h}(s_h, (a_1, \ldots, a_m, a_{m+1,h}^\dagger))]
         \label{eq:max-hat-diff}.
     \end{align}
  \end{lemma}
   \end{enumerate}
   It is immediate from our construction above that the following fact holds.
   \begin{lemma}
     \label{lem:alg-simulation}
     The joint distribution of $\tau_{j,h}$, for $j \in [m+1]$ and $h \in [H]$, as computed by Algorithm \ref{alg:3nash}, coincides with the distribution of $\tau_{j,h}$ in an episode of $\MG$ when players follow the policy $\pi_{m+1}^\dagger \times \ol \sigma_{-(m+1)}$. 
   \end{lemma}

   \paragraph{Analyzing the distributions $\wh q_{j,h}$.} %
   Fix any $i \in [m+1]$. We next  prove some facts about the
   distributions $\wh q_{j,h}$ defined above (as a function of
   $(\tau_{i,h-1}, s_h)$) in the process of computing
   $\pi_{i,h}^\dagger(\tau_{i,h-1}, s_h)$.

   For each $h \in [H]$, consider any choice of $(\tau_{i,h-1}, s_h)
   \in \CH_{i,h-1} \times \MS$; note that for each $j \in [m+1]$, the
   distributions $\wh q_{j,h} \in \Delta(\MA_j)$ for $h \in [H]$ may
   be viewed as an application Vovk's aggregating algorithm (Proposition
  \ref{prop:online-tvd}) in the following setting: the number of steps
  ($T$, in the context of Proposition
  \ref{prop:online-tvd}; note that $T$ has a different meaning in the
  present proof)
  horizon is $H$, the context space is $\bigcup_{h=1}^H \CH_{j,h-1}
  \times \MS$, and the output space is $\MA_j$. The expert set is
  $\cI=\crl{\rho_j\^1, \ldots, \rho_j\^T}$ (which has $\abs{\cI}=T$),
  and the experts' predictions on a context $(\tau_{j,h-1}, s) \in \CH_{j,h-1} \times \MS$ are defined via $\rho_j\^t(\cdot | \tau_{j,h-1}, s) :=  \sigma_{j,h}\^t( \cdot | \tau_{j,h-1}, s) \in \Delta(\MA_j)$. Then for each $h \in [H]$, the distribution $\wh q_{j,h}$ is obtained by updating the aggregating algorithm with the context-observation pairs $(\tau_{j,h'-1}, a_{j,h'})$ for $h' = 1, 2, \ldots, h-1$.

  In more detail, fix any $t^\st \in [T]$ and $j \in [m+1]$ with $i
  \neq j$. We may apply Proposition \ref{prop:online-tvd} with the
  number of steps set to $H$, the set of experts as $\cI=\{ \rho_j\^1,
  \ldots, \rho_j\^T \}$, and contexts and outcomes generated according
  to the distribution induced by running the policy $\pi_{i}^\dagger
  \times \sigma_{-i}\^{t^\st}$ in the Markov game $\MG$ as follows:
  \begin{itemize}
  \item For each $h \in [H]$, we are given, at steps $h' < h$, the actions $a_{k,h'}$ rewards $r_{k,h'}$ for all agents $k \in [m+1]$, as %
    well as the states $s_1, \ldots, s_h$.
    \begin{itemize}
    \item For each $k \in [m+1]$, set $\tau_{k,h-1} = (s_1, a_{k,1}, r_{k,1}, \ldots, s_{h-1}, a_{k,h-1}, r_{k,h-1})$ to be agent $k$'s history.
    \item The \emph{context} fed to the aggregation algorithm at step $h$ is $(\tau_{j,h-1}, s_h)$. 
    \item The \emph{outcome} at step $h$ is given by $a_{j,h} \sim \sigma_{j,h}\^{t^\st}(\cdot | \tau_{j,h-1}, s_h)$; note that this choice satisfies the realizability assumption in Proposition \ref{prop:online-tvd}.
    \item To aid in generating the next context at step $h+1$, choose $a_{k,h} \sim \sigma_{k,h}^{t^\st}(\tau_{k,h-1}, s_h)$ for all $k \in [m+1] \backslash \{ i,j \}$ and $a_{i,h} = \pi_{i,h}^\dagger(\tau_{i,h-1}, s_h)$. 
      Then set $s_{h+1}$ to be the next state given the transitions of $\MG$ and the action profile $\ba_h = (a_{1,h}, \ldots, a_{m+1,h})$. %
    \end{itemize}
  \end{itemize}

  By Proposition \ref{prop:online-tvd}, it  follows that   
  for any fixed $t^\st \in [T]$ and $j \in [m+1]$ with $j \neq i$,
  under the process described above we have
  \begin{align}
    \E_{\pi_{i}^\dagger\times \sigma_{-i}\^{t^\st}} \left[ \sum_{h=1}^H \tvd{\sigma_{j,h}\^{t^\st}(\tau_{j,h-1}, s_h)}{\wh q_{j,h}} \right] \leq \sqrt{H\cdot\log T}.\label{eq:tvd-qhat-nonmarkov}
  \end{align}

   \paragraph{Analyzing the value of $\pi_{m+1}^\dagger$.}
   Next, using the development above, we show that if \cref{alg:3nash}
   successfully computes a Nash equilibrium with constant probability (via $\pi_{m+1}^\dagger$)
   whenever $\wb{\sigma}$ is an $\eps$-CCE. We first state the following claim,
   which is proven in the sequel by analyzing the values $V_i^{\pi_i^\dagger \times \ol \sigma_{-i}}$ for $i \in [m]$.
     \begin{lemma}
    \label{lem:main-players-lb}
    If $\ol \sigma$ is an $\ep$-CCE of $\MG$, then it holds that for all $i \in [m]$,
    \begin{align}
V_{i}^{\ol \sigma} \geq -\ep - m\sqrt{\log(T)/H}\nonumber.
    \end{align}
  \end{lemma}
  Note that in the game $\MG$, since for all $h \in [H]$, $s \in \MS$
  and $\ba \in \ol \MA$, it holds that $\left|\sum_{j=1}^{m+1} R_{j,h}
  (s,\ba)\right| \leq \frac{(m+1)\ep^2}{H}$ (which holds since in (\ref{eq:rewards-multiplayers-true}), $\enc(\ba)$ is multiplied by $\frac{1}{H} \cdot 2^{-3\lceil \log 1/\ep \rceil}$),
    it follows that $\left|\sum_{j=1}^{m+1} V_{j}^{\ol \sigma}\right| \leq (m+1)\ep^2$. Thus, by Lemma \ref{lem:main-players-lb}, we have
  $V_{m+1}^{\ol \sigma} \leq (m+1)\ep^2 + m \cdot (\ep +
  m\sqrt{\log(T)/H})$, and since $\ol \sigma$ is an $\ep$-CCE of $\MG$ it
  follows that
  \begin{align}
    \label{eq:vmdagger-ub}
    V_{m+1}^{\pi_{m+1}^\dagger\times \ol \sigma_{-(m+1)}} \leq 2(m+1) \cdot \ep  + m^2\cdot \sqrt{\log(T)/H}.
  \end{align}

  To simplify notation, we will write $\wh q_h := \wh q_{1,h} \times
  \cdots \times \wh q_{m,h}$ in the below calculations, where we
  recall that each $\wh q_{j,h}$ is determined given the history up to step $h$, $(\tau_{j,h-1}, s_h)$, as defined in (\ref{eq:define-whq-real}) and (\ref{eq:define-tilq-real}). %
  An action profile drawn from $\wh q_h$ is denoted as $\ba \sim \wh q_h$, with $\ba \in \MA$.  
We may now write $V_{m+1}^{\pi_{m+1}^\dagger \times \ol
  \sigma_{-(m+1)}}$ as follows:
  \begin{align}
    & V_{m+1}^{\pi_{m+1}^\dagger\times \ol \sigma_{-(m+1)}}\nonumber\\
    =& \E_{t^\st \sim [T]} \sum_{h =1}^H \E_{\pi_{m+1}^\dagger\times \sigma_{-(m+1)}\^{t^\st}} \E_{\substack{a_{j,h} \sim \sigma_{j,h}\^{t^\st}(\tau_{j,h-1}, s_h) \ \forall j \in [m] \\ a_{m+1,h} \sim \pi_{m+1,h}^\dagger(\tau_{m+1,h-1}, s_h) \\ \ba := (a_{1,h}, \ldots, a_{m+1,h})}} \left[ R_{m+1,h}(s_h, \ba)\right]\nonumber\\
    \geq & \E_{t^\st \sim [T]} \sum_{h=1}^H \E_{\pi_{m+1}^\dagger\times \sigma_{-(m+1)}\^{t^\st}} \Bigg(\E_{\substack{a_{j,h} \sim \wh q_{j,h} \ \forall j \in [m] \\ a_{m+1,h} \sim \pi_{m+1,h}^\dagger(\tau_{m+1,h-1}, s_h) \\ \ba := (a_{1,h}, \ldots, a_{m+1,h})}} \left[ R_{m+1,h}(s_h, \ba)\right]\icml{ \nonumber\\
    & \qquad \qquad } - \frac 1H \sum_{j \in [m]} \tvd{\sigma_{j,h}\^{t^\st}(\tau_{j,h-1}, s_h)}{\wh q_{j,h}} \Bigg) \nonumber\\
    \geq &  \E_{t^\st \sim [T]} \sum_{h=1}^H\E_{\pi_{m+1}^\dagger\times \sigma_{-(m+1)}\^{t^\st}}  \Bigg(\max_{a_{m+1,h}' \in \MA_{m+1}} \E_{\ba \sim \wh q_h} \left[  R_{m+1,h}(s_h, (\ba, a_{m+1,h}'))\right] - \frac{2\ep}{H} - \frac{\delta}{H} \icml{\nonumber\\
    & \qquad \qquad} - \frac 1H \sum_{j \in [m]} \tvd{\sigma_{j,h}\^{t^\st}(\tau_{j,h-1},s_h)}{\wh q_{j,h}} \Bigg) \nonumber\\
    \geq & \frac{1}{H}\cdot  \E_{t^\st \sim [T]} \sum_{h =1}^H \E_{\pi_{m+1}^\dagger\times \sigma_{-(m+1)}\^{t^\st}} \left( \max_{j \in [m], a_{j,h}' \in \MA_j} \E_{\ba \sim \wh q_h }[(M_j)_{a_j', \ba_{-j}} - (M_j)_{\ba}]\right) - \frac{m}{H} \cdot \sqrt{H \log T} - 2\ep - \delta - \ep^2\nonumber,
  \end{align}
  where:
  \begin{itemize}
  \item The first inequality follows from the fact that $R_{m+1,h}(\cdot)$ takes values in $[-1/H,1/H]$ and the fact that the total variation between product distributions is bounded above by the sum of total variation distances between each of the pairs of component distributions.
  \item The second inequality follows from the inequality (\ref{eq:max-hat-diff}) of Lemma \ref{lem:sampling-qhat}. %
  \item The final equality follows from the definition of the rewards
    in (\ref{eq:rewards-multiplayers-true}) and
    (\ref{eq:rewards-multiplayers}), and by summing
    (\ref{eq:tvd-qhat-nonmarkov}) over $j \in [m]$. We remark that the $-\ep^2$ term in the final line comes from the term $\frac{1}{H} \cdot 2^{-3\lceil \log 1/\ep \rceil} \cdot \enc(\ba)$ in (\ref{eq:rewards-multiplayers-true}).
  \end{itemize}
  Rearranging and using (\ref{eq:vmdagger-ub}) as well as the fact that $\delta + \ep^2= \ep/(6H) + \ep^2 \leq \ep$ (as $\ep \leq 1/2$), we get that
  \begin{align}
     & \E_{t^\st \sim [T]}  \E_{\pi_{m+1}^\dagger\times \sigma_{-(m+1)}\^{t^\st}} \sum_{h=1}^H \left( \max_{j \in [m], a_{j,h}' \in \MA_j} \E_{\ba \sim \wh q_h } [(M_j)_{a_j', \ba_{-j}} - (M_j)_{\ba}] \right) \nonumber\\
    \leq & 2H \cdot \ep \cdot (m+1) +(m+1)m \cdot  \sqrt{H \log T} + 3H\ep\nonumber.
  \end{align}
  Since $\wh q_h$ is a product distribution a.s., we have that
  \[
    \max_{j \in [m], a_{j,h}' \in \MA_j} \E_{\ba \sim \wh q_h }
    [(M_j)_{a_j', \ba_{-j}} - (M_j)_{\ba}] \geq 0.
  \] 
  Therefore, by Markov's inequality, with probability at least $1/2$
  over the choice of $t^\st \sim [T]$ and the trajectories $(\tau_{j,h-1}, s_h) \sim \pi_{m+1}^\dagger\times \sigma_{-(m+1)}\^{t^\st}$ for $j \in [m]$ (which collectively determine $\wh q_h$), there is some $h \in [H]$ so that
  \begin{align}
    \max_{j \in [m], a_{j,h}' \in \MA_j} \E_{\ba \sim \wh q_h } [(M_j)_{a_j', \ba_{-j}} - (M_j)_{\ba}] \leq 10(m+1) \cdot  \ep + 2(m+1)m \cdot \sqrt{ \log(T)/H} \leq 12(m+1) \cdot \ep\label{eq:hatq-ne},
  \end{align}
  where the final inequality follows as long as $H \cdot \ep^2 \geq
  m^2 \log T$, i.e., $T \leq \exp \left( \frac{H \cdot \ep^2}{m^2}
  \right)$, which holds since $H \geq n_0$ and we have assumed that $T \leq \exp(\ep^2 \cdot n_0 / m^2)$. 

  Note that (\ref{eq:hatq-ne}) implies that with probability at least $1/2$ under an episode drawn from $\pi_{m+1}^\dagger \times \ol \sigma_{-(m+1)}$, there is some $h \in [H]$ so that $\wh q_h$ is a $12(m+1) \cdot \ep$-Nash equilibrium of the stage game $G$. Thus, by Lemma \ref{lem:alg-simulation}, with probability at least $1/2$ under an episode drawn from the distribution of Algorithm \ref{alg:3nash}, there is some $h \in [H]$ so that $\wh q_h$ is a $12(m+1) \cdot \ep$-Nash equilibrium of $G$.

  Finally, the following two observations conclude the proof of Lemma \ref{lem:algorithmic-lemma}.
  \begin{itemize}
  \item If $\wh q_h$ is a $12(m+1) \cdot \ep$-Nash equilibrium of $G$,
    then by definition of the reward function $R_{m+1, h}(\cdot)$ in (\ref{eq:rewards-multiplayers-true}), upper bounding $\frac{1}{H} \cdot 2^{-3\lceil \log 1/\ep \rceil} \cdot \enc(\ba)$ by $\ep^2/H$, 
    \begin{align}
\max_{a' \in \MA_{m+1}} \E_{\ba \sim \wh q_h} \left[ R_{m+1, h}(\mf s, (\ba, a'))\right] \leq \frac{1}{H} \cdot 12(m+1) \cdot \ep + \frac{\ep^2}{H}\nonumber,
    \end{align}
    which implies, by Lemma \ref{lem:sampling-qhat}, that with
    probability at least $1-\delta$ over the draw of $\ba_h^1, \ldots, \ba_h^K$,
    \begin{align}
\max_{a' \in \MA_{m+1}} \left\{ \wh R_{m+1, h}(a') \right\} \leq \frac{1}{H} \cdot 12(m+1) \cdot \ep + \frac{\ep^2}{H} + \frac{\ep}{H} \leq \frac{1}{H} \cdot 14(m+1) \cdot \ep\nonumber,
    \end{align}
    i.e., the check in Line \ref{line:check-nash} of Algorithm \ref{alg:3nash} will pass and the algorithm will return $\wh q_h$ (if step $h$ is reached). 
  \item Conversely, if $\max_{a' \in \MA_{m+1}} \left\{ \wh R_{m+1,h}(a') \right\} \leq 14(m+1) \cdot \ep$, i.e., the check in Line \ref{line:check-nash} passes, then by Lemma \ref{lem:sampling-qhat}, with probability at least $1-\delta$ over $\ba_h^1, \ldots, \ba_h^K$,
    \begin{align}
\max_{a' \in \MA_{m+1}} \E_{\ba \sim \wh q_h} \left[ R_{m+1, h}(\mf s, (\ba, a')) \right] \leq \frac{1}{H} \cdot 14(m+1) \cdot \ep + \frac{\ep}{H}\leq \frac{1}{H} \cdot 15(m+1) \cdot \ep \nonumber,
    \end{align}
    which implies, by the definition of $R_{m+1,h}(\cdot)$ in (\ref{eq:rewards-multiplayers-true}) and (\ref{eq:rewards-multiplayers}), that $\wh q_h$ is a $16(m+1) \cdot \ep$-Nash equilibrium of $G$. 
  \end{itemize}
  Taking a union bound over all $H$ of the probability-$\delta$
  failure events from Lemma \ref{lem:sampling-qhat} for the sampling
  $\ba_h^1, \ldots, \ba_h^K \sim \wh q_h$ (for $h \in [H]$), as well as over the probability-$1/2$ event that there is no $\wh q_h$ which is a $12(m+1) \cdot \ep$-Nash equilibrium of $G$, we obtain that with probability at least $1 - 1/2 - H \cdot \ep/(6H) \geq 1/3$, Algorithm \ref{alg:3nash} outputs a $16(m+1) \cdot \ep$-Nash equilibrium of $G$.
\end{proof}

Finally, we prove the remaining claims stated without proof above. %

\begin{proof}[Proof of Lemma \ref{lem:sampling-qhat}]
  Since $R_{m+1,h}(\mf s, \ba) \in [-1/H,1/H]$ for each $\ba \in
  \ol\MA$, by Hoeffding's inequality, for any fixed $a' \in
  \MA_{m+1}$, with probability at least $1-\delta/ |\MA_{m+1}| =
  1-\delta/(mn_0)$ over the draw of $\ba_h^1, \ldots, \ba_h^K \sim \bigtimes_{j \in [m]} \wh q_{j,h}$, it holds that
  \begin{align}
 \left|\wh R_{m+1,h}(a') - \E_{a_j \sim \wh q_{j,h} \ \forall j \in [m]} [R_{m+1,h}(s_h, (a_1, \ldots, a_m, a'))]\right| \leq \frac{2}{H} \cdot \sqrt{\frac{\log mn_0/\delta}{K}} \leq \frac{\ep}{H} \nonumber,
  \end{align}
  where the final inequality follows from the choice of $K = \lceil 4 \log(mn_0/\delta) / \ep^2 \rceil$. The statement of the lemma follows by a union bound over all $|\MA_{m+1}|$ actions $a'\in\cA_{m+1}$. 
\end{proof}

\begin{proof}[Proof of Lemma \ref{lem:main-players-lb}]
  Fix any agent $i \in [m]$. We will argue that the policy
  $\pi_i^\dagger \in \Pidet_i$ defined within the proof of
  Lemma \ref{lem:algorithmic-lemma} satisfies $V_{i}^{\pi_i^\dagger, \ol \sigma_{-i}} \geq -m\sqrt{\log(T)/H}$. Since $\ol \sigma$ is an $\ep$-CCE of $\MG$, it follows that
  \begin{align}
\ep \geq V_{i}^{\pi_i^\dagger, \ol \sigma_{-i}} - V_{i}^{\ol \sigma} \geq -m\sqrt{\log(T) / H} - V_{i}^{\ol \sigma}\nonumber,
  \end{align}
  from which the result of Lemma \ref{lem:main-players-lb} follows after rearranging terms.
   
    To simplify notation, let us write $\wh q_{-i,h} := \bigtimes_{j
      \neq i} \wh q_{j,h}$, where we recall that each $\wh q_{j,h}$ is determined given the history up to step $h$, $(\tau_{j,h-1}, s_h)$, as defined in (\ref{eq:define-whq-real}) and (\ref{eq:define-tilq-real}). An action profile drawn from $\wh q_{-i,h}$ is denoted by $\ba_{-i} \sim \wh q_{-i,h}$, with $\ba_{-i} \in \ol \MA_{-i}$.
    We compute
      \begin{align}
        & V_{i}^{\pi_i^\dagger\times \ol \sigma_{-i}}\nonumber\\
        =& \E_{t^\st \sim [T]} \sum_{h=1}^H \E_{\pi_i^\dagger\times \sigma_{-i}\^{t^\st}} \E_{\ba_{-i} \sim \bigtimes_{j \neq i}  \sigma_{j,h}\^{t^\st}(\tau_{j,h-1}, s_h)} \left[ R_{i,h}(s_h, (\pi_{i,h}^\dagger(\tau_{i,h-1}, s_h), \ba_{-i}))\right]\nonumber\\
        \geq & \E_{t^\st \sim [T]} \sum_{h=1}^H \E_{\pi_i^\dagger\times \sigma_{-i}\^{t^\st}} \left(\E_{\ba_{-i} \sim \wh q_{-i,h}} \left[R_{i,h}(s_h, (\pi_{i,h}^\dagger(\tau_{i,h-1}, s_h), \ba_{-i}))\right] - \frac{1}{H} \sum_{j \neq i} \tvd{\sigma_{j,h}\^{t^\st}(\tau_{j,h-1}, s_h)}{\wh q_{j,h}}\right)\nonumber\\
        \geq & \E_{t^\st \sim [T]} \sum_{h =1}^H \E_{\pi_i^\dagger\times \sigma_{-i}\^{t^\st}}  \left(\max_{a_i' \in \MA_i} \E_{\ba_{-i} \sim \wh q_{-i,h}} \left[ R_{i,h}(s_h, (a_i', \ba_{-i}))\right] \right) - \frac{m}{H} \cdot \sqrt{H \log T}\nonumber\\
        \geq & - m\sqrt{\log(T) / H}\nonumber,
      \end{align}
      where:
      \begin{itemize}
      \item The first inequality follows from the fact that the rewards $R_{i,h}(\cdot)$ take values in $[-1/H,1/H]$ and that the total variation between product distributions is bounded above by the sum of total variation distances between each of the pairs of component distributions.
      \item The second inequality follows from the definition of $\pi_{i,h}^\dagger(\tau_{i,h-1}, s_h)$ in terms of $\wh q_{-i,h}$ in (\ref{eq:define-pii-dagger}) as well as (\ref{eq:tvd-qhat-nonmarkov}) applied to each $j \neq i$ and each $t^\st \in [T]$.
        \item The  final inequality follows by Lemma
          \ref{lem:playerk-nonneg} below, applied to agent $i$ and to
          the distribution $\wh q_{-i,h}$, which we recall is a
          product distribution almost surely. %
      \end{itemize}
     \end{proof}
    \begin{lemma}
      \label{lem:playerk-nonneg}
      For any $i \in [m]$, $s \in \MS, h \in [H]$, and any product distribution $q \in \Delta(\ol \MA_{-i})$, it holds that
      \begin{align}
\max_{a_i'\in \MA_i} \E_{\ba \sim q} \left[ R_{i,h}(s, (a_i', \ba))\right] \geq 0\nonumber.
      \end{align}
    \end{lemma}
    \begin{proof}
      Choose $a_i^\st := \argmax_{a_i' \in \MA} \E_{\ba\sim q} \left[
        (M_i)_{a_i', \ba}\right]$. Now we compute 
\begin{align}
  H \cdot   \E_{\ba \sim q} \left[ R_{i,h}(s, (a_i^\st, \ba))\right] \geq &H \cdot \min_{a_{m+1}' \in \MA_{m+1}} \E_{\ba\sim q} \left[ R_{i,h}(s, (a_i^\st, a_{m+1}', \ba_{-(m+1)}))\right]\nonumber\\
  \geq& \min_{(j, a_j') \in \MA_{m+1}} \One{j = i} \cdot  \E_{\ba \sim q} \left[   (M_i)_{a_i^\st, \ba} - (M_i)_{a_i', \ba}\right]\nonumber\\
  \geq & 0\nonumber,
\end{align}
where the first inequality follows since $q$ is a product distribution, the second inequality uses that $\enc(\cdot)$ is non-negative, and the final inequality follows since by choice of $a_i^\st$ we have $\E_{\ba \sim q} \left[ (M_i)_{a_i^\st, \ba} \right] \geq \E_{\ba \sim q} \left[ (M_i)_{a_i', \ba} \right]$ for all $a_i' \in \MA_i$. 
\end{proof}

\subsection{Remarks on bit complexity of the rewards}
The Markov game $\MG(G)$ constructed to prove Theorem
\ref{thm:nonmarkov-multiplayer} uses lower-order bits of the rewards
to record the action profile taken each step. These lower order bits
may be used by each agent to infer what actions were taken by other
agents at the previous step, and we use this idea to construct the
best-response policies $\pi_i^\dagger$ defined in the proof.  As a
result of this aspect of the construction, the rewards of the game
$\MG(G)$ each take $O(m \cdot \log(n) + \log(1/\ep))$ bits to
specify. As discussed in the proof of Theorem
\ref{thm:nonmarkov-multiplayer}, it is without loss of generality to
assume that the payoffs of the given normal-form game $G$ take $O(\log
1/\ep)$ bits each to specify, so when either $m \gg 1$ or $n \gg
1/\ep$, the construction of $\MG(G)$ uses more bits to express its
rewards than what is used for the normal-form game $G$.

It is possible to avoid this phenomenon by instead using the state transitions
of the Markov game to encode the action profile taken at each step, as
was done in the proof of Theorem \ref{thm:markov-formal}. The idea,
which we sketch here, is to replace the game $\MG(G)$ of Definition \ref{def:mg-g-multiplayers} with the following game $\MG'(G)$:
\begin{defn}[Alternative construction to Definition \ref{def:mg-g-multiplayers}]
  \label{def:mg-g-gen}
  Given an $m$-player, $n_0$-action normal-form game $G$, 
      we define the game $\MG'(G) = (\MS, H, (\MA_i)_{i \in [2]}, \BP,
      (R_i)_{i \in [2]}, \mu)$ as follows.
\begin{itemize}
\item The horizon of $\MG$ is $H = n_0$. %
\item Let $A = n_0$. The action spaces of agents $1, 2, \ldots, m$ are given by $\MA_1 = \cdots = \MA_m = [A]$. The action space of agent $m+1$ is %
  \begin{align}
\MA_{m+1} = \{ (j, a_j) \ : \ j \in [m], a_j\in \MA_j \}\nonumber,
  \end{align}
  so that $|\MA_{m+1}| = Am \leq n$. 

 We write $\MA = \prod_{j=1}^m \MA_j$ to denote the joint action space of the first $m$ agents, and $\ol \MA  := \prod_{j=1}^{m+1} \MA_j$ to denote the joint action space of all agents. Then $|\ol \MA| = A^m \cdot (mA) = mA^{m+1} \leq n$. %

\item The state space $\MS$ is defined as follows. There are $|\ol \MA|$ states, one for each action tuple $\ba \in \ol \MA$. For each $\ba \in \ol \MA$, we denote the corresponding state by $\mf s_\ba$. %

\item For all $h \in [H]$, the reward to agent $j \in [m+1]$ given action profile $\ba = (a_1, \ldots, a_{m+1})$ at any state $s \in \MS$ is as follows: writing $a_{m+1} = (j', a_{j'}')$,%
\begin{align}
  R_{j,h}(s, \ba) := \begin{cases}
    0 &: j \not \in \{j', m+1\} \\
  \frac{1}{H} \cdot \left((M_j)_{a_1, \ldots, a_m} - (M_j)_{a_1, \ldots, a_{j'}', \ldots, a_m}\right) &: j=j' \\
  \frac{1}{H} \cdot \left( (M_j)_{a_1, \ldots, a_{j'}', \ldots, a_m} - (M_j)_{a_1, \ldots, a_m} \right) &: j=m+1.
  \end{cases}\label{eq:define-nonmarkov-rewards}
\end{align}
\item At each step $h\in [H]$, if action profile $\ba \in \ol\MA$ is taken, the game transitions to the state $\mf s_{\ba}$. %
\end{itemize}
\end{defn}
Note that the number of states of $\MG'(G)$ is equal to $|\ol \MA| =
mn_0^{m+1}$, and so $|\MG'(G)| = mn_0^{m+1}$. As a result, if we were
to use the game $\MG'(G)$ in place of $\MG(G)$ in the proof of Theorem
\ref{thm:nonmarkov-multiplayer}, we would need to define $n_0 :=
\lfloor n^{1/(m+1)}/m \rfloor$ to ensure that $|\MG'(G)| \leq n$, and
so the condition $T < \exp(\ep^2 \cdot \lfloor n/m \rfloor / m^2)$
would be replaced by $T < \exp(\ep^2 \cdot \lfloor n^{1/(m+1)}/m
\rfloor / m^2)$. This would only lead to a small quantitative
degradement in the statement of Theorem \ref{thm:nonmarkov-formal},
with the condition in the statement replaced by $T < \exp(c \cdot \ep^2 \cdot n^{1/3})$ for some constant $c > 0$. However, it would render the statement of Theorem \ref{thm:statistical-lb} essentially vacuous. For this reason, we opt to go with the approach of Definition \ref{def:mg-g-multiplayers} as opposed to Definition \ref{def:mg-g-gen}.

We expect that the construction of Definition
\ref{def:mg-g-multiplayers} can nevertheless still be modified to use
$O(\log 1/\ep)$ bits to express each reward in the Markov game
$\MG$. In particular, one could introduce stochastic transitions to
encode in the state of the Markov game a small number of random bits
of the full action profile played at each step. We leave such an
approach for future work.

\section{Equivalence between $\Pirndrnd_j$ and $\Delta(\Pidet_j)$}
\label{sec:equivalence}
In this section we consider an alternate definition of the space $\Pirndrnd_i$ of randomized general policies of player $i$, and show that it is equivalent to the one we gave in Section \ref{sec:prelim}.

In particular, suppose we were to define a randomized general policy of agent $i$ as a distribution over deterministic general
  policies of agent $i$: we write $\Pirnd_i := \Delta(\Pidet_i)$ to
  denote the space of such distributions. %
  Moreover, write $\Pirnd := \Pirnd_1 \times \cdots \times \Pirnd_m = \Delta(\Pidet_1) \times \cdots \times \Delta(\Pidet_m)$ to denote the space of product distributions over agents' deterministic policies. %
  Our goal in this section is to show that policies in $\Pirnd$ are equivalent to those in $\Pirndrnd$ in the following sense: there is an embedding map $\PE : \Pirndrnd  \ra \Pirnd$, not depending on the Markov game, so that the distribution of a trajectory drawn from any $\sigma \in \Pirndrnd$, for any Markov game, is the same as the distribution of a trajectory drawn from $\PE(\sigma)$ (Fact \ref{fac:gen-decompose}). Furthermore, $\PE$ is surjective in the following sense: any policy $\til \sigma \in \Pirnd$ produces trajectories that are distributed identically to those of $\PE(\sigma)$ (and thus of $\sigma$), for some $\sigma \in \Pirndrnd$ (Fact \ref{fac:emb-inverse}). 
In Definition \ref{def:define-embed} below, we define $\PE$.   
  
\begin{defn}
  \label{def:define-embed}
  For $j \in [m]$ and $ \sigma_j \in \Pirndrnd_j$, define $\PE_j( \sigma_j) \in \Pirnd_j = \Delta(\Pidet_j)$ to put the following amount of mass on each $\pi_j \in \Pidet_j$:
  \begin{align}
(\PE_j( \sigma_j))(\pi_j) := \prod_{h=1}^H \prod_{(\tau_{j,h-1}, s_h) \in \CH_{j,h-1} \times \MS}  \sigma_j(\pi_{j,h}(\tau_{j,h-1}, s_h) \ | \ \tau_{j,h-1}, s_h)\label{eq:emb-def}
  \end{align}
  Furthermore, for $\sigma = (\sigma_1, \ldots, \sigma_m) \in \Pirndrnd$, define $\PE(\sigma) = (\PE(\sigma_1), \ldots, \PE(\sigma_m))$. 
\end{defn}
Note that, in the special case that $ \sigma_j \in \Pidet_j$, $\PE_j( \sigma_j)$ is the point mass on $\sigma_j$. 
\begin{fact}[Embedding equivalence]
  \label{fac:gen-decompose}
  Fix a $m$-player Markov game $\MG$ and, arbitrary policies $ \sigma_j \in \Pirndrnd_j$. %
  Then a trajectory drawn from the product policy $ \sigma = ( \sigma_1, \ldots,  \sigma_m) \in \Pirndrnd_1 \times \cdots \times \Pirndrnd_m$ is distributed identically to a trajectory drawn from $\PE( \sigma) \in \Pirnd$.%
\end{fact}
\noah{maybe Mention that this corresponds to strategic form strategy equivalence in EFG literature}  The proof of Fact \ref{fac:gen-decompose} is provided in Section \ref{sec:equivalence-proofs}. 
Next, we show that the mapping $\PE$ is surjective in the following sense:
\begin{fact}[Right inverse of $\PE_j$]
  \label{fac:emb-inverse}
  There is a mapping $\PF : \Pirnd \ra \Pirndrnd$ so that for any Markov game $\MG$ and %
   any $\til\sigma \in \Pirnd$, the distribution of a trajectory drawn from $\til\sigma$ is identical to the distribution of a trajectory drawn from $\PE \circ \PF(\til\sigma)$. 
\end{fact}
We will write $\PF((\til\sigma_1, \ldots, \til\sigma_m)) := (\PF_1(\til\sigma_1), \ldots, \PF_m(\til\sigma_m))$. Fact \ref{fac:emb-inverse} states that the policy $\PF(\til \sigma)$ maps, under $\PE$, to a policy in $\Pirnd$ which is equivalent to $\til\sigma$ (in the sense that their trajectories are identically distributed for any Markov game). 

An important consequence of Fact \ref{fac:gen-decompose} is that the expected reward (i.e., value) under any $ \sigma \in \Pirndrnd$ is the same as that of $\PE(\sigma)$. Thus given a Markov game, the induced normal-form game in which the players' pure action sets are $\Pirndrnd_1, \ldots, \Pirndrnd_m$ is equivalent to the normal-form game in which the players' pure action sets are $\Pidet_1, \ldots, \Pidet_m$, in the following sense: for any mixed strategy in the former, namely a product distributional policy $\distp \in \Delta(\Pirndrnd_1) \times \cdots \times \Delta(\Pirndrnd_m)$, the policy $\E_{\sigma \sim \distp}[\PE(\sigma)] \in \Delta(\Pidet_1) \times \cdots \times \Delta(\Pidet_m) = \Pirnd$ is a mixed strategy in the latter which gives each player the same value as under $\distp$. (Note that $\E_{\sigma \sim \distp}[\PE(\sigma)]$ is indeed a product distribution since $\distp$ is a product distribution and $\PE$ factors into individual coordiantes.) Furthermore, by Fact \ref{fac:emb-inverse}, any distributional policy in $\Pirnd$ arises in this manner, for some $\distp \in \Delta(\Pirndrnd_1) \times \cdots \times \Delta(\Pirndrnd_m)$; in fact, $\distp$ may be chosen to place all its mass on a single $\sigma \in \Pirndrnd_1 \times \Pirndrnd_m$. Since $\PE$ factors into individual coordinates, it follows that $\PE$ yields a one-to-one mapping between the coarse correlated equilibria (or any other notion of equilibria, e.g., Nash equilibria or correlated equilibria) of these two normal-form games. 
\noah{this is a bit handwavy but it is annoying to write down the defns of the normal-form games in detail and it's not really necessary for the paper}\noah{probably want to say that if $\distp$ is product then so is $\E_{\sigma \sim \distp}[\PE(\sigma)]$.}

\subsection{Proofs of the equivalence}
\label{sec:equivalence-proofs}
\begin{proof}[Proof of Fact \ref{fac:gen-decompose}]
  Consider any trajectory $\tau = (s_1, \ba_1, \br_1, \ldots, s_H, \ba_H, \br_H)$ consisting of a sequence of $H$ states and actions and rewards for each of the $m$ agents. Assume that $r_{i,h} = R_{i,h}(s, \ba_h)$ for all $i, h$ (as otherwise $\tau$ has probability 0 under any policy). Write:
  \begin{align}
p_\tau := \prod_{h=1}^{H-1} \BP_h(s_{h+1} | s_h,\ba_h)\nonumber.
  \end{align}
  Then the probability of observing $\tau$ under $ \sigma$ is
  \begin{align}
    p_\tau \cdot \prod_{h=1}^{H-1} \prod_{j=1}^m  \sigma_{j,h}(a_{j,h} | \tau_{j,h-1}, s_h)\label{eq:prob-tilsigma}
  \end{align}
  where, per usual, $\tau_{j,h-1} = (s_1, a_{j,1}, r_{j,1}, \ldots, s_{h-1}, a_{j,h-1}, r_{j,h-1})$. Write $\sigma =(\sigma_1, \ldots, \sigma_m)= \PE(\sigma)$. The probability of observing $\tau$ under $\sigma$ is
  \begin{align}
p_\tau \cdot \prod_{j \in [m]} \sum_{\pi_j \in \Pidet_j :\ \forall h,\ \pi(\tau_{j,h-1}, s_h) = a_{j,h}} \sigma_j(\pi_j)\label{eq:prob-sigma}
  \end{align}
  It is now straightforward to see from the definition of $\sigma_j(\pi_j)$ in (\ref{eq:emb-def}) that the quantities in (\ref{eq:prob-tilsigma}) and (\ref{eq:prob-sigma}) are equal.
\end{proof}

\begin{proof}[Proof of Fact \ref{fac:emb-inverse}]
  Fix a policy $\til\sigma_j \in \Pirnd_j = \Delta(\Pidet_j)$. We define $\PF_j(\til\sigma_j)$ to be the policy $ \sigma_j \in \Pirndrnd_j$, which is defined as follows: for $\tau_{j,h-1} = (s_{j,1}, a_{j,1}, r_{j,1}, \ldots, s_{j,h-1}, a_{j,h-1}, r_{j,h-1}) \in \CH_{j,h-1}$, $s_h \in \MS$, we have, for $a_{j,h} \in \MA_j$, 
  \begin{align}
 \sigma_j(\tau_{j,h-1}, s_h)(a_{j,h}) = \frac{\til\sigma_j \left( \{ \pi_j \in \Pidet_j \ : \ \pi_j(\tau_{j,g}, s_g) = a_{j,g} \ \forall g \leq h \} \right)}{\til\sigma_j \left( \{ \pi_j \in \Pidet_j \ : \ \pi_j(\tau_{j,g}, s_g) = a_{j,g} \ \forall g \leq h-1 \} \right)}\nonumber.
  \end{align}
  If the denominator of the above expression is 0, then $ \sigma_j(\tau_{j,h-1}, s_h)$ is defined to be an arbitrary distribution on $\Delta(\MA_j)$. (For concreteness, let us say that it puts all its mass on a fixed action in $\MA_j$.) Furthermore, for $\til \sigma \in \Pirnd$, define $\PF(\til \sigma) := (\PF_1(\til \sigma_1), \ldots, \PF_m(\til \sigma_m)) \in \Pirndrnd$. 

  Next, fix any $\til\sigma = (\til\sigma_1, \ldots, \til\sigma_m) \in \Pirnd_1 \times \cdots \times \Pirnd_m$.
  Let $\sigma = \PF(\til\sigma)$. By Fact \ref{fac:gen-decompose}, it suffices to show that the distribution of trajectories under $\sigma$ is the same as the distribution of trajectories drawn from $ \sigma$.

   So consider any trajectory $\tau = (s_1, \ba_1, \br_1, \ldots, s_H, \ba_H, \br_H)$ consisting of a sequence of $H$ states and actions and rewards for each of the $m$ agents. Assume that $r_{i,h} = R_{i,h}(s, \ba_h)$ for all $i, h$ (as otherwise $\tau$ has probability 0 under any policy). Write:
  \begin{align}
p_\tau := \prod_{h=1}^{H-1} \BP_h(s_{h+1} | s_h,\ba_h)\nonumber.
  \end{align}
  Then the probability of observing $\tau$ under $ \sigma$ is
  \begin{align}
    &    p_\tau \cdot \prod_{h=1}^{H} \prod_{j=1}^m  \sigma_{j,h}(a_{j,h} | \tau_{j,h-1}, s_h)\nonumber\\
    =& p_\tau \cdot \prod_{j=1}^m  \prod_{h=1}^{H} \frac{\til\sigma_j \left( \{ \pi_j \in \Pidet_j \ : \ \pi_j(\tau_{j,g}, s_g) = a_{j,g} \ \forall g \leq h \} \right)}{\til\sigma_j \left( \{ \pi_j \in \Pidet_j \ : \ \pi_j(\tau_{j,g}, s_g) = a_{j,g} \ \forall g \leq h-1 \} \right)}\nonumber\\
    =& p_\tau \cdot \prod_{j=1}^m \til\sigma_j \left( \{ \pi_j \in \Pidet_j \ : \ \pi_j(\tau_{j,g}, s_g) = a_{j,g} \ \forall g \leq H \} \right)\nonumber,
  \end{align}
  which is equal to the probability of observing $\tau$ under $\til\sigma$. 
\end{proof}

\end{document}